\documentclass{article}

\usepackage[numbers]{natbib}
\usepackage[english]{babel}
\usepackage[a4paper,top=2cm,bottom=2cm,left=1.5cm,right=1.5cm,marginparwidth=1.75cm]{geometry}

\usepackage{amsmath,amsthm,amssymb,amsfonts,mathtools,bm}
\usepackage{comment}
\usepackage{tcolorbox}
\usepackage{mathrsfs}
\usepackage{enumitem}
\usepackage{graphicx}
\usepackage{subcaption}
\usepackage{float}
\usepackage{xcolor}
\usepackage{booktabs,multirow}
\usepackage{url}

\definecolor{targetcolor}{HTML}{697386}
\definecolor{sourcecolor}{HTML}{B28A46}
\definecolor{ridgeblue}{HTML}{377EB8}
\definecolor{ridgegreen}{HTML}{31945B}
\definecolor{ridgeoptimal}{HTML}{D43D35}

\renewcommand{\P}{\mathbb{P}}
\newcommand{\E}{\mathbb{E}}

\newcommand{\cN}{\mathcal{N}}

\newcommand{\R}{\mathbb{R}}

\newcommand{\eps}{\varepsilon}

\newcommand{\<}{\langle}
\renewcommand{\>}{\rangle}

\def\sT{{\mathsf T}}

\DeclareMathOperator*{\argmin}{arg\,min}
\DeclareMathOperator*{\argmax}{arg\,max}

\newtheorem{theorem}{Theorem}
\newtheorem*{theorem*}{Theorem}
\newtheorem{lemma}{Lemma}

\newtheorem{assumption}{Assumption}

\newtheorem{corollary}{Corollary}

\theoremstyle{definition}

\newtheoremstyle{myremark} 
    {\topsep}                    
    {\topsep}                    
    {\it}                        
    {}                           
    {\bf}                        
    {.}                          
    {.5em}                       
    {}  

\theoremstyle{myremark}
\newtheorem{remark}{Remark}[section]

\DeclareSymbolFont{rsfs}{U}{rsfs}{m}{n}
\DeclareSymbolFontAlphabet{\mathscrsfs}{rsfs}

\def\bA{{\boldsymbol A}}
\def\bB{{\boldsymbol B}}
\def\bC{{\boldsymbol C}}
\def\bD{{\boldsymbol D}}
\def\bE{{\boldsymbol E}}

\def\bG{{\boldsymbol G}}

\def\bH{{\boldsymbol H}}

\def\bI{{\boldsymbol I}}
\def\bJ{{\boldsymbol J}}
\def\bK{{\boldsymbol K}}
\def\bL{{\boldsymbol L}}
\def\bM{{\boldsymbol M}}
\def\bN{{\boldsymbol N}}

\def\bP{{\boldsymbol P}}
\def\bQ{{\boldsymbol Q}}
\def\bR{{\boldsymbol R}}
\def\bS{{\boldsymbol S}}
\def\bT{{\boldsymbol T}}

\def\bV{{\boldsymbol V}}

\def\bX{{\boldsymbol X}}
\def\bY{{\boldsymbol Y}}
\def\bZ{{\boldsymbol Z}}

\def\be{{\boldsymbol e}}

\def\bg{{\boldsymbol g}}
\def\bh{{\boldsymbol h}}

\def\bu{{\boldsymbol u}}
\def\bv{{\boldsymbol v}}
\def\bw{{\boldsymbol w}}
\def\bx{{\boldsymbol x}}
\def\by{{\boldsymbol y}}
\def\bz{{\boldsymbol z}}

\def\bmu{{\boldsymbol \mu}}

\def\beps{{\boldsymbol \eps}}

\def\btheta{{\boldsymbol \theta}}

\def\bDelta{{\boldsymbol \Delta}}

\def\bTheta{{\boldsymbol \Theta}}

\def\Tr{{\rm Tr}}
\def\op{{\rm op}}

\def\cL{{\mathcal L}}

\def\sV{{\sf V}}

\def\bDelta{{\boldsymbol \Delta}}

\def\bA{{\boldsymbol A}}
\def\btheta{{\boldsymbol \theta}}
\def\bTheta{{\boldsymbol \Theta}}

\def\bP{{\boldsymbol P}}

\def\bS{{\boldsymbol S}}

\def\bD{{\boldsymbol D}}

\def\bL{{\boldsymbol L}}

\def\bsigma{{\boldsymbol \sigma}}

\def\bR{{\boldsymbol R}}

\def\bC{{\boldsymbol C}}

\def\sR{\mathsf R}
\def\sV{\mathsf V}
\def\sB{\mathsf B}

\def\br{{\boldsymbol r}}

\def\sR{\mathsf R}
\def\sV{\mathsf V}
\def\sB{\mathsf B}

\def\op{{\rm op}}
\newcommand{\dd}{\mathop{}\!\mathrm{d}}

\usepackage{hyperref}
\hypersetup{
    colorlinks=true,
    citecolor=blue,
    linkcolor=magenta,
}
\newcommand{\affmark}[1]{%
  \begingroup
  \renewcommand{\thefootnote}{\arabic{footnote}}%
  \footnotemark[#1]%
  \endgroup
}
\newcommand{\afftext}[2]{%
  \begingroup
  \renewcommand{\thefootnote}{\arabic{footnote}}%
  \footnotetext[#1]{#2}%
  \endgroup
}

\title{{\fontsize{22}{27}\selectfont When do data mixtures improve scaling laws?\\ Insights from high-dimensional regression}}

\author{
Diyuan Wu\thanks{Equal contribution.} \affmark{1}
\\
\and
Lehan Chen\footnotemark[1]
\affmark{2}\\
\and 
Theodor Misiakiewicz\thanks{Equal advising.} \affmark{3}
\\
\and Marco Mondelli\footnotemark[2] \affmark{1}
}
\date{}

\begin{document}
\maketitle
\afftext{1}{Institute of Science and Technology Austria (ISTA). Emails: \texttt{\{diyuan.wu, marco.mondelli\}@ist.ac.at}}
\afftext{2}{Department of Applied and Computational Mathematics, Yale University. Emails: \texttt{lehan.chen@yale.edu}}
\afftext{3}{Department of Statistics and Data Science, Yale University. Emails: \texttt{theodor.misiakiewicz@yale.edu}}

\begin{abstract}
Modern machine learning systems are trained on mixtures of data from different domains, and choosing the right mixture can substantially improve downstream performance. Despite an extensive literature on data mixing and reweighting, existing work is largely empirical and it remains unclear when auxiliary data genuinely improves scaling laws rather than merely providing more samples. To gain insight into this question, we study a high-dimensional mixed-data regression model with a shared regression function, heterogeneous covariances and noise levels, and dataset sizes that may grow at different rates. We establish the minimax risk under an ellipsoidal parameter constraint for the general covariance structure and  derive deterministic equivalents for the test error of ridge regression under commutative covariances. We then specialize to a target domain and an auxiliary domain with aligned power-law covariance spectra, where the theory yields explicit scaling laws in terms of spectral decay, target regularity, and the relative growth of the two datasets. These laws identify regimes in which combining data mixtures provably yields a faster scaling rate than using either dataset alone. In particular, improving the scaling law requires a specific interplay between spectra and relative sample sizes of the domains. Our numerical experiments on language models exhibit the same qualitative phenomenon: appropriate data mixtures yield a faster decrease in target-domain test loss than training on either domain alone.
\end{abstract}

\section{Introduction}

Many successful machine learning systems learn from
data spanning multiple domains \citep{grattafiori2024llama,zitkovich2023rt2},
making the composition of their training sets a central design choice. Several approaches to optimize data mixtures have been proposed, including transferring mixture weights from smaller
to larger models \citep{xie2023doremi,fan2023doge,liu2025regmix},
predicting how the optimal mixture changes with the total
training budget \citep{kang2025autoscalescaleawaredatamixing},
and adjusting the weights during a training run according to the model's progress on different domains
\citep{chen2023skillit,chen2025aioli,jiang2025ado}. The dependence of the performance on model size, data volume, and mixture composition has been described by empirical scaling laws
\citep{aghajanyan2023scaling,shukor2026scaling,sedova2026scaling},
with recent work suggesting that the mixture composition can affect the scaling
exponent \citep{hamidieh2026domain}. At the same time, \cite{wang2026smalltrainingrunsreliably} have observed that the choice of the mixture may be sensitive to small changes in training hyperparameters. Together, these findings motivate a principled understanding of how and when data mixtures improve performance and, in particular, the scaling laws.

When data is drawn from a single distribution, a line of theoretical work has derived scaling laws for linear and kernel regression \citep{maloney2022solvable,bahri2024explaining,lin2024scaling,lin2025improved,atanasov2026scaling}, as well as for simplified models of neural networks \citep{paquette20244+,bordelon2024dynamical,ferbach2025dimension,ren2025emergence,defilippis2024dimension}. However,  less is known about how these laws change when training on multiple data sources. In the model studied by \cite{hashimoto2021model}, mixture composition affects only the prefactor and leaves the scaling exponent unchanged. Later, \cite{jain2024scaling} derived scaling laws for mixtures of real and surrogate data, and \cite{wu2026improved} showed that, in weak-to-strong generalization, a student trained on teacher-generated labels can achieve a better scaling exponent than its teacher. Furthermore, in a stylized memorization model, Medvedev et al.~\citep{medvedev2026shift} showed that optimizing training proportions can improve the scaling exponent relative to sampling from the test distribution. Most recently, Dai and Zheng~\citep{dai2026explaining} developed a theoretical model of how mixture proportions affect learning curves, predicting effective mixtures across model and dataset sizes. These results leave open the characterization of when combining datasets yields faster scaling rates than either dataset alone, and how these gains depend on the properties of the data sources and their relative sizes.

In this paper, we study a tractable setting of high-dimensional regression and identify regimes in which data mixtures improve scaling laws.  We consider multiple training domains with a common regression function, different feature covariances and noise levels, and sample sizes growing at different rates. Our contributions can be summarized as follows:

\begin{itemize}[itemsep=2pt,leftmargin=*]
\item \textbf{Minimax risk and ridge performance.} We characterize both minimax statistical limits and the performance of a ridge estimator: for general covariances, Theorem \ref{thm:minimax-ellipsoid} gives a variational characterization of the minimax prediction risk under an ellipsoidal parameter constraint, with matching bounds up to constants; for covariances that commute, Theorem \ref{thm:deteq-risk} derives deterministic equivalents for the bias and variance of mixed ridge regression, which depend only on population-level quantities (covariances, sample sizes, noise levels, regularization).

\item \textbf{Scaling laws for data mixtures.} We next specialize both minimax and ridge characterizations to two mixtures with power-law spectra, under source and capacity conditions (Theorems \ref{thm:minimax}-\ref{thm:scaling-optimal-ridge}). This allows to identify regimes where \emph{(i)} ridge regression achieves the minimax optimal exponent in the scaling law, and \emph{(ii)} combining different mixtures strictly improves the scaling exponent as compared to using either dataset alone.

\item \textbf{Numerical evidence from language models.} Finally, we demonstrate via numerical simulations that decoder-only transformers exhibit a qualitatively similar advantage in data mixing: in regimes consistent with our theory, combining target and auxiliary data produces a faster decrease in target-domain test loss than training on either dataset alone (Figure \ref{fig:lang-scalinglaw}).
\end{itemize}

\section{Related work}

From an \textbf{empirical viewpoint}, existing work on data mixtures has studied how the composition of the training set affects performance and how to choose effective mixtures. DoReMi, DoGE, and Chameleon
select domain weights using proxy models, gradient alignment, and
feature statistics, respectively
\citep{xie2023doremi,fan2023doge,xie2025chameleon}.
Skill-It, Aioli, and ADO adapt data selection during training
\citep{chen2023skillit,chen2025aioli,jiang2025ado}, while RegMix
and MixMin choose mixtures through performance prediction and
surrogate optimization \citep{liu2025regmix,thudi2025mixmin}. Empirical mixing laws model loss as a function of domain proportions and training scale \citep{ge2024bimix,ye2025data,shukor2026scaling}, and AutoScale extrapolates optimized mixtures to larger training budgets \citep{kang2025autoscalescaleawaredatamixing}. Additional work on scaling laws examines competition and positive transfer between modalities \citep{aghajanyan2023scaling}, and data mixing with repetition when target data is scarce \citep{sedova2026scaling}. Our work provides a theoretical perspective on data mixtures: we identify conditions under which combining datasets improves the scaling exponent over either dataset alone, and then validate them on language model experiments.


From a \textbf{theoretical viewpoint}, work on data mixtures includes distribution-weighted combinations of predictors \citep{Mansour08Domain,hoffman2018algorithmstheorymultiplesourceadaptation},
selecting weights that balance
distribution mismatch and estimation error
\citep{konstantinov2019robustlearninguntrustedsources,Deng2023Mixture}, and
adaptive sampling across distributions
\citep{haghtalab2022demand}. 
In nonparametric regression, gains from combining mixture data
have been characterized over Lipschitz classes
\citep{schmidthieber2024local}, and recent work on H\"older classes establishes
minimax rates faster than those attainable from either dataset alone
\citep{zhou2025synergistic,zamolodtchikov2026minimax}. Whether faster rates are possible for high-dimensional regression remains, however, open. In particular, existing analyses of data mixtures in this setting
do not establish that mixing improves the scaling exponent
\citep{dai2026explaining,jain2024scaling}, and for ridge regression with data labeled according to different objectives, Jagadeesan et al.~\citep{jagadeesan2025safety} have identified regimes in which scaling 
worsens at large dataset sizes. 
A related topic is covariate shift, studied through minimax analyses of kernel ridge regression \citep{ma2023optimally}, asymptotic characterizations of the risk in kernel and
random-feature regression \citep{canatar2021out,tripuraneni2021overparameterization}, and minimum-norm interpolation \citep{mallinar2024minimum}. 
Under both covariate and model shift, \cite{yang2025precise,song2024generalization} have established precise risk characterizations  for pooled least squares and minimum-norm interpolation,
 in the proportional regime where number of samples and feature dimension grow at the same rate. 
Our analysis accommodates decaying covariance spectra and datasets
whose sizes may grow at different rates, and it yields scaling laws under
source-capacity conditions
\citep{caponnetto2007optimal,rudi2017generalization}.
Such conditions also underpin scaling laws for kernel and random feature regression
\citep{cui2021generalization,defilippis2024dimension,atanasov2026scaling}. 
We use this framework to establish when data mixing improves the scaling law exponent, both in a minimax sense and for ridge regression.

\section{Problem setup}
\label{sec:setup}

We consider $K$ independent training datasets. For each $i \in [K]$, we observe $n_i$ independent samples $(\bx_i^{(j)},y^{(j)}_i)_{j \in [n_i]}$ from the $i$-th dataset, where $\bx_i^{(j)}$ are drawn from a distribution $P_i$ on $\R^d$ with 
\begin{equation*}
\E_{P_i}[\bx] = 0, \qquad \E_{P_i}[\bx\bx^\top]=\bC_i,\qquad i=1,\dots,K.
\end{equation*}
We set $N:=\sum_{i=1}^K n_i$. The dimension $d$ may be infinite, in which case $\bC_i$ is a trace-class operator on $\ell_2$.  Denote $\bX_i = [\bx^{(1)}_i, \ldots, \bx^{(n_i)}_i]^\sT \in \R^{n_i \times d}$ the design matrix. All domains share the same target parameter $\btheta_*$ with $\| \bC_i^{1/2} \btheta_* \|_2^2 <\infty$, $i\in [K]$, while their noise levels $\sigma_{\eps_i}^2$ may differ. The labels $\by_i = (y_i^{(j)})_{j \in [n_i]}$ are given by
\begin{equation*}
\by_i=\bX_i\btheta_*+\mathbf{\eps}_i,\qquad \mathbf{\eps}_i\sim\cN(0,\sigma_{\eps_i}^2\bI_{n_i}),\qquad i=1,\dots,K.
\end{equation*}
The test distribution is a mixture $P=\sum_k\pi_k^*P_k$ of the same domains, with $\pi_k^*\geq0$ and $\sum_k\pi_k^*=1$; a single target domain corresponds to $\pi^*=\be_1$. For an estimator $\widehat\btheta$, we consider the excess test risk 
\begin{equation*}
\sR_\beps (\widehat\btheta) := \E_{\bx \sim P} \left[ \big( \bx^\sT \widehat \btheta - \bx^\sT \btheta_* \big)^2  \right] = \sum_{k=1}^K\pi_k^*\big\|\bC_k^{1/2}(\widehat\btheta-\btheta_*)\big\|_2^2.
\end{equation*}
In this paper, we characterize the minimax risk in this setting (i.e., the optimal risk over all estimators)
, and the risk of the following mixed ridge estimator 
\begin{equation}
\widehat\btheta=\argmin_{\btheta}\Big\{\sum_{i=1}^K\|\bX_i\btheta-\by_i\|_2^2+\lambda\|\btheta\|_2^2\Big\}=\bG\sum_{i=1}^K\bX_i^\top\by_i,
\label{eq:mixed-ridge}
\end{equation}
with regularization parameter $\lambda >0$, where $\bG=(\sum_{i=1}^K\bX_i^\top\bX_i+\lambda\bI)^{-1}$. Conditionally on designs, the risk of \eqref{eq:mixed-ridge} averaged over label noise admits the bias-variance decomposition $\sR(\widehat\btheta) := \E_{\beps} \sR_\beps(\widehat\btheta) =\sum_k\pi_k^*(\sB_k+\sV_k)$ with
\begin{equation*}
\sB_k=\lambda^2\big\langle\btheta_*,\bG\bC_k\bG\btheta_*\big\rangle,\qquad \sV_k=\sum_{i=1}^K\sigma_{\eps_i}^2\Tr\big(\bC_k\bG\bX_i^\top\bX_i\bG\big).
\end{equation*}
Importance-weighted objectives $\sum_ia_i\|\bX_i\btheta-\by_i\|_2^2$ are covered by the substitution $(\bC_i,\sigma_{\eps_i}^2, \pi_k^*)\to(a_i\bC_i,a_i\sigma_{\eps_i}^2,\pi_k^*/a_k)$ , so all results below also describe domain reweighting. 
We assume the following concentration property on the designs.

\begin{assumption}[Concentration of the designs]
\label{asm:distr}
There exist constants $C,c,\eta>0$ such that, for every $i\in[K]$, every PSD operator $\bA \in \R^{d \times d}$, and every $t>0$, a sample $\bx\sim P_i$ satisfies 
\begin{equation}
    \P(|\bx^\top\bA\bx-\Tr(\bC_i\bA)|> t\|\bC_i^{1/2}\bA\bC_i^{1/2}\|_F)\leq Ce^{-ct^{1/\eta}}.
\end{equation}
\end{assumption}

This condition covers several popular assumptions in the high-dimension regression literature, including independent sub-Gaussian coordinates and convex Lipschitz concentration \citep{cheng2024dimension,misiakiewicz2024non}. 


\paragraph{Connection to kernel methods.}
Let $\psi=(\psi_j)_{j\geq1}$ be the feature map of a reproducing kernel, and let the inputs of domain $i$ be $\bz\sim\rho_i$. Setting $\bx=\psi(\bz)$ gives $\bC_i=\E_{\rho_i}[\psi(\bz)\psi(\bz)^\top]$, and the shared-parameter model states that all domains are labeled by the same function $f_*=\<\psi,\btheta_*\>$ in the RKHS. The estimator \eqref{eq:mixed-ridge} is then kernel ridge regression under data mixture.

\section{General results}
\label{sec:deterministic-equivalence}


\paragraph{Minimax error under ellipsoid constraint.}

Given a nonnegative self-adjoint operator $\bS$ and a radius
$R>0$, consider
$\bTheta=\{\btheta\in\operatorname{Dom}(\bS):
\|\bS^{1/2}\btheta\|_2\leq R\}$, and define the minimax
 risk
\begin{equation*}
\mathrm R_*(n_1,\dots,n_K)=\inf_{\widehat\btheta}\sup_{\btheta_*\in\bTheta}\sum_{k=1}^K\pi_k^*\,\E_{\btheta_*}\big[\|\bC_k^{1/2}(\widehat\btheta-\btheta_*)\|_2^2\big],
\end{equation*}
where the infimum is over all measurable estimators of the $K$ datasets and the expectation is over designs and noise. We assume $\ker\bS\subseteq\bigcap_i\ker\bC_i$ and take inverse powers of $\bS$ on $(\ker\bS)^\perp$. The relevant quantities are the test and training covariances rescaled by the ellipsoid constraint,
\begin{equation*}
\bH_i=\bS^{-1/2}\bC_i\bS^{-1/2},\qquad \bQ=\sum_{k=1}^K\pi_k^*\bH_k,\qquad \bM=\sum_{i=1}^K\frac{n_i}{\sigma_{\eps_i}^2}\bH_i,
\end{equation*}
together with the largest signal-to-noise ratio permitted by the class, $E_0=R^2\max_i\|\bH_i\|_{\op}/\sigma_{\eps_i}^2$, and the variational functional
\begin{equation}
\cL(\bM)=\sup_{\substack{\bA\succeq0\\\Tr(\bA)\leq R^2}}\Tr\Big[\bQ\bA^{1/2}\big(\bI+\bA^{1/2}\bM\bA^{1/2}\big)^{-1}\bA^{1/2}\Big],
\label{eq:ellipsoid-variational}
\end{equation}
where in infinite dimension the supremum is over positive trace-class operators.

\begin{theorem}[Minimax risk under an ellipsoid constraint]
\label{thm:minimax-ellipsoid}

Under Assumption~\ref{asm:distr}, there is a constant $C > 0$ depending only on the constants in that assumption such that, for every $n_1,\dots,n_K$, 
\begin{equation*}
\frac1{\pi^2}\,\cL(\bM)\leq\mathrm R_*(n_1,\dots,n_K)\leq(1+CE_0)\,\cL(\bM).
\end{equation*}
In particular, if $E_0$ is bounded by a constant, then $\mathrm R_*(n_1,\dots,n_K)\asymp\cL(\bM)$ with constants independent of the dimension and of the sample sizes.

\end{theorem}

The proof compares the regression model with the Gaussian sequence model $\by_{\rm seq}=\bM^{1/2}\bu+\bg$, $\bg\sim\cN(0,\bI)$, for the normalized parameter $\bu=\bS^{1/2}\btheta$. For the latter, we show that the minimax risk among \emph{linear} estimators is exactly $\cL(\bM)$, and the unrestricted minimax rate is  within a factor of $\pi^{-2}$ by a van Trees-type inequality \citep{gassiat2024vantrees}. For the original regression model, the  van Trees argument gives the same lower bound. We obtain the upper bound by constructing an explicit estimator and controlling the additional variance from the random designs, which introduces the factor $1+CE_0$. The detailed proof is given in Appendix~\ref{apx:minimax-ellipsoid}. 

Theorem~\ref{thm:minimax-ellipsoid} shows that, at the level of minimax rates, a collection of heterogeneous datasets is equivalent to a single Gaussian sequence model with information operator $\bM=\sum_i n_i\sigma_{\eps_i}^{-2}\bH_i$. Domain $i$ contributes information in proportion to $n_i/\sigma_{\eps_i}^2$, with greater contributions along directions in which its covariance is large relative to the constraint. When $\bS$ and the $\bC_i$ commute, the variational problem reduces to a scalar program solved by truncation:  in their common eigenbasis, ordering the coordinates by decreasing $[\bQ]_{jj}$, one obtains  (see Appendix~\ref{apx:pf-thm:minimax} for a similar calculation in a power-law model)
\begin{equation}
\mathrm R_*(n_1,\dots,n_K)\asymp\inf_{m\geq0}\Big\{R^2\max_{j>m}[\bQ]_{jj}+\sum_{j\leq m}\frac{[\bQ]_{jj}}{[\bM]_{jj}}\Big\}.
\label{eq:truncation-intuition}
\end{equation}
Each estimated coordinate costs its test weight divided by the total information that all domains carry about it, and the remaining coordinates cost the squared radius times their largest test weight.

\begin{remark}[Tightness of the minimax bounds] 
    The factor $E_0$ is $O(1)$ whenever $R=O(1)$, $\sigma_{\eps_i}^2$ are bounded below, and the constraint dominates the covariances, $\max_i\|\bH_i\|_{\op}=O(1)$. This covers Euclidean balls with bounded covariance spectra as well as the source-condition ellipsoids of Section~\ref{sec:synthetic-scaling}. The lower bound uses Gaussian noise but only second moments of the designs. On ellipsoids, linear estimators are minimax within the Ibragimov--Hasminskii constant $1.25$ \citep{donoho1990minimax}; our factor $\pi^2$ is not optimized, but the argument handles the weighted loss induced by $\bQ$ and non-Gaussian designs.
\end{remark}




\paragraph{Deterministic equivalent for the mixed ridge estimator.}

Let $(\mu_1, \dots, \mu_K)$ be the unique positive solution of the fixed-point system
\begin{equation*}
     \mu_i = \frac{n_i}{1 + \Tr(\bC_i \overline \bG) }, \qquad \overline \bG = (\sum_{i=1}^K \mu_i \bC_i + \lambda)^{-1}, \quad i \in [K].
 \end{equation*} 
 The matrix $\overline\bG$ is a deterministic proxy for the resolvent $\bG$: each dataset is replaced by its population covariance scaled by an effective sample size $\mu_i\leq n_i$, deflated by the effective dimension $\Tr(\bC_i\overline\bG)$ that the dataset has to fit. For a deterministic $\bA$, define $\tau_{\bA}\in\R^K$ and the $K\times K$ matrix $\bL$ by
\begin{equation*}
\tau_{\bA}[i]=\Tr(\bA\overline\bG \bC_i\overline\bG),\qquad \bL_{ij}=\frac{n_i}{\mu_i^2}\mathbf 1[i=j]-\Tr(\bC_i\overline\bG \bC_j\overline\bG).
\end{equation*}

\begin{assumption}\label{asm:nu}
    There exists a constant $K_0 > 0$ such that for all $i \in [K]$, $\lambda\nu_\lambda^i(n_i) \geq n_i^{-K_0}$, where $\nu_\lambda^i(n_i)$ is defined as in \eqref{eq:defnu}.
\end{assumption}

This condition imposes a polynomial lower bound on the sum of the ridge regularization and a covariance-dependent spectral term, and it is common in related work \citep{misiakiewicz2024non, defilippis2024dimension,wu2026improved}. We note that the assumption is satisfied under the power-law spectra considered in Section \ref{sec:synthetic-scaling}.

\begin{theorem}[Deterministic equivalent for mixed ridge]
\label{thm:deteq-risk}
Under Assumption \ref{asm:distr}, assume $\bC_k$ commute with each other.
Let $K$ be fixed, $\lambda > 0$ satisfy Assumption \ref{asm:nu},
and 
$\pi^*$ be the deterministic test mixture weights. Define $
    \bC_{\pi^*} = \sum_{k=1}^K \pi_k^*\bC_k,  \bA_* = \btheta_*\btheta_*^\top.
$
Then $\bL$ is invertible and, for any $D > 0$, with probability at least $1-\sum_{i=1}^K n_i^{-D}$ over the training designs,
\begin{equation}\label{eq:deteqapprx}    
    |\sR(\widehat \theta) - \overline\sR| \leq C_{D, c_0}\left(\lambda^2 \eps_2 V_{\bA_*} + V_{\bC_{\pi^*}}\sum_{i=1}^K \sigma_{\eps_i}^2\eps_{3,i}\right),
\end{equation}
\begin{align*}
\overline\sR =\sum_{k=1}^K\pi_k^*\left(\overline\sB_k+\overline\sV_k\right), \,\, \overline\sB_k&=\lambda^2\left[\left\langle\btheta_*,\overline\bG\bC_k\overline\bG\btheta_*\right\rangle+\tau_{\btheta_* \btheta_*^\top}^\top\bL^{-1}\tau_{\bC_k}\right],\,\,
\overline\sV_k=\sum_{i=1}^K\sigma_{\eps_i}^2\,\tau_{\bC_k}^\top\bL^{-1}\be_i,
\end{align*}
with $\be_i$ the $i$-th unit vector in $\mathbb{R}^K$, $\nu$ defined as in \eqref{eq:defnu}, 
$e_K = \left(\sum_{i=1}^K n_i^{-1}\right)^{1/2}$ and
\begin{equation*}
  V_\bA \hspace{-.1em}=\hspace{-.1em} \sum_{i=1}^{K} n_i\tau_\bA[i],\,\,  \eps_2 \hspace{-.1em}= e_K^3(\nu^{14}+\nu^8\log^{4\eta+3/2}(N)), \,\, 
        \eps_{3,i} = e_K(e_K^2\nu^{14}+\nu^8+\nu^5\log^{3\eta+3/2}(N)).
\end{equation*}
\end{theorem}
The proof (Appendix \ref{sec:pfthm2}) relies on computing the deterministic equivalents of functionals including $\Tr(\bA\bG)$, $\Tr(\bA\bG\bC_k\bG)$, and $\Tr(\bA\bG\bX_k^\top\bX_k\bG)$, where $\bA$ is a fixed PSD matrix and $\bG$ is the resolvent. The argument follow a similar  decomposition and leave-one-out strategy as \cite{misiakiewicz2024non,defilippis2024dimension,wu2026improved}. Here, we extend their analyses to heterogeneous data mixtures, by controlling a coupled family of covariance-weighted resolvent traces.

The error bound in the RHS of \eqref{eq:deteqapprx} is $O(e_K),$ which is additive as in \citep{wu2026improved}, due to asymmetric terms. Under power-law spectra and for optimal ridge regularization, we further show in Lemma \ref{lem:vanishing-approx} that the error bound is $o(\overline{R})$. More broadly, we conjecture that a multiplicative guarantee holds for a range of parameters, leaving this technical problem as a future direction.  We also note that Theorem \ref{thm:deteq-risk} requires all the covariances to commute, and conjecture this to be a technical requirement as well: we expect the deterministic equivalent in Theorem \ref{thm:deteq-risk} to still hold for non-commutative covariances, possibly at the cost of a worse additive error. We refer to Figure \ref{fig:cifar-imagenet} (discussed in Section \ref{sec:num}) for an empirical demonstration of the validity of the deterministic equivalent predictions, well beyond our technical assumptions.






\section{Scaling laws for data mixtures}
\label{sec:synthetic-scaling}

We now specialize to $K=2$ domains and to a test distribution equal to the first domain, $\pi^*=\be_1$. Domain~1 is the \emph{target} and domain~2 the \emph{auxiliary} domain. Assuming $\bC_1$, $\bC_2$, and $\bS$ commute, and under a standard power-law decay condition on their eigenvalues, we characterize both the minimax rates and the rates achieved by ridge regression.


\paragraph{Power-law model.} Without loss of generality, we consider $\bC_1$ and $\bC_2$ diagonal and take $d=\infty$. We further assume aligned power-law spectra:
\begin{equation}
[\bC_1]_{kk}=k^{-\alpha_1},\qquad [\bC_2]_{kk}=k^{-\alpha_2},\qquad n_1=n,\qquad n_2=\lfloor n^{\gamma_2}\rfloor,\qquad \sigma_{\eps_1}^2,\sigma_{\eps_2}^2=\Theta(1),
\label{eq:ridge-power-model}
\end{equation}
with $\alpha_1,\alpha_2>1$ and $\gamma_2>0$, and set $\delta = \alpha_1 - \alpha_2$. A positive $\delta$ means that the auxiliary spectrum has heavier tails: relative to the target, the auxiliary domain puts more mass on high-frequency directions. For the minimax rates, we consider the parameter class
\begin{equation*}
\bTheta=\big\{\btheta:\|\bC_1^{1/2-r_1}\btheta\|_2\leq R,\ \|\bC_2^{1/2-r_2}\btheta\|_2\leq R\big\},\qquad s=\max\Big\{\alpha_1r_1,\ \alpha_2r_2+\frac\delta2\Big\},
\end{equation*}
which imposes a source condition with exponent $r_i>0$ in each domain \citep{caponnetto2007optimal}. Up to a change of radius, $\bTheta$ is the single ellipsoid $\{\btheta:\sum_k k^{2s-\alpha_1}\btheta [k]^2\leq R^2\}$ (Appendix~\ref{apx:pf-thm:minimax}), so $s$ is the regularity of the target in the target geometry, and $2s>\delta$. For mixed ridge regression, we consider the fixed signal $\btheta_*[k]=k^{-\beta}$, $\beta>1/2$, which belongs to the ellipsoid above for every $s'<s=(\alpha_1+2\beta-1)/2$ but not for $s'=s$. In both cases, $[\bC_1^{1/2}\btheta_*]_k\asymp k^{-(1+2s)/2}$: the target-only problem is the standard source-capacity setting with capacity exponent $\alpha_1$.



\paragraph{Minimax rates.}
Let $\sR^*(n_1,n_2)=\inf_{\widehat\btheta}\sup_{\btheta_*\in\bTheta}\E_{\btheta_*}\|\bC_1^{1/2}(\widehat\btheta-\btheta_*)\|_2^2$ be the minimax target risk, with $n_i=0$ meaning that dataset $i$ is unavailable. Define
\begin{equation*}
\Gamma_{\rm tar}=\frac{2s}{1+2s},\qquad
\Gamma_{\rm aux} =\begin{cases}\dfrac{2s\gamma_2}{1+2s-\delta},&\delta<1,\\[6pt]\gamma_2,&\delta\geq1,\end{cases}\qquad
\gamma_{\rm c}=1-\frac{\delta}{1+2s}.
\end{equation*}



\begin{theorem}[Minimax scaling law]
\label{thm:minimax}
Under Assumption~\ref{asm:distr} and the model \eqref{eq:ridge-power-model}, the risks of the individual datasets satisfy $\sR^*(n,0)\asymp n^{-\Gamma_{\rm tar}}$, $\sR^*(0,n_2)\asymp n^{-\Gamma_{\rm aux}}$ if $\delta \neq 1$, and $\sR^*(0,n_2)\asymp n^{-\Gamma_{\rm aux}} \log n$ if $\delta = 1$. Furthermore, the risk for the data mixture satisfies
\begin{equation*}
 \sR^*(n,n_2)\asymp\begin{cases}
\sR^*(n,0),&\gamma_2\leq\gamma_{\rm c},\\
\sR^*(0,n_2) ,&\gamma_2>\gamma_{\rm c},\ \delta \leq 1,\;\text{ or }\; \gamma_2\geq1,\ \delta>1,\\
n^{-(\delta-1+\gamma_2)/\delta},&\gamma_{\rm c}<\gamma_2<1,\ \delta>1.
\end{cases}
\label{eq:minimax-rates}
\end{equation*}
\end{theorem}




The proof (Appendix~\ref{apx:pf-thm:minimax}) applies Theorem~\ref{thm:minimax-ellipsoid} with $\bS=\bC_1^{1-2s/\alpha_1}$, for which $E_0=O(1)$, and evaluates the truncation formula \eqref{eq:truncation-intuition}, which here reads
\begin{equation}
\sR^*(n,n_2)\asymp\inf_{m\geq0}\Big\{m^{-2s}+\sum_{j\leq m}\frac1{I_j}\Big\},\qquad I_j\asymp n+n_2j^{\delta}.
\label{eq:minimax-truncation}
\end{equation}
The linear estimator in Appendix~\ref{apx:pf-thm:minimax-ellipsoid} attains this rate but requires knowledge of the population covariances. We will next identify when ridge regression matches this rate without such knowledge.


Let us further comment on Theorem \ref{thm:minimax}. In  \eqref{eq:minimax-truncation}, $I_j^{-1}$ can be regarded as the error of fitting coordinate $j$ of the signal. When $\delta > 0,$ $I_j$ is dominated by $n_2 j^{\delta} \gg n$ above the crossover index $j_{\rm c}:=(n/n_2)^{1/\delta}=n^{(1-\gamma_2)/\delta}$ where the error mainly comes from the auxiliary data. In contrast, below the crossover index $j_{\rm c}$, $I_j\asymp n$. We discuss three cases: 
\begin{itemize}
[itemsep=2pt,leftmargin=*]
    \item[(1)] \emph{Too few auxiliary samples ($\gamma_2\leq\gamma_{\rm c}$).} The target-only estimator learns $m_{\rm tar}=n^{1/(1+2s)}$ coordinates, and $\gamma_2\leq\gamma_{\rm c}$ is exactly $j_{\rm c}\geq m_{\rm tar}$. Thus, auxiliary data only helps with coordinates already well estimated in target data, and the mixed minimax rate does not improve on the target-only rate. 

     \item[(2)] \emph{Light auxiliary tails ($\delta\leq1$).} Beyond $j_{\rm c}$, the estimation cost $\sum_{j>j_{\rm c}}(n_2j^{\delta})^{-1}$ is not summable and the variance is dominated by the highest coordinates being learnt, as if only auxiliary data was available. Thus, data mixing does not improve the minimax rate.

     \item[(3)] \emph{Heavy auxiliary tails ($\delta>1$).} Here $\sum_{j>j_{\rm c}}(n_2j^{\delta})^{-1}\asymp j_{\rm c}/n$ is summable: once auxiliary data takes over, all remaining coordinates are learned at a total cost comparable to that of the first $j_{\rm c}$ ones. Thus, the mixed minimax rate can improve upon both target-only and auxiliary-only rates.

\end{itemize}

In summary, the mixed minimax rate improves on both single-dataset minimax rates if and only if $\delta>1$ and $\gamma_2\in(\gamma_{\rm c},1)$. For $\delta>0$, one can disregard the auxiliary data when $\gamma_2\leq\gamma_{\rm c}$, and disregard the target data when $\gamma_2\geq1$, without worsening the rate.


\begin{remark}[Positive distribution shift]
    At equal sample sizes ($\gamma_2=1$) and $\delta>0$, the auxiliary-only exponent $\Gamma_{\rm aux}(1)=2s/(1+2s-\delta)$ for $\delta<1$, or $1$ for $\delta\geq1$, is strictly larger than $\Gamma_{\rm tar}$: $n$ samples from the shifted distribution are more informative about the target than $n$ target samples, because they put larger weight on high-frequency directions. This is an example of `positive distribution shift', the observation that training over a different data distribution can help improve the test performance \citep{medvedev2026positive,medvedev2026shift}. For $\delta<0$, by contrast, $\gamma_{\rm c}>1$: polynomially more auxiliary than target samples are needed before the rate changes.
\end{remark}

\paragraph{Scaling law of mixed ridge regression.}

We now turn to the estimator in \eqref{eq:mixed-ridge} with the fixed signal $\btheta_*[k]=k^{-\beta}$, $s=(\alpha_1+2\beta-1)/2$, and $\delta = \alpha_1 - \alpha_2$. Let $\overline\sR_1(\lambda)$ denote the deterministic equivalent of the target risk given by Theorem~\ref{thm:deteq-risk}, and $\overline\sR_1^*=\inf_{\lambda>0}\overline\sR_1(\lambda)$ the risk under optimal regularization. Define $(x)_+:=\max\{x, 0\}$, and consider the truncated coefficients
\begin{equation*}
a=\min\{s,\alpha_1\},\quad b=\min\{s,\alpha_2\},\quad \kappa=
    \begin{cases}
        (a-b)/\delta,&\delta\ne0,\\
        0,&\delta=0,
    \end{cases}
    \quad
    \gamma_c^{\rm ridge}
    =1-\frac{\delta}{1+2\max\{a, b\}}
\end{equation*}
and define $$\Gamma_{\rm tar}^{\rm ridge}=\frac{2a}{1+2a},
\quad
\Gamma_{\rm aux}^{\rm ridge}=\frac{2b\gamma_2}{2b+(1-\delta)_+}.$$


\begin{theorem}[Ridge scaling law]
\label{thm:scaling-optimal-ridge}
Under Assumption~\ref{asm:distr} and the model \eqref{eq:ridge-power-model}, the optimally regularized target risks of individual datasets satisfy $\overline\sR_1^*(n,0)=\widetilde\Theta(n^{-\Gamma_{\rm tar}^{\rm ridge}})$ and $\overline\sR_1^*(0,n_2)=\widetilde\Theta(n^{-\Gamma_{\rm aux}^{\rm ridge}})$, with $\widetilde\Theta(\cdot)$ omitting logarithmic factors. Furthermore, the risk for the data mixture satisfies $\overline\sR_1^*(n,n_2)=\widetilde\Theta(n^{-\Gamma}),$
where
\begin{equation}
\Gamma=
\begin{cases}
\displaystyle
\Gamma_{\rm tar}^{\rm ridge}
+\frac{2\kappa(\gamma_2-1)_+}{1+2a},
&\gamma_2\leq\gamma_c^{\rm ridge},\\[10pt]
\displaystyle
\Gamma_{\rm aux}^{\rm ridge}
+\frac{2(1-\delta)_+\kappa(1-\gamma_2)_+}{2b+(1-\delta)_+},
&\begin{array}{l}
\gamma_2>\gamma_c^{\rm ridge},\ \delta\leq1,\ \text{or}\\
\gamma_2\geq1,\ \delta>1,
\end{array}\\[10pt]
\displaystyle
\frac{\delta-1+\gamma_2}{\delta},
&\gamma_c^{\rm ridge}<\gamma_2<1,\ \delta>1.
\end{cases}
\label{eq:optimal-ridge-polynomial}
\end{equation}
\end{theorem}

The above rate is evaluated on the deterministic equivalent $\overline\sR$ of Theorem \ref{thm:deteq-risk}. Lemma \ref{lem:vanishing-approx} below (proved in Appendix \ref{apx:pf-lem:vanishing-approx}) shows that, for the optimal ridge, the error bound in \eqref{eq:deteqapprx} is vanishing (i.e., $o(\overline\sR)$) and, thus, the same rate is achieved by the excess test risk $\sR(\widehat \theta)$ of the mixed ridge estimator. 

\begin{lemma}[Vanishing approximation error for optimal ridge.] 
\label{lem:vanishing-approx}
In the setting of Theorem~\ref{thm:scaling-optimal-ridge}, there exists $\lambda_* $ such that $\overline\sR_1(\lambda_*) = \widetilde \Theta(\overline{\sR}_1^*(n,n_2))$, and $|\sR_1(\widehat \btheta_{\lambda_*}) -  \overline{\sR}_1(\lambda_*)| = o(\overline{\sR}_1^*(n,n_2)),$ with probability at least $1-\sum_{i=1}^2 n_i^{-D}.$
\end{lemma}

Comparing the rates for ridge of Theorem~\ref{thm:scaling-optimal-ridge} with the minimax ones of Theorem~\ref{thm:minimax} identifies the regimes in which the deterministic equivalent is minimax optimal. 

\begin{corollary}[Minimax optimality of ridge]
\label{cor:ridge-unrestricted-comparison}
In the setting of Theorem~\ref{thm:scaling-optimal-ridge}, we have that $\overline{\sR}_1^*(n,n_2)
=\widetilde{\Theta}\bigl(\sR^*(n,n_2)\bigr)$
if and only if one of the following holds: (i) $\gamma_2\leq\gamma_{\rm c}$ and $s\leq\alpha_1$; (ii) $\gamma_2\geq\gamma_{\rm c}$ and either
$s\leq\alpha_2$ or $\delta\geq1$.
In all other cases, $\overline{\sR}_1^*(n,n_2)$ is polynomially
larger than $\sR^*(n,n_2)$.
\end{corollary}


In particular, \textbf{whenever the mixed minimax rate is strictly better than the minimax rate of either dataset alone} (i.e., $\delta>1$ and $\gamma_{\rm c}<\gamma_2<1$), \textbf{ridge regression achieves the mixed minimax rate}. 
This conclusion holds for every regularity $s$, including $s>\alpha_1$,
where target-only ridge is suboptimal because of saturation.
Thus, a sufficiently heavier auxiliary spectral tail allows ridge
to attain the improved minimax exponent using only
$n_2=o(n)$ auxiliary samples. We note that \emph{(i)} $\gamma_2>\gamma_{\rm c}$ specifies how many auxiliary
samples are needed to improve on the target-only risk, and \emph{(ii)}
$\gamma_2<1$ ensures that the target samples still improve on the
auxiliary-only risk. This establishes a regime in terms of sample sizes where data mixtures are particularly effective. 

\begin{remark}[Saturation induced by mixing]
Ridge regression cannot attain the minimax rate when the source exponent $s$ exceeds the capacity exponent of the spectrum it fits. This is the classical saturation of Tikhonov regularization \citep{bauer2007regularization,li2023saturation}, and it appears in the mixed regression setting of Theorem \ref{thm:scaling-optimal-ridge} through the truncations $a=\min\{s,\alpha_1\}$ and $b=\min\{s,\alpha_2\}$. 
For $\alpha_2<s\leq\alpha_1$, $0<\delta<1$, ridge is rate optimal on the target data alone, and when  $\gamma_2>\gamma_{\rm c}$, auxiliary data still improves its exponent, but it no longer attains the mixed minimax rate. When $\delta>1$, saturation is harmless as soon as $\gamma_2\geq\gamma_{\rm c}$. For $s>\alpha_1$, ridge is already suboptimal on the target data alone, and when $\delta>0$ the threshold $\gamma_{\rm c}^{\rm ridge}$ lies below $\gamma_{\rm c}$: auxiliary data partially compensates for saturation before the minimax rate itself changes. We expect iterated Tikhonov or early stopping to remove these truncations \citep{bauer2007regularization}.


\end{remark}






\section{Numerical experiments}\label{sec:num}

\paragraph{Linear models.} We consider a linear model following the setup of Section \ref{sec:setup}.
Figure \ref{fig:powerlaw} indicates that the empirical risk of the mixed ridge estimator matches the deterministic equivalent of Theorem \ref{thm:deteq-risk} and the scaling law predicted by  Theorem \ref{thm:scaling-optimal-ridge}. Furthermore, in panel (a), ridge regression achieves the mixed minimax rate and strictly improves upon the minimax rates of either dataset alone; in contrast, in panels (b) and (c), no improvement can be obtained and the rate of mixed ridge respectively coincides with the auxiliary-only and target-only rate. Next, we evaluate our deterministic equivalent on real image features from CIFAR-10 and ImageNet-100.
The results in Figure \ref{fig:cifar-imagenet} indicate that the validity of the deterministic equivalent of Theorem \ref{thm:deteq-risk} goes well beyond the technical assumptions made therein. 

\begin{figure}[t]
\centering
\captionsetup[subfigure]{skip=2pt}

\begin{subfigure}[t]{0.328\linewidth}
    \centering
    \includegraphics[
        width=\linewidth,
        trim=5 3 5 3,
        clip
    ]{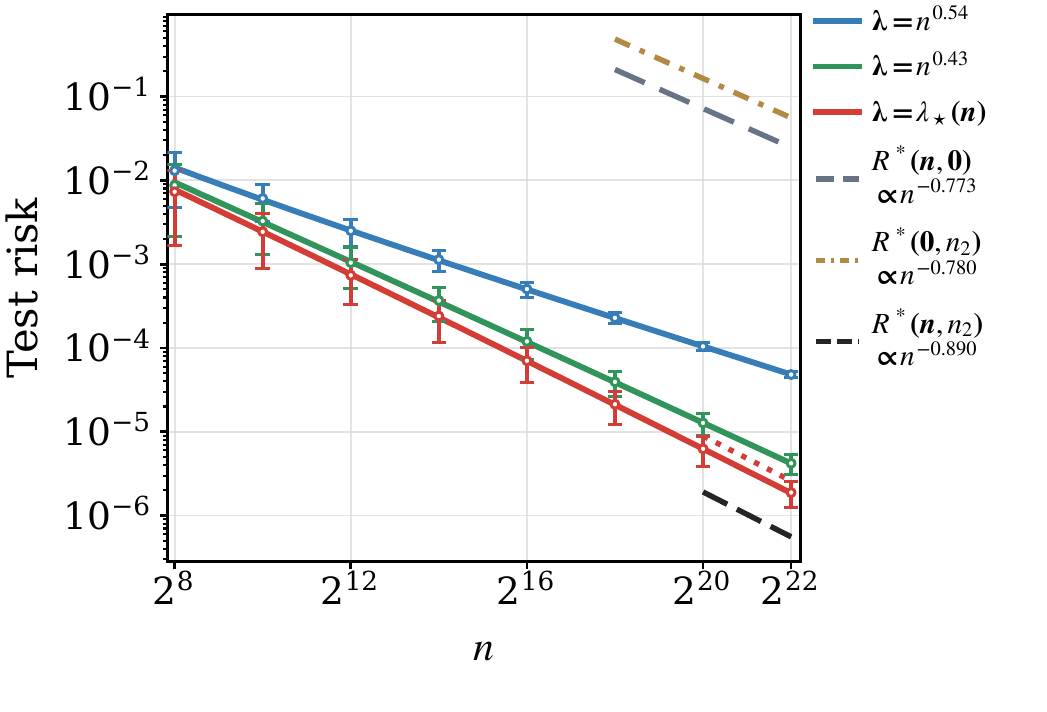}
    \caption{$\gamma_2=0.78$}
    \label{fig:powerlaw-improve-both}
\end{subfigure}\hfill%
\begin{subfigure}[t]{0.328\linewidth}
    \centering
    \includegraphics[
        width=\linewidth,
        trim=5 3 5 3,
        clip
    ]{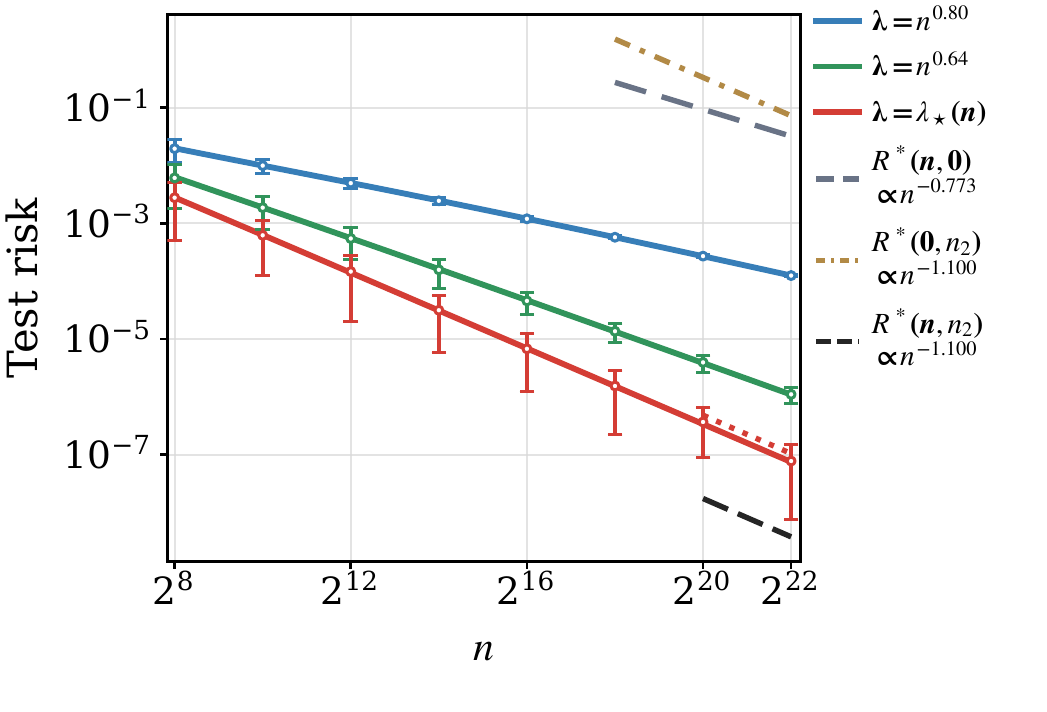}
    \caption{$\gamma_2=1.1$}
    \label{fig:powerlaw-aux}
\end{subfigure}\hfill%
\begin{subfigure}[t]{0.328\linewidth}
    \centering
    \includegraphics[
        width=\linewidth,
        trim=5 3 5 3,
        clip
    ]{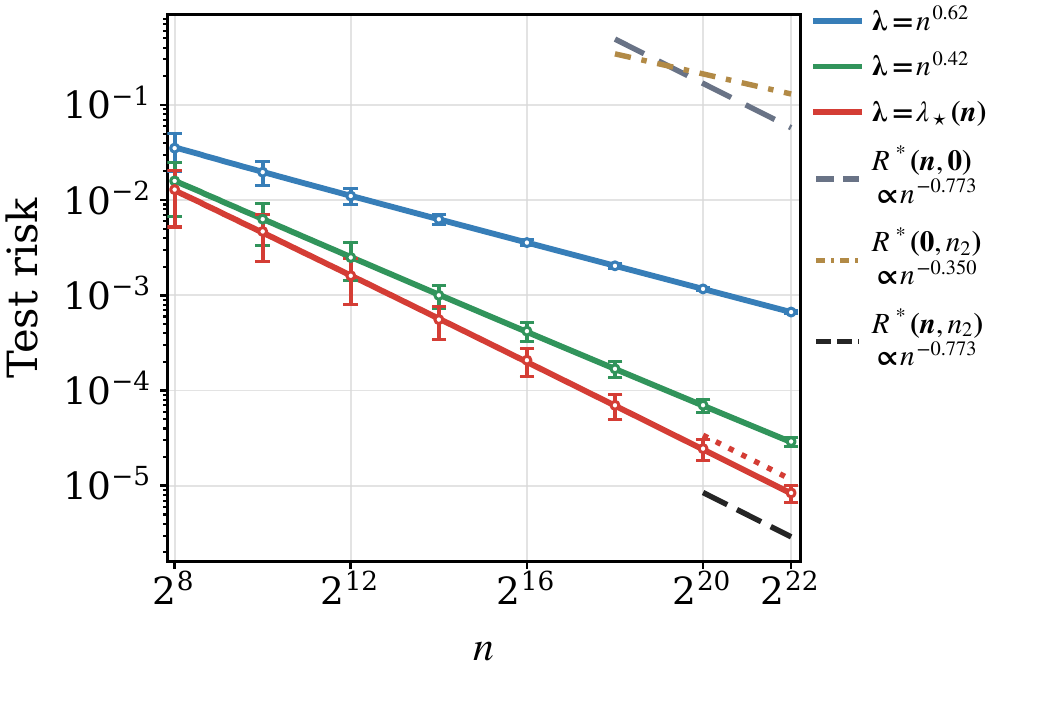}
    \caption{$\gamma_2=0.35$}
    \label{fig:powerlaw-tar}
\end{subfigure}

\caption{Test error (dots) of mixed ridge regression under power-law covariance, together with its corresponding deterministic equivalent from Theorem \ref{thm:deteq-risk} (solid lines). The \textcolor{targetcolor}{grey}, \textcolor{sourcecolor}{golden} and black dash lines are the minimax rates of target-only, auxiliary-only and mixture datasets from Theorem
\ref{thm:minimax}. The \textcolor{ridgeblue}{blue}, \textcolor{ridgegreen}{green} and \textcolor{ridgeoptimal}{red} lines correspond to different regularizations with the optimal one in \textcolor{ridgeoptimal}{red}. Here $\bx_{i} \sim \cN(0, \bC_i),$ with $\bC_i$ as in \eqref{eq:ridge-power-model}. We set $\alpha_1 = 3.2$, $\alpha_2 = 1.2$, $\beta=0.6$, and $d=512$. }
\label{fig:powerlaw}
\end{figure}

\begin{figure}[t]
\centering
\captionsetup[subfigure]{skip=2pt}

\centering
\begin{subfigure}[t]{0.38\linewidth}
    \centering
    \includegraphics[width=\linewidth]{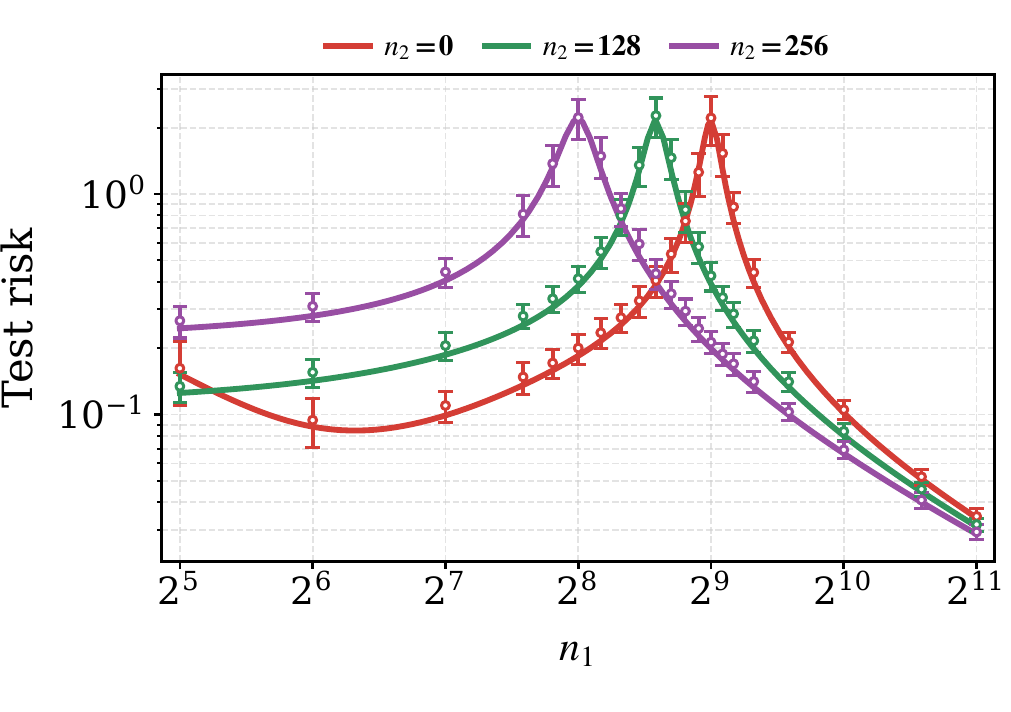}
    \caption{ImageNet-100}
    \label{fig:deteq-imagenet}
\end{subfigure}\hspace{4em}
\begin{subfigure}[t]{0.38\linewidth}
    \centering
    \includegraphics[width=\linewidth]{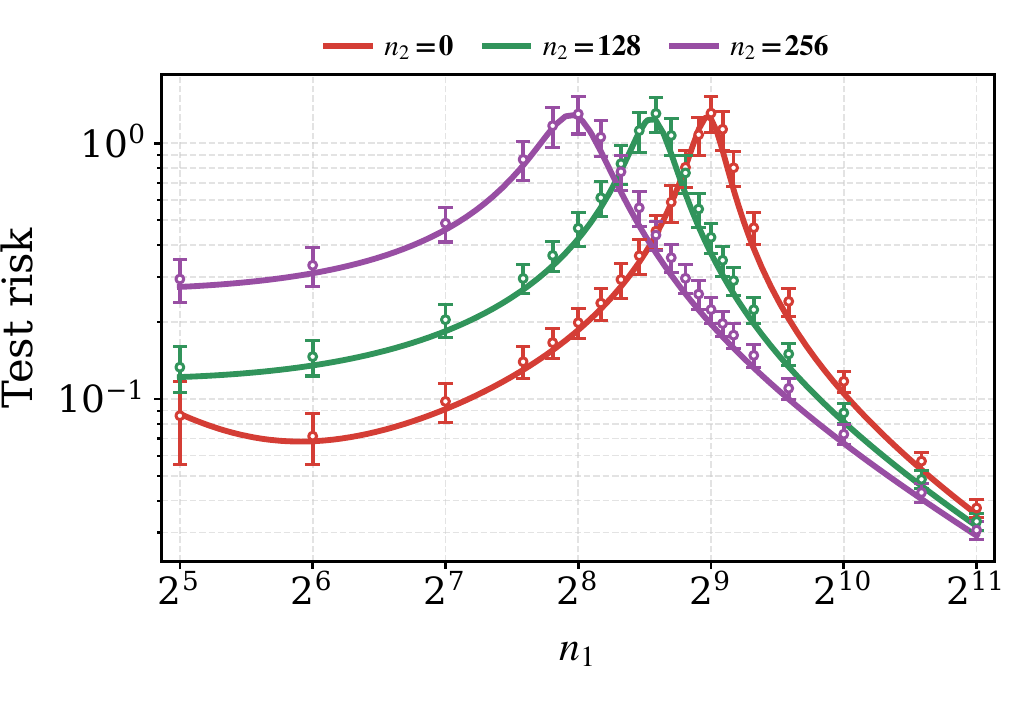}
    \caption{CIFAR-10}
    \label{fig:deteq-cifar}
\end{subfigure}
\caption{Test error (dots) of mixed ridge regression, together with its corresponding deterministic equivalent (solid lines). For CIFAR-10, features are extracted via a pretrained CLIP ViT-B/32 encoder, classes $0$--$4$ form the target domain and classes
$5$--$9$ form the auxiliary domain.
For ImageNet-100, features are extracted via a pretrained ResNet-18 encoder, and the corresponding
class splits are $0$--$49$ and $50$--$99$. In both settings, we take $d=512$ and generate labels according to
$y=\bx^\top\btheta_*+\xi$, with $\xi\sim\mathcal{N}(0,0.1)$.
The shared $\btheta_*$ is supported on the top-$10$ 
eigenvectors of 
$( \bC_1+\bC_2)/2$, with random signs and coefficient
magnitudes proportional to $j^{-1/2}$ for $j=1,\ldots,10$.
  } 
\label{fig:cifar-imagenet}
\end{figure}

\paragraph{Language modeling.} Finally, we consider autoregressive language
modeling. We pretrain an 81.5M-parameter GPT-2-style model
\citep{radford2019language} via nanoGPT
\citep{karpathy2022nanogpt} on the SlimPajama dataset
\citep{cerebras2023slimpajama}. We use $n$ tokens from the \texttt{StackExchange}
domain (target domain) and $n_2=
    4
    \left(
        \frac{n}{4\,\mathrm{Mi}}
    \right)^{\gamma_2}$ tokens from a fixed mixture of the remaining six
domains (auxiliary domain). Full experimental details are in
Appendix~\ref{apx:detail-experiments}. 
Figure~\ref{fig:lang-scalinglaw} (panels (a)-(c)) shows that the improvement in the scaling law crucially relies on the relative growth rate of
the token budget of the auxiliary domain: for $\gamma_2=1.25$, the scaling law of the data mixture improves upon that of either dataset alone; however, this improvement is not noticeable anymore already when $\gamma_2=0.75$ or $\gamma_2=1.75$. This qualitatively supports our theoretical
finding that data mixtures improve scaling only within an appropriate
relative-growth regime. Finally, the token-frequency distribution in panel (d) shows that the auxiliary domain
places relatively more probability mass on tokens that are infrequent
in the target domain. This provides a token-level analogue of the heavier-tailed
auxiliary covariance in our linear theory: the auxiliary data offer
greater coverage of directions/tokens that are weakly represented
in the target distribution. We also find that the scaling law of the data mixture does not improve when the token frequencies of the two datasets are too close with each other---an observation that agrees with our theory (see Figure \ref{fig:lang-scalinglaw-close} in Appendix \ref{apx:non-improve-close}).

\begin{figure}[t]
\centering
\captionsetup[subfigure]{skip=2pt}
\centering

\begin{subfigure}[t]{0.4\linewidth}
    \centering
    \includegraphics[width=\linewidth,trim=6 4 6 4,clip]{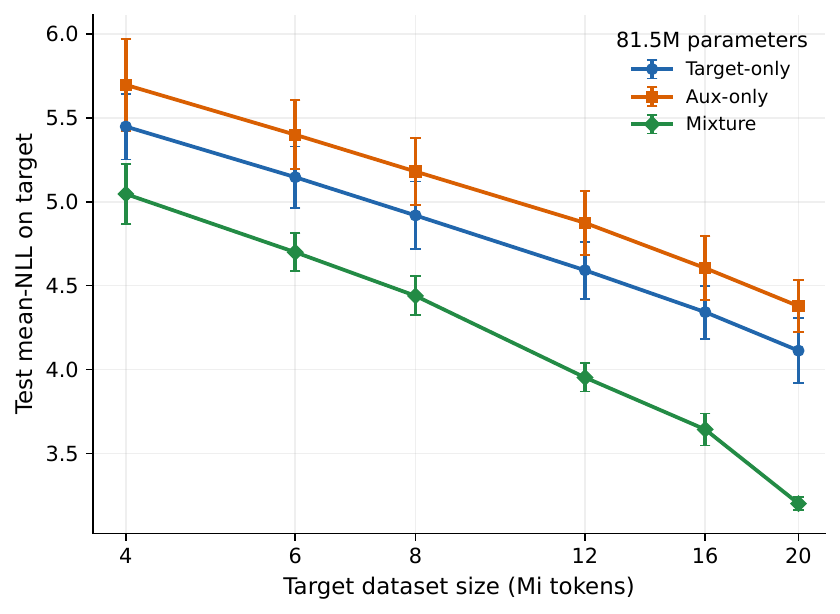}
    \caption{$\gamma_2=1.25$}
    \label{fig:result2-b125}
\end{subfigure}\hspace{2em}
\begin{subfigure}[t]{0.4\linewidth}
    \centering
    \includegraphics[width=\linewidth,trim=6 4 6 4,clip]{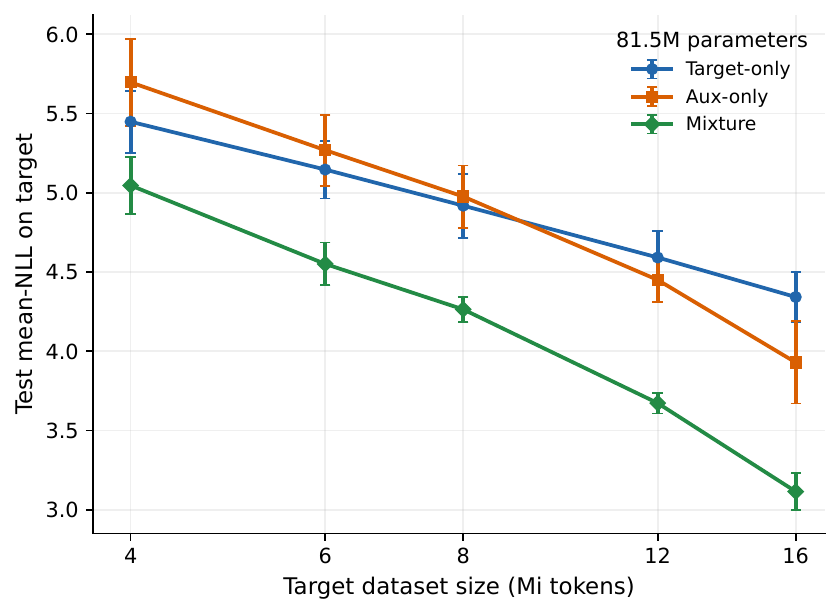}
    \caption{$\gamma_2=1.75$}
    \label{fig:result3-b175}
\end{subfigure}

\begin{subfigure}[t]{0.4\linewidth}
    \centering
    \includegraphics[width=\linewidth,trim=6 4 6 4,clip]{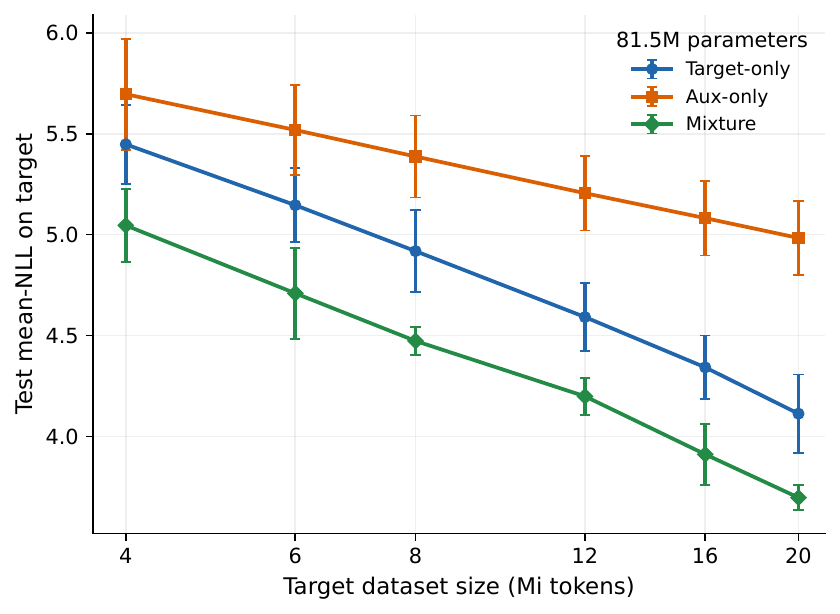}
    \caption{$\gamma_2=0.75$}
    \label{fig:result3-b075}
\end{subfigure}\hspace{2em}
\begin{subfigure}[t]{0.4\linewidth}
    \centering
    \includegraphics[width=\linewidth,trim=6 4 6 4,clip]{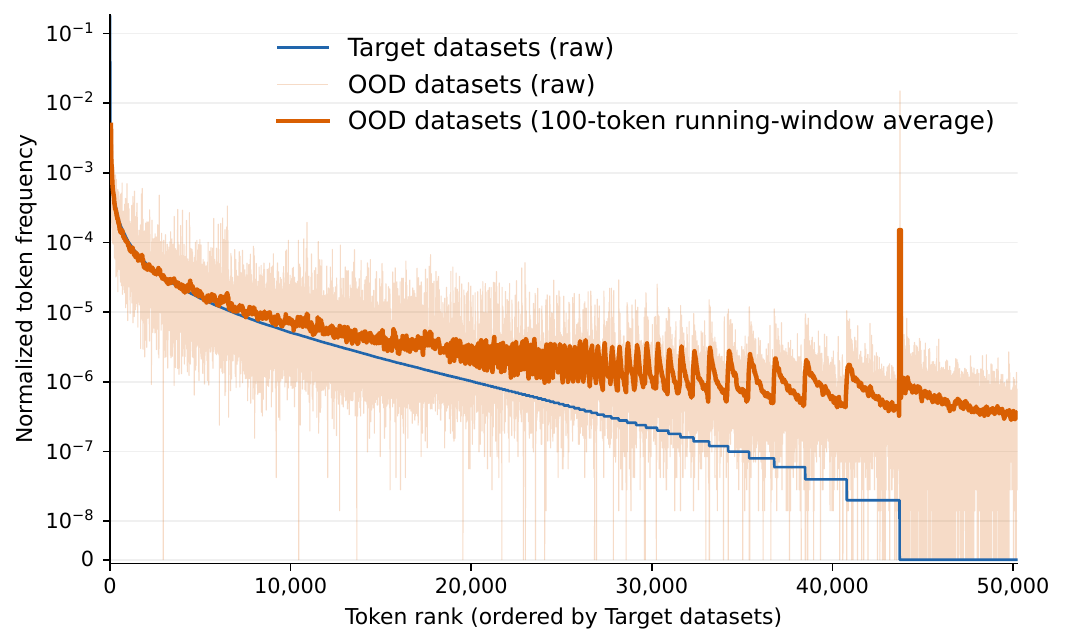}
    \caption{Token-frequency distributions}
    \label{fig:token-freq}
\end{subfigure}

\caption{
Target-domain test negative log-likelihood (NLL) scaling of an 81.5M-parameter GPT-2-style language model
trained from scratch on SlimPajama. Panels
(a)--(c) report mean NLL on a fixed target-domain test set as a function
of the reference target budget $n$. We use
$n\in\{4,6,8,12,16,20\}\,\mathrm{Mi}$, except for the
$\gamma_2=1.75$ experiment that uses $n\leq16\,\mathrm{Mi}$. Panel (d) compares the normalized
target and auxiliary token frequencies. Token types are
ordered by decreasing target-domain frequency. The light auxiliary
curve shows the raw frequencies, and the darker, thicker curve shows
their 100-token running-window average.
}
\label{fig:lang-scalinglaw}
\end{figure}

\section{Conclusion}


We study high-dimensional regression on data mixtures that share a regression function but differ in their covariances and noise levels. For general covariances, we characterize the minimax risk, and for commutative covariances, we derive deterministic equivalents for the test error of a mixed ridge estimator. Specializing these results to target and auxiliary domains with aligned power-law spectra, we derive scaling laws, thus establishing conditions under which mixing data is especially helpful. In particular, we show that, whenever the mixed minimax rate improves on the rate of either dataset alone, ridge regression is minimax optimal. This is the case when \emph{(i)} the auxiliary spectrum is heavy tailed ($\delta>1$), i.e., the two datasets cover complementary parts of the spectrum, and \emph{(ii)} the auxiliary sample size has an appropriate growth rate ($\gamma_{\rm c}<\gamma_2<1$). Our experiments on transformers trained on a mixture of text corpora display a phenomenology that qualitatively agrees with the theory.

Our work opens the door to several interesting future directions. 
A natural next step is to determine whether gradient descent methods attain the mixed minimax rates when ridge saturates, as recently shown for a single dataset by \cite{wu2025risk}.
Considering random-feature models \citep{defilippis2024dimension,bordelon2024dynamical}, or features trained via a step of GD \citep{ba2022highdimensionalasymptoticsfeaturelearning,pmlr-v235-cui24d}, would then identify how model size must scale with data budget to preserve the improvement in the scaling law.
Finally, it would be valuable to jointly optimize domain proportions under data and compute constraints, as well as to characterize mixture transfer across model scales, motivated
by recent theory on hyperparameter transfer \citep{ghosh2026understanding}.



\section*{Acknowledgments and AI use}

DW and MM were funded in part by the Austrian Science Fund (FWF) 10.55776/COE12. For the purpose of open access, the authors have applied a CC BY public copyright license to any Author Accepted Manuscript version arising from this submission. MM was partially funded by the European Union (ERC, INF2, project number 101161364). Views and opinions expressed are however those of the author(s) only and do not necessarily reflect those of the European Union or the European Research Council Executive Agency. Neither the European Union nor the granting authority can be held responsible for them.

We used generative AI tools to polish the writing, implement straightforward parts of the code, fill in proof details after we had identified the key methods and arguments, check proofs for correctness, and improve the coverage of related work. We did not use generative AI tools to develop the paper's core ideas and define the problem setting. We take full responsibility for the final content of this work, including all text, claims, code, and other artifacts produced with the assistance of generative AI. 


\bibliography{references}

\appendix





\section{Proof of Theorem \ref{thm:minimax-ellipsoid}}
\label{apx:minimax-ellipsoid}

First, define the normalized weight $\bu = \bS^{1/2} \btheta, \widehat \bu =\bS^{1/2} \widehat \btheta,$ and given any PSD matrix $\bN,$ we construct an auxiliary Gaussian sequence model \citep{johnstone2002function}, \begin{equation}
\label{def:seq-ellipsoid}
\by_{\bN, \rm{seq}}=\bN^{1/2}\bu+\bg, \quad \bg\sim\mathcal N(0,\bI),
\end{equation} where $\bN \in \R^{d \times d}$ and $\bu, \bg,\by_{\bN, \rm{seq}} \in \R^d. $

Define the unrestricted and linear minimax errors of the Gaussian sequence model as 
\begin{align*}
\mathrm R_{*, \rm{seq}}(\bN)&=\inf_{\widehat\bu}\sup_{\|\bu\|_2\leq R}\E_{\bu}\|\bQ^{1/2}(\widehat\bu(\by_{\bN, \rm{seq}})-\bu)\|_2^2,\\
\mathrm R_{*,\mathrm{seq}}^{\mathrm{lin}}(\bN)&=\inf_{\bB}\sup_{\|\bu\|_2\leq R}\E_{\bu}\|\bQ^{1/2}(\bB\by_{\bN, \rm{seq}}-\bu)\|_2^2,
\end{align*} where we recall $\bQ = \sum_{k=1}^K \pi_k^* \bH_k.$

We first have the following variational characterization of the minimax error of the Gaussian sequence model \eqref{def:seq-ellipsoid}. 
\begin{theorem}[Exact Gaussian linear minimax error]
\label{thm:ellipsoid-gaussian-seq}
For the Gaussian sequence model \eqref{def:seq-ellipsoid} defined by any PSD matrix $\bN,$
\begin{equation}
 \pi^{-2} \mathcal L(\bN)\leq \mathrm R_{*,\mathrm{seq}}(\bN) \leq  \mathrm R_{*,\mathrm{seq}}^{\mathrm{lin}}(\bN) =  \mathcal L(\bN),
\label{eq:ellipsoid-linear-identity}
\end{equation}
where $L(\bN)=\sup_{\substack{\bA\succeq0\\\Tr(\bA)\leq R^2}}\Tr\!\left[\bQ\bA^{1/2}(\bI+\bA^{1/2}\bN\bA^{1/2})^{-1}\bA^{1/2}\right].$
\end{theorem}

We prove Theorem \ref{thm:ellipsoid-gaussian-seq} in Appendix \ref{apx:pf-thm:ellipsoid-gaussian-seq}, and next in Appendix \ref{apx:pf-thm:minimax-ellipsoid}, we explain the way to extend the proof in Theorem \ref{thm:ellipsoid-gaussian-seq} to the original regression problem in Section \ref{sec:setup}, which proves Theorem \ref{thm:minimax-ellipsoid}. We first prove these result for $d<\infty$. The extension to the infinite-dimensional case is treated in
Appendix~\ref{apx:ellipsoid-infinite}.

\subsection{Proof of Theorem \ref{thm:ellipsoid-gaussian-seq}}

\label{apx:pf-thm:ellipsoid-gaussian-seq}

\paragraph{Proof of the upper bound.}
First of all, it is clear by definition that $\mathrm R_{*,\mathrm{seq}}(\bN) \leq  \mathrm R_{*,\mathrm{seq}}^{\mathrm{lin}}(\bN),$ and we characterize $\mathrm R_{*,\mathrm{seq}}^{\mathrm{lin}}(\bN)$ in the rest of the proof. 

For any $\bB$, computing the test error  gives
\begin{equation*}
\begin{split}
    \mathrm R_{*,\mathrm{seq}}^{\rm{lin}}(\bN) = &\inf_{\bB } \sup_{\|\bu\|_2 \leq R}\E_{\bu}\|\bQ^{1/2}(\bB\by_{\bN, \rm{seq}} -\bu)\|_2^2 \\
    =& \inf_{\bB } \sup_{\|\bu\|_2 \leq R}\bu^\top(\bI-\bN^{1/2}\bB^\top)\bQ(\bI-\bB\bN^{1/2})\bu+\Tr(\bQ\bB\bB^\top).
\end{split}
\end{equation*}
Note that for every any PSD matrix $\bD$,
\begin{equation*}
\sup_{\|\bu\|_2\leq R}\bu^\top\bD\bu=R^2\|\bD\|_{\mathrm{op}}=\sup_{\substack{\bA\succeq0\\\Tr(\bA)\leq R^2}}\Tr(\bD\bA).
\end{equation*}
Consequently, we could rewrite the test error as
\begin{equation*}
F(\bB,\bA)=\Tr\!\left[\bQ\bigl((\bI-\bB\bN^{1/2})\bA(\bI-\bN^{1/2}\bB^\top)+\bB\bB^\top\bigr)\right],
\end{equation*}
which implies
\begin{equation*}
    \mathrm R_{*,\mathrm{seq}}^{\rm{lin}}(\bN) = \inf_{\bB}\sup_{\substack{\bA\succeq0\\\Tr(\bA)\leq R^2}}F(\bB,\bA).
\end{equation*}

To characterize the exact solution of the above minimax problem, we first lower bound the minimax solution by
\begin{equation*}
    \begin{split} \inf_{\bB}\sup_{\substack{\bA\succeq0\\\Tr(\bA)\leq R^2}}F(\bB,\bA) \geq \sup_{\substack{\bA\succeq0\\\Tr(\bA)\leq R^2}}\inf_{\bB}F(\bB,\bA) = \sup_{\substack{\bA\succeq0\\\Tr(\bA)\leq R^2}}F(\bB_\bA, \bA) = \cL(\bN),
    \end{split}
\end{equation*} where for any $\bA,$ we define $\bB_\bA = \argmin_{\bB}F(\bB,\bA).$ The observation is that for any $\bA$ in the admissible regime, $\bB_\bA$ is achievable and the exact form can be characterized. To see this, note that $F(\bB, \bA)$ is quadratic in $\bB$ and collecting all the terms depends on $\bB$ into a quadratic form gives \begin{equation}
F(\bB,\bA)=\Tr[\bQ\bP(\bA,\bN)]+\Tr[\bQ(\bB-\bB_{\bA})\bV_{\bA}(\bB-\bB_{\bA})^\top],
\label{eq:ellipsoid-complete-square}
\end{equation}  where we define
\begin{equation*}
\begin{split}
&\bP(\bA,\bN)=\bA^{1/2}(\bI+\bA^{1/2}\bN\bA^{1/2})^{-1}\bA^{1/2}, \\
& \bV_{\bA}=\bI+\bN^{1/2}\bA\bN^{1/2}, \quad  \bB_{\bA}=\bP(\bA,\bN)\bN^{1/2},
\end{split} 
\end{equation*} and use the fact that
\begin{align}
\bP(\bA,\bN)&=\bA-\bA\bN^{1/2}(\bI+\bN^{1/2}\bA\bN^{1/2})^{-1}\bN^{1/2}\bA,\nonumber\\
\bP(\bA,\bN)\bN^{1/2}&=\bA\bN^{1/2}(\bI+\bN^{1/2}\bA\bN^{1/2})^{-1}.
\label{eq:ellipsoid-resolvent-identities}
\end{align}

Next, note that for any given $\bB,$  the feasible regime of $\bA$ is compact and the optimum is achievable. Thus, define \begin{equation*}
    \bA_* =  \argmax_{\substack{\bA\succeq0\\\Tr(\bA)\leq R^2}}F(\bB_\bA, \bA) = \argmax_{\substack{\bA\succeq0\\\Tr(\bA)\leq R^2}}\Tr[\bQ \bP(\bA, \bN)].
\end{equation*}  It is clear by definition that $\bB_* := \bB_{\bA_*} = \argmin_{\bB} F(\bB, \bA_*).$ We then show that \begin{equation*}
    \begin{split}
        \bA_* \in \argmax_{\substack{\bA\succeq0\\\Tr(\bA)\leq R^2}}F(\bB_*, \bA),
    \end{split}
\end{equation*} which implies \begin{equation*}
    \begin{split}  \inf_{\bB}\sup_{\substack{\bA\succeq0\\\Tr(\bA)\leq R^2}}F(\bB,\bA) = \sup_{\substack{\bA\succeq0\\\Tr(\bA)\leq R^2}}\inf_{\bB}F(\bB,\bA) = \cL(\bN).
    \end{split}
\end{equation*}

To see this, it is sufficient to show that for any feasible $\bA,$ $F(\bB_*, \bA) - F(\bB_*, \bA_*) \leq 0.$ Recall that $F(\bB,\bA)=\Tr\!\left[\bQ\bigl((\bI-\bB\bN^{1/2})\bA(\bI-\bN^{1/2}\bB^\top)+\bB\bB^\top\bigr)\right]$. Then, a direct computation gives \begin{equation*}
\begin{split}
     F(\bB_*, \bA) - F(\bB_*, \bA_*) =  \Tr[\bQ(\bI-\bB_*\bN^{1/2})(\bA - \bA_*)(\bI-\bN^{1/2}\bB_*^{\top}) ] .
\end{split}
\end{equation*}

Let us define $f(\bA) = \Tr[\bQ \bP(\bA,\bN)] $. Then, differentiating \eqref{eq:ellipsoid-resolvent-identities} in the direction $\bDelta$ gives
\begin{equation*}
\begin{split}
    D f_{\bA_*}[\bDelta] := \frac{\dd}{\dd t} f(\bA_* + t \Delta) \big |_{t=0} =&\Tr[\bQ(\bI-\bP(\bA_*,\bN)\bN)\bDelta(\bI-\bN\bP(\bA_*,\bN))] \\
     =& \Tr[\bQ(\bI-\bB_*\bN^{1/2})\bDelta(\bI-\bN^{1/2}\bB_*^{\top}) ],
\end{split}
\end{equation*} which implies $D f_{\bA_*}[\bA - \bA_*]\le 0.$

Since by definition $\bA_* =\argmax_{\substack{\bA\succeq0\\\Tr(\bA)\leq R^2}}\Tr[\bQ \bP(\bA, \bN)],$  for any feasible direction $\bA - \bA_*,$ we have $F(\bB_*, \bA) - F(\bB_*, \bA_*) = D f_{\bA_*}[\bA - \bA_*] \leq 0,$ which concludes the proof.

\paragraph{Proof of the lower bound.}

We prove the lower bound via probabilistic methods. In particular, we construct a prior distribution $p(\bu)$ supported on $\{\bu: \|\bu\|_2 \leq R\},$ and use the fact that \begin{equation*}
    \begin{split}
        \rm{R}_{*, \rm{seq}}(\bN) \geq \inf_{\widehat \bu} \E_{\bu \sim p_{\bu}} \E_{\bu}\|\bQ^{1/2}(\widehat\bu(\by_{\bN, \rm{seq}})-\bu)\|_2^2. 
    \end{split}
\end{equation*}

For the construction, fix a feasible finite-rank $\bA=\sum_{\ell=1}^r a_\ell\bv_\ell\bv_\ell^\top \in \R^{d \times d}$, where the $\bv_\ell$ are orthonormal, $a_\ell>0$, and $\sum_\ell a_\ell\leq R^2$. Let $\bT$ have columns $\sqrt{a_\ell}\bv_\ell$, so $\bT\bT^\top=\bA$ and $\bT^\top\bT=\operatorname{Diag}(a_1,\dots,a_r)$. Consider $\boldsymbol\omega$ independent coordinates with density
\begin{equation*}
q(w)=\cos^2(\pi w/2)\mathbf1\{|w|\leq1\},
\end{equation*} which have the following properties
\begin{equation*}
\int_{-1}^1q'(w)\,dw=0,\qquad \int_{-1}^1\frac{q'(w)^2}{q(w)}\,dw=\pi^2\int_{-1}^1\sin^2(\pi w/2)\,dw=\pi^2.
\end{equation*}

Let $\bu=\bT\boldsymbol\omega$. This prior is supported on the parameter ball, since
\begin{equation*}
\|\bu\|_2^2=\sum_{\ell=1}^r a_\ell\omega_\ell^2\leq\sum_{\ell=1}^r a_\ell=\Tr(\bA)\leq R^2.
\end{equation*}

Next, for such prior distribution $p(\bu),$ we provide a  lower bound on the minimax risk. Let $p(\by_{\bN, \rm{seq}} | \bu)$ be the probability distribution of $\by_{\bN, \rm{seq}}$ given $\bu.$ Since the randomness of $\bu$ only comes from $\boldsymbol\omega,$ we rewrite the joint distribution as $p(\by_{\bN, \rm{seq}}, \boldsymbol\omega)$ and define the score function w.r.t.\ $\boldsymbol\omega$ as \begin{equation*}
    \begin{split}
        \zeta = \nabla_{\boldsymbol\omega} \log p(\by_{\bN, \rm{seq}} , \boldsymbol\omega) = \nabla_{\boldsymbol\omega} \log p(\by_{\bN, \rm{seq}} \big| \boldsymbol\omega) + \nabla_{\boldsymbol\omega} \log q(\boldsymbol\omega).
    \end{split}
\end{equation*}

We then compute two quantities: \begin{equation*}
\begin{split}
\E[\be\zeta^\top]
&=\iint \be\bigl(\nabla_{\boldsymbol\omega}p(\by_{\bN,\rm{seq}},\boldsymbol\omega)\bigr)^\top\,d\mu(\by_{\bN,\rm{seq}})\,d\boldsymbol\omega\\
&=-\iint \frac{\partial\be}{\partial\boldsymbol\omega}p(\by_{\bN,\rm{seq}},\boldsymbol\omega)\,d\mu(\by_{\bN,\rm{seq}})\,d\boldsymbol\omega\\
&=\bQ^{1/2}\bT\iint p(\by_{\bN,\rm{seq}},\boldsymbol\omega)\,d\mu(\by_{\bN,\rm{seq}})\,d\boldsymbol\omega
=\bQ^{1/2}\bT,\\[2mm]
\\
\bJ := \E[\zeta\zeta^\top]
&=\E\!\left[(\nabla_{\boldsymbol\omega}\log p(\by_{\bN,\rm{seq}}\mid\boldsymbol\omega))(\nabla_{\boldsymbol\omega}\log p(\by_{\bN,\rm{seq}}\mid\boldsymbol\omega))^\top\right]\\
&\quad+\E\!\left[(\nabla_{\boldsymbol\omega}\log q(\boldsymbol\omega))(\nabla_{\boldsymbol\omega}\log q(\boldsymbol\omega))^\top\right]\\
&\quad+\underbrace{\E\!\left[(\nabla_{\boldsymbol\omega}\log p(\by_{\bN,\rm{seq}}\mid\boldsymbol\omega))(\nabla_{\boldsymbol\omega}\log q(\boldsymbol\omega))^\top\right]}_{=0}\\
&\quad+\underbrace{\E\!\left[(\nabla_{\boldsymbol\omega}\log q(\boldsymbol\omega))(\nabla_{\boldsymbol\omega}\log p(\by_{\bN,\rm{seq}}\mid\boldsymbol\omega))^\top\right]}_{=0}\\
&=\bT^\top\bN\bT+\pi^2\bI_r,
\end{split}
\end{equation*} where we use the identities \begin{equation*}
    \begin{split}  \E\!\left[\nabla_{\boldsymbol\omega}\log p(\by_{\bN,\rm{seq}}\mid\boldsymbol\omega)\,\middle|\,\boldsymbol\omega\right]
&=\int \nabla_{\boldsymbol\omega}p(\by_{\bN,\rm{seq}}\mid\boldsymbol\omega)\,d\mu(\by_{\bN,\rm{seq}})
=\nabla_{\boldsymbol\omega}1=0, \\
\E\!\left[\frac{q'(\omega_\ell)}{q(\omega_\ell)}\right]
&=\int_{-1}^1q'(w)\,dw=q(1)-q(-1)=0,\\
\E\!\left[\frac{q'(\omega_\ell)}{q(\omega_\ell)}\frac{q'(\omega_m)}{q(\omega_m)}\right]
&=\begin{cases}
\displaystyle\int_{-1}^1\frac{q'(w)^2}{q(w)}\,dw
=\pi^2\int_{-1}^1\sin^2(\pi w/2)\,dw=\pi^2,&\ell=m,\\
\displaystyle\left(\int_{-1}^1q'(w)\,dw\right)^2=0,&\ell\neq m.
\end{cases}\\[2mm]
    \end{split}
\end{equation*}

For an arbitrary estimator $\widehat \bu,$ define $\be=\bQ^{1/2}(\widehat\bu-\bT\boldsymbol\omega),$ and it is sufficient to lower bound $\E[\|\be\|_2^2].$ 

Note that $\E[\| \be - \bQ^{1/2} \bT \bJ^{-1} \zeta\|_2^2] \geq 0.$ 
Expanding a nonnegative square yields
\begin{align*}
0&\leq\E\|\be-\bQ^{1/2}\bT\bJ^{-1}\zeta\|_2^2\\
& = \E[\|\be\|_2^2] - 2 \Tr(\bQ^{1/2} \bT \bJ^{-1}\E[ \zeta \be^\top]) + \Tr[\bJ^{-1} \bT^\top \bQ \bT \bJ^{-1}] \E[\zeta \zeta^\top]\\ 
&=\E\|\be\|_2^2-\Tr[\bQ\bT(\bT^\top\bN\bT+\pi^2\bI_r)^{-1}\bT^\top].
\end{align*}
Recall by definition $\bA = \bT \bT^\top$. Then, the resulting lower bound is
\begin{equation*}
\begin{split}
\E\|\be\|_2^2\geq & \Tr[\bQ\bT(\bT^\top\bN\bT+\pi^2\bI_r)^{-1}\bT^\top]\\
\geq &\pi^{-2}\Tr[\bQ\bP(\bA,\bN)].
\end{split}
\end{equation*}
Taking the supremum over feasible $\bA$ concludes the proof.

\subsection{Proof of Theorem~\ref{thm:minimax-ellipsoid}}
\label{apx:pf-thm:minimax-ellipsoid}

Recall the normalized weights $\bu=\bS^{1/2}\btheta$ and $\widehat\bu=\bS^{1/2}\widehat\btheta$. Denote the full regression data by $\mathcal D=\{(\bX_i,\by_i)\}_{i=1}^K$. The parameter constraint and the test error can be written in the same form as for the Gaussian sequence model, which gives
\begin{equation*}
\mathrm R_*(n_1,\dots,n_K)=\inf_{\widehat\bu}\sup_{\|\bu\|_2\leq R}\E_{\bu}\|\bQ^{1/2}(\widehat\bu(\mathcal D)-\bu)\|_2^2.
\end{equation*}
Here the expectation includes both the designs and the observation noise, and the infimum is over all estimators based on $\mathcal D$. We also recall
\begin{equation*}
\bM=\sum_{i=1}^K\frac{n_i}{\sigma_{\eps_i}^2}\bH_i,\qquad E_0=R^2\max_{1\leq i\leq K}\frac{\|\bH_i\|_{\mathrm{op}}}{\sigma_{\eps_i}^2},\qquad v=1+\kappa E_0.
\end{equation*}

\paragraph{Proof of the upper bound.}
Our goal is to construct a regression estimator whose worst-case risk is bounded by $v\mathcal L(\bM)$, where $v=1+\kappa E_0$ and $\kappa$ is a finite constant depending only on the concentration constants in Assumption~\ref{asm:distr}. Since the minimax risk is the infimum over all estimators, it suffices to construct $\widehat\bu_{\rm{reg}}$ such that
\begin{equation*}
\sup_{\|\bu\|_2\leq R}\E_{\bu}\|\bQ^{1/2}(\widehat\bu_{\rm{reg}}-\bu)\|_2^2\leq v\mathcal L(\bM).
\end{equation*}
We do so by comparing its risk with that of an optimal linear estimator in an auxiliary Gaussian sequence model. All regression expectations below include both the random designs and the observation noise.

Define the statistic
\begin{equation*}
\bz=\bS^{-1/2}\sum_{i=1}^K\frac{1}{\sigma_{\eps_i}^2}\bX_i^\top\by_i.
\end{equation*}
We will show that, uniformly over $\|\bu\|_2\leq R$,
\begin{equation}
\E_{\bu}\bz=\bM\bu,\qquad \operatorname{Cov}_{\bu}(\bz)\preceq v\bM.
\label{eq:ellipsoid-regression-moments}
\end{equation}
We first explain how these moment bounds yield the desired estimator, and then verify them.

Set $\bN=\bM/v$. By \eqref{eq:ellipsoid-regression-moments}, the rescaled statistic $\bz/v$ has mean $\bN\bu$ and covariance at most $\bN$. To see its connection with the Gaussian sequence model, recall that
\begin{equation*}
\by_{\bN,\rm{seq}}=\bN^{1/2}\bu+\bg,\qquad \bg\sim\mathcal N(0,\bI).
\end{equation*}
Multiplying the Gaussian observation by $\bN^{1/2}$ gives
\begin{equation*}
\bN^{1/2}\by_{\bN,\rm{seq}}=\bN\bu+\bN^{1/2}\bg.
\end{equation*}
The moment bounds \eqref{eq:ellipsoid-regression-moments} gives
\begin{align*}
\E_{\bu}[\bz/v]&=\bN\bu=\E_{\bu}[\bN^{1/2}\by_{\bN,\rm{seq}}],\\
\operatorname{Cov}_{\bu}(\bz/v)&\preceq\bN=\operatorname{Cov}_{\bu}(\bN^{1/2}\by_{\bN,\rm{seq}}).
\end{align*}
This motivates us to apply the same linear estimators constructed for the sequence model to the regression model here.
To choose this linear transformation, take
\begin{equation*}
\bA_*\in\argmax_{\substack{\bA\succeq0\\\Tr(\bA)\leq R^2}}\Tr[\bQ\bP(\bA,\bN)],\qquad \bB_*=\bP(\bA_*,\bN)\bN^{1/2}.
\end{equation*}
By \eqref{eq:ellipsoid-linear-identity} and the optimizer identified in \eqref{eq:ellipsoid-complete-square}, the Gaussian estimator
\begin{equation*}
\bB_*\by_{\bN,\rm{seq}}=\bP(\bA_*,\bN)(\bN^{1/2}\by_{\bN,\rm{seq}})
\end{equation*}
has worst-case risk $\mathcal L(\bN)$. The matrices defining this estimator depend only on the model quantities and the parameter constraint, not on the unknown $\bu$. Replacing its input $\bN^{1/2}\by_{\bN,\rm{seq}}$ by $\bz/v$, define
\begin{equation*}
\widehat\bu_{\rm{reg}}=\frac{1}{v}\bP(\bA_*,\bN)\bz,\qquad \widehat\btheta_{\rm{reg}}=\bS^{-1/2}\widehat\bu_{\rm{reg}}.
\end{equation*}

Standard bias-variance decomposition for linear estimators gives, for every $\|\bu\|_2\leq R$,
\begin{equation*}
\begin{split}
&\E_{\bu}\|\bQ^{1/2}(\widehat\bu_{\rm{reg}}-\bu)\|_2^2\\
&\quad=\|\bQ^{1/2}(\bP(\bA_*,\bN)\bN-\bI)\bu\|_2^2\\
&\qquad+v^{-2}\Tr[\bQ\bP(\bA_*,\bN)\operatorname{Cov}_{\bu}(\bz)\bP(\bA_*,\bN)]\\
&\quad\leq\|\bQ^{1/2}(\bP(\bA_*,\bN)\bN-\bI)\bu\|_2^2+\Tr[\bQ\bP(\bA_*,\bN)\bN\bP(\bA_*,\bN)]\\
&\quad=\E_{\bu}\|\bQ^{1/2}(\bB_*\by_{\bN,\rm{seq}}-\bu)\|_2^2.
\end{split}
\end{equation*}
 Taking the supremum over the parameter ball gives
\begin{equation*}
\mathrm R_*(n_1,\dots,n_K)\leq\sup_{\|\bu\|_2\leq R}\E_{\bu}\|\bQ^{1/2}(\widehat\bu_{\rm{reg}}-\bu)\|_2^2\leq\mathcal L(\bN)=\mathcal L(\bM/v).
\end{equation*}

It remains to verify \eqref{eq:ellipsoid-regression-moments} and compare $\mathcal L(\bM/v)$ with $\mathcal L(\bM)$. For the mean, using $\by_i=\bX_i\bS^{-1/2}\bu+\mathbf{\eps}_i$ and the independence and zero mean of the observation noise gives
\begin{equation*}
\begin{split}
\E_{\bu}\bz
&=\bS^{-1/2}\sum_{i=1}^K\frac{1}{\sigma_{\eps_i}^2}\E[\bX_i^\top\bX_i]\bS^{-1/2}\bu\\
&=\sum_{i=1}^K\frac{n_i}{\sigma_{\eps_i}^2}\bH_i\bu=\bM\bu.
\end{split}
\end{equation*}

For the covariance, consider one observation $(\bx,y)$ from domain $i$, where $y=\langle\bx,\bS^{-1/2}\bu\rangle+\eps$. Its contribution in a direction $\bw$ is
\begin{equation*}
\langle\bx,\bS^{-1/2}\bw\rangle y=\langle\bx,\bS^{-1/2}\bw\rangle\langle\bx,\bS^{-1/2}\bu\rangle+\langle\bx,\bS^{-1/2}\bw\rangle\eps.
\end{equation*}
The first term fluctuates because the inputs are random, and controlling its variance requires a fourth-moment bound, which is obtained from Assumption~\ref{asm:distr}. More precisely, for $\bx\sim P_i$ and any $\bw$ with $\bw^\top\bH_i\bw>0$, applying that assumption to $\bS^{-1/2}\bw\bw^\top\bS^{-1/2}$ gives
\begin{equation*}
P_i\!\left[\left|\langle\bx,\bS^{-1/2}\bw\rangle^2-\bw^\top\bH_i\bw\right|>t\bw^\top\bH_i\bw\right]\leq C\exp(-ct^{1/\eta}),\qquad t>0,
\end{equation*}
where $C,c,\eta$ are the concentration constants in the assumption. Integrating this tail bound yields
\begin{equation*}
\begin{split}
\E_{P_i}\langle\bx,\bS^{-1/2}\bw\rangle^4
&=(\bw^\top\bH_i\bw)^2+\E_{P_i}\!\left[\bigl(\langle\bx,\bS^{-1/2}\bw\rangle^2-\bw^\top\bH_i\bw\bigr)^2\right]\\
&\leq\left(1+2C\int_0^\infty t\exp(-ct^{1/\eta})\,dt\right)(\bw^\top\bH_i\bw)^2\\
&=\kappa(\bw^\top\bH_i\bw)^2.
\end{split}
\end{equation*}
This defines a finite constant $\kappa$ depending only on the concentration constants. If $\bw^\top\bH_i\bw=0$, then $\langle\bx,\bS^{-1/2}\bw\rangle=0$ almost surely, so the same bound holds. Applying Cauchy--Schwarz now gives
\begin{equation*}
\begin{split}
\E_{P_i}\!\left[\langle\bx,\bS^{-1/2}\bw\rangle^2\langle\bx,\bS^{-1/2}\bu\rangle^2\right]&\leq\left(\E_{P_i}\langle\bx,\bS^{-1/2}\bw\rangle^4\E_{P_i}\langle\bx,\bS^{-1/2}\bu\rangle^4\right)^{1/2}\\
&\leq\kappa(\bw^\top\bH_i\bw)(\bu^\top\bH_i\bu).
\end{split}
\end{equation*}

Since $\eps$ is independent of $\bx$, with mean zero and variance $\sigma_{\eps_i}^2$, we have
\begin{equation*}
\begin{split}
\operatorname{Var}_{\bu}\!\left(\langle\bx,\bS^{-1/2}\bw\rangle y\right)
&=\E_{P_i}\!\left[\langle\bx,\bS^{-1/2}\bw\rangle^2\langle\bx,\bS^{-1/2}\bu\rangle^2\right]-(\bw^\top\bH_i\bu)^2+\sigma_{\eps_i}^2\bw^\top\bH_i\bw\\
&\leq\bigl(\kappa\bu^\top\bH_i\bu+\sigma_{\eps_i}^2\bigr)\bw^\top\bH_i\bw\\
&\leq v\sigma_{\eps_i}^2\bw^\top\bH_i\bw.
\end{split}
\end{equation*}
The last inequality holds uniformly over the parameter ball because
\begin{equation*}
\frac{\bu^\top\bH_i\bu}{\sigma_{\eps_i}^2}\leq\frac{R^2\|\bH_i\|_{\mathrm{op}}}{\sigma_{\eps_i}^2}\leq E_0,\qquad \|\bu\|_2\leq R,
\end{equation*}
and $v=1+\kappa E_0$. Thus, the additional variance caused by the random inputs is controlled uniformly over all admissible parameters.

Let $\bx_{i,j}^\top$ denote the $j$-th row of $\bX_i$. Independence of the observations implies
\begin{equation*}
\begin{split}
\bw^\top\operatorname{Cov}_{\bu}(\bz)\bw
&=\sum_{i=1}^K\sum_{j=1}^{n_i}\frac{1}{\sigma_{\eps_i}^4}\operatorname{Var}_{\bu}\!\left(\langle\bx_{i,j},\bS^{-1/2}\bw\rangle y_{i,j}\right)\\
&\leq v\sum_{i=1}^K\frac{n_i}{\sigma_{\eps_i}^2}\bw^\top\bH_i\bw\\
&=v\bw^\top\bM\bw.
\end{split}
\end{equation*}
Since this holds for every $\bw$, it proves $\operatorname{Cov}_{\bu}(\bz)\preceq v\bM$ and completes the verification of \eqref{eq:ellipsoid-regression-moments}.

Finally, we compare the Gaussian linear minimax risks at $\bM/v$ and $\bM$. Since $v\geq1$, for every feasible $\bA$,
\begin{equation*}
\bI+v^{-1}\bA^{1/2}\bM\bA^{1/2}\succeq v^{-1}(\bI+\bA^{1/2}\bM\bA^{1/2}).
\end{equation*}
Inversion reverses the PSD order, and congruence by $\bA^{1/2}$ therefore gives
\begin{equation*}
\bP(\bA,\bM/v)\preceq v\bP(\bA,\bM).
\end{equation*}
Taking the trace against $\bQ$ and then the supremum over feasible $\bA$ yields $\mathcal L(\bM/v)\leq v\mathcal L(\bM)$. Combining this comparison with the risk bound for the constructed estimator proves
\begin{equation}
\mathrm R_*(n_1,\dots,n_K)\leq\mathcal L(\bM/v)\leq v\mathcal L(\bM).
\label{eq:ellipsoid-regression-upper}
\end{equation}

If the designs are Gaussian, the same construction admits a sharper constant. In this case, the exact fourth-moment identity gives
\begin{equation*}
\begin{split}
\operatorname{Var}_{\bu}\!\left(\langle\bx,\bS^{-1/2}\bw\rangle y\right)
&=(\bu^\top\bH_i\bu+\sigma_{\eps_i}^2)\bw^\top\bH_i\bw+(\bw^\top\bH_i\bu)^2\\
&\leq(2\bu^\top\bH_i\bu+\sigma_{\eps_i}^2)\bw^\top\bH_i\bw\\
&\leq(1+2E_0)\sigma_{\eps_i}^2\bw^\top\bH_i\bw,
\end{split}
\end{equation*}
where we used $(\bw^\top\bH_i\bu)^2\leq(\bw^\top\bH_i\bw)(\bu^\top\bH_i\bu)$. Consequently, \eqref{eq:ellipsoid-regression-moments} holds with $v=1+2E_0$, and repeating the estimator construction and risk comparison with this choice gives \eqref{eq:ellipsoid-regression-upper} with the sharper factor.

\paragraph{Proof of the lower bound.}
We aim to prove $\mathrm R_*(n_1,\dots,n_K)\geq\pi^{-2}\mathcal L(\bM),$ and follow the same construction proof as for the Gaussian sequence model.

We recall the construction below. Fix a nonzero feasible matrix $\bA$ with spectral decomposition
\begin{equation*}
\bA=\sum_{\ell=1}^r a_\ell\bv_\ell\bv_\ell^\top,\qquad a_\ell>0,\qquad \sum_{\ell=1}^r a_\ell\leq R^2,
\end{equation*}
where the $\bv_\ell$ are orthonormal. As in Appendix~\ref{apx:pf-thm:ellipsoid-gaussian-seq}, let $\bT$ have columns $\sqrt{a_\ell}\bv_\ell$, and set $\bu=\bT\boldsymbol\omega$, where the coordinates of $\boldsymbol\omega$ are independent with density
\begin{equation*}
q(w)=\cos^2(\pi w/2)\mathbf1\{|w|\leq1\}.
\end{equation*}
Then $\bT\bT^\top=\bA$, and this prior is supported on the parameter ball because
\begin{equation*}
\|\bu\|_2^2=\sum_{\ell=1}^r a_\ell\omega_\ell^2\leq\sum_{\ell=1}^r a_\ell=\Tr(\bA)\leq R^2.
\end{equation*}
Consequently, for every estimator $\widehat\bu$,
\begin{equation*}
\sup_{\|\bu\|_2\leq R}\E_{\bu}\|\bQ^{1/2}(\widehat\bu(\mathcal D)-\bu)\|_2^2
\geq\E_{\boldsymbol\omega}\E_{\mathcal D\mid\boldsymbol\omega}\|\bQ^{1/2}(\widehat\bu(\mathcal D)-\bT\boldsymbol\omega)\|_2^2,
\end{equation*} where $\mathcal D$ denotes the whole dataset.

Write $q(\boldsymbol\omega)=\prod_{\ell=1}^r q(\omega_\ell)$ and $p(\mathcal D,\boldsymbol\omega)=p(\mathcal D\mid\boldsymbol\omega)q(\boldsymbol\omega)$. Here the density of the data is taken with respect to the product of the design law and the Lebesgue measure for the responses, denoted by $\mu$. This reference measure is independent of $\boldsymbol\omega$. Define the joint score and estimation error by
\begin{equation*}
\zeta=\nabla_{\boldsymbol\omega}\log p(\mathcal D,\boldsymbol\omega),\qquad \be=\bQ^{1/2}(\widehat\bu(\mathcal D)-\bT\boldsymbol\omega).
\end{equation*}

We show that the following identity holds in the regression case as well, and the rest follows exactly as in Appendix \ref{apx:pf-thm:ellipsoid-gaussian-seq}:
\begin{equation*}
\bJ:=\E[\zeta\zeta^\top]=\bT^\top\bM\bT+\pi^2\bI_r,\qquad \E[\be\zeta^\top]=\bQ^{1/2}\bT.
\end{equation*}

First, we compute $\E[\zeta\zeta^\top]$. The joint score decomposes as
\begin{equation*}
\zeta=\nabla_{\boldsymbol\omega}\log p(\mathcal D\mid\boldsymbol\omega)+\nabla_{\boldsymbol\omega}\log q(\boldsymbol\omega).
\end{equation*}
The design distribution does not depend on $\boldsymbol\omega$, and the conditional response distribution is
\begin{equation*}
\by_i\mid\bX_i,\boldsymbol\omega\sim\mathcal N(\bX_i\bS^{-1/2}\bT\boldsymbol\omega,\sigma_{\eps_i}^2\bI_{n_i}),\qquad i=1,\dots,K.
\end{equation*}
Differentiating the conditional Gaussian likelihood therefore gives
\begin{equation*}
\begin{split}
\nabla_{\boldsymbol\omega}\log p(\mathcal D\mid\boldsymbol\omega)
&=\sum_{i=1}^K\frac{1}{\sigma_{\eps_i}^2}\bT^\top\bS^{-1/2}\bX_i^\top(\by_i-\bX_i\bS^{-1/2}\bT\boldsymbol\omega)\\
&=\sum_{i=1}^K\frac{1}{\sigma_{\eps_i}^2}\bT^\top\bS^{-1/2}\bX_i^\top\mathbf{\eps}_i.
\end{split}
\end{equation*}
This expression has mean zero conditionally on the designs and $\boldsymbol\omega$. The cross terms between domains vanish by independence and the zero mean of the observation noises. Using their covariances and then averaging over the designs yields
\begin{equation*}
\begin{split}
&\E\!\left[(\nabla_{\boldsymbol\omega}\log p(\mathcal D\mid\boldsymbol\omega))(\nabla_{\boldsymbol\omega}\log p(\mathcal D\mid\boldsymbol\omega))^\top\,\middle|\,\boldsymbol\omega\right]\\
&\quad=\sum_{i=1}^K\frac{1}{\sigma_{\eps_i}^4}\bT^\top\bS^{-1/2}\E[\bX_i^\top\mathbf{\eps}_i\mathbf{\eps}_i^\top\bX_i]\bS^{-1/2}\bT\\
&\quad=\sum_{i=1}^K\frac{1}{\sigma_{\eps_i}^2}\bT^\top\bS^{-1/2}\E[\bX_i^\top\bX_i]\bS^{-1/2}\bT\\
&\quad=\sum_{i=1}^K\frac{n_i}{\sigma_{\eps_i}^2}\bT^\top\bH_i\bT=\bT^\top\bM\bT.
\end{split}
\end{equation*}
This is precisely the average Fisher information obtained in the Gaussian sequence model with $\bN=\bM$ under the same parameterization $\bu=\bT\boldsymbol\omega$. The calculation uses only the second moments of the designs; the Gaussian assumption is used for the observation noise.

For the prior contribution, the one-dimensional density satisfies
\begin{equation*}
\int_{-1}^1q'(w)\,dw=0,\qquad \int_{-1}^1\frac{q'(w)^2}{q(w)}\,dw=\pi^2\int_{-1}^1\sin^2(\pi w/2)\,dw=\pi^2.
\end{equation*}
Independence of the prior coordinates therefore gives
\begin{equation*}
\E[\nabla_{\boldsymbol\omega}\log q(\boldsymbol\omega)]=0,\qquad
\E\!\left[(\nabla_{\boldsymbol\omega}\log q(\boldsymbol\omega))(\nabla_{\boldsymbol\omega}\log q(\boldsymbol\omega))^\top\right]=\pi^2\bI_r.
\end{equation*}
Moreover, the likelihood score has mean zero conditionally on $\boldsymbol\omega$, so its cross terms with the prior score vanish. Combining the two contributions proves
\begin{equation*}
\bJ=\E[\zeta\zeta^\top]=\bT^\top\bM\bT+\pi^2\bI_r.
\end{equation*}

Next, we verify $\E[\be\zeta^\top]=\bQ^{1/2}\bT$. With the data held fixed, $\widehat\bu(\mathcal D)$ does not depend on $\boldsymbol\omega$, and hence
\begin{equation*}
\frac{\partial\be}{\partial\boldsymbol\omega}=-\bQ^{1/2}\bT.
\end{equation*}
Since $q(-1)=q(1)=0$, the boundary terms vanish when integrating by parts in each prior coordinate. Using $\zeta\,p(\mathcal D,\boldsymbol\omega)=\nabla_{\boldsymbol\omega}p(\mathcal D,\boldsymbol\omega)$, we obtain
\begin{equation*}
\begin{split}
\E[\be\zeta^\top]
&=\iint\be(\nabla_{\boldsymbol\omega}p(\mathcal D,\boldsymbol\omega))^\top\,d\mu(\mathcal D)\,d\boldsymbol\omega\\
&=-\iint\frac{\partial\be}{\partial\boldsymbol\omega}p(\mathcal D,\boldsymbol\omega)\,d\mu(\mathcal D)\,d\boldsymbol\omega\\
&=\bQ^{1/2}\bT.
\end{split}
\end{equation*}
The calculation first applies when $\bQ^{1/2}\widehat\bu$ is bounded. For an estimator with finite average risk under the prior, truncation and the finite second moment of $\zeta$ justify passage to the $L_2$ limit. An estimator with infinite average risk already satisfies the desired lower bound. Thus the identity, and consequently the preceding risk bound, applies to every estimator with finite average risk.

It remains to express the lower bound in terms of the variational objective. Since $\bT\bT^\top=\bA$, the resolvent identities give
\begin{equation*}
\begin{split}
\bT(\bT^\top\bM\bT+\pi^2\bI_r)^{-1}\bT^\top
&=\pi^{-2}\bT(\bI_r+\pi^{-2}\bT^\top\bM\bT)^{-1}\bT^\top\\
&=\pi^{-2}\bP(\bA,\bM/\pi^2).
\end{split}
\end{equation*}
Furthermore, $\bM/\pi^2\preceq\bM$, so inversion reverses the PSD order in the definition of $\bP$ and yields
\begin{equation*}
\bP(\bA,\bM/\pi^2)\succeq\bP(\bA,\bM).
\end{equation*}
Combining these relations with the comparison between worst-case and average risk, we obtain, for every estimator and every nonzero feasible $\bA$,
\begin{equation*}
\begin{split}
\sup_{\|\bu\|_2\leq R}\E_{\bu}\|\bQ^{1/2}(\widehat\bu(\mathcal D)-\bu)\|_2^2
&\geq\E\|\be\|_2^2\\
&\geq\pi^{-2}\Tr[\bQ\bP(\bA,\bM/\pi^2)]\\
&\geq\pi^{-2}\Tr[\bQ\bP(\bA,\bM)].
\end{split}
\end{equation*}
Taking the infimum over estimators and then the supremum over feasible $\bA$ proves
\begin{equation*}
\mathrm R_*(n_1,\dots,n_K)\geq\pi^{-2}\mathcal L(\bM/\pi^2)\geq\pi^{-2}\mathcal L(\bM).
\end{equation*}
The matrix $\bA=0$ contributes zero and does not change the supremum.

Combining this lower bound with \eqref{eq:ellipsoid-regression-upper} proves Theorem~\ref{thm:minimax-ellipsoid}. Finally, Theorem~\ref{thm:ellipsoid-gaussian-seq} gives
\begin{equation*}
\pi^{-2}\mathcal L(\bM)\leq\mathrm R_{*,\rm{seq}}(\bM)\leq\mathcal L(\bM),
\end{equation*}
and therefore
\begin{equation*}
\pi^{-2}\mathrm R_{*,\rm{seq}}(\bM)\leq\mathrm R_*(n_1,\dots,n_K)\leq\pi^2v\,\mathrm R_{*,\rm{seq}}(\bM).
\end{equation*}
Thus, when the concentration constants and $E_0$ are uniformly bounded, the unrestricted minimax risks of the regression and Gaussian sequence models have the same order.

\subsection{Extension to infinite dimension}
\label{apx:ellipsoid-infinite}

We work on $\mathcal H=(\ker\bS)^\perp$ and set
$D=\operatorname{Dom}(\bS^{-1/2})=\operatorname{Ran}(\bS^{1/2})$.
This is dense in $\mathcal H$, and normalized parameters lie in
$D\cap\{\bu:\|\bu\|_2\le R\}$. All normalized covariance forms are
understood as bounded operators. Let $\mathcal A$ be the positive
trace-class operators of trace at most $R^2$.

For bounded $\bN\succeq0$ and a Hilbert--Schmidt operator $W$, define
\begin{align*}
F(W,\bA)&=\|(W\bN^{1/2}-\bQ^{1/2})\bA^{1/2}\|_{\mathrm{HS}}^2
          +\|W\|_{\mathrm{HS}}^2,\\
J_{\bN}(W)&=\sup_{\bA\in\mathcal A}F(W,\bA)
=R^2\|W\bN^{1/2}-\bQ^{1/2}\|_{\op}^2+\|W\|_{\mathrm{HS}}^2.
\end{align*}
The same completion of squares as in finite dimension gives
\begin{align*}
\min_W F(W,\bA)&=\Tr[\bQ\bP(\bA,\bN)],\\
W_{\bA}&=\bQ^{1/2}\bP(\bA,\bN)\bN^{1/2},\qquad
\|W_{\bA}\|_{\mathrm{HS}}^2\le R^2\|\bQ\|_{\op}.
\end{align*}
Thus all these minimizers lie in the weakly compact Hilbert--Schmidt
ball $\mathcal B$ of radius $R\sqrt{\|\bQ\|_{\op}}$.
The function $F$ is convex and weakly lower semicontinuous in $W$,
and affine and trace-norm continuous in $\bA$.
Sion's minimax theorem~\citep{sion1958minimax} therefore yields
\[
\min_W J_{\bN}(W)
=\min_{W\in\mathcal B}\sup_{\bA\in\mathcal A}F(W,\bA)
=\sup_{\bA\in\mathcal A}\min_{W\in\mathcal B}F(W,\bA)
=\mathcal L(\bN).
\]
The first equality follows from $J_{\bN}(W)\ge\|W\|_{\mathrm{HS}}^2$
and $J_{\bN}(0)=R^2\|\bQ\|_{\op}$. Let $W_*$ be a minimizer.
Projecting its output onto $(\ker\bQ)^\perp$ and its input onto
$(\ker\bN)^\perp$ cannot increase $J_{\bN}$.
Since $\bQ^{1/2}D$ and $\bN^{1/2}D$ are dense in these respective
subspaces, there are finite-rank operators
$L_m=\sum_j a_{mj}\otimes b_{mj}$, with $a_{mj},b_{mj}\in D$,
such that $W_m=\bQ^{1/2}L_m\bN^{1/2}\to W_*$ in Hilbert--Schmidt
norm. Here $(a\otimes b)h=a\langle b,h\rangle$.
Boundedness of $\bN$ implies $J_{\bN}(W_m)\to\mathcal L(\bN)$.

For the regression upper bound, set $v=1+\kappa E_0$ and $\bN=\bM/v$.
For $b\in D$, the scalar statistic
\[
z(b)=\sum_{i=1}^K\sum_{r=1}^{n_i}\sigma_{\eps_i}^{-2}
\langle\bx_i^{(r)},\bS^{-1/2}b\rangle y_i^{(r)}
\]
satisfies $\E_{\bu}z(b)=\langle b,\bM\bu\rangle$ and
$\operatorname{Var}_{\bu}z(b)\le v\langle b,\bM b\rangle$,
by the scalar calculation proving \eqref{eq:ellipsoid-regression-moments}.
Consequently the well-defined finite-rank estimator
$\widehat\btheta_m=\bS^{-1/2}\sum_j a_{mj}z(b_{mj})/v$ satisfies,
with $\bC_{\rm test}=\sum_k\pi_k^*\bC_k$,
\begin{align*}
\sup_{\btheta\in\bTheta}\E_{\btheta}
\|\bC_{\rm test}^{1/2}(\widehat\btheta_m-\btheta)\|_2^2
&\le R^2\|\bQ^{1/2}(L_m\bN-\bI)\|_{\op}^2
       +\|\bQ^{1/2}L_m\bN^{1/2}\|_{\mathrm{HS}}^2\\
&=J_{\bN}(W_m)\longrightarrow\mathcal L(\bM/v).
\end{align*}
Hence $\mathrm R_*\le\mathcal L(\bM/v)\le v\mathcal L(\bM)$.

For the lower bound, choose finite-rank orthogonal projections $\Pi_m$
onto the first $m$ elements of an orthonormal basis contained in $D$.
For any $\bA\in\mathcal A$, the feasible operators
$\bA_m=\Pi_m\bA\Pi_m=T_mT_m^*$ converge to $\bA$ in trace norm;
choose $T_m$ with orthogonal columns in $D$.
The preceding cosine-prior argument applies to
$\btheta=\bS^{-1/2}T_m\boldsymbol\omega$.
For arbitrary estimators, apply its score calculation directly to
$e=\bC_{\rm test}^{1/2}(\widehat\btheta-\btheta)$:
the target derivative is $G_m=\bC_{\rm test}^{1/2}\bS^{-1/2}T_m$,
with $G_m^*G_m=T_m^*\bQ T_m$.
It therefore gives
\[
\mathrm R_*\ge\pi^{-2}\Tr[\bQ\bP(\bA_m,\bM/\pi^2)].
\]
The resolvent identity makes $\bP(\bA,\bN)$ trace-norm continuous
in $\bA$ for bounded $\bN$. Letting $m\to\infty$ and taking the
supremum over $\bA$ yields
$\mathrm R_*\ge\pi^{-2}\mathcal L(\bM/\pi^2)
\ge\pi^{-2}\mathcal L(\bM)$, completing the proof.

\section{Proof of Theorem \ref{thm:deteq-risk}}
\label{apx:pf-thm:deteq-risk}

\subsection{Invertibility of \texorpdfstring{$\bL$}{L}}

Recall that $\bL=\bD-\bK$, where $\bD=\operatorname{diag}(n_1/\mu_1^2,\dots,n_K/\mu_K^2)$ and $\bK_{ij}=\Tr(\bC_i\overline\bG\bC_j\overline\bG)$. Since the covariance matrices are positive semidefinite, $K_{ij} \geq 0$
and $\bK_{ij}=\bK_{ji}$. Using $\overline\bG^{-1}=\sum_{j=1}^K\mu_j\bC_j+\lambda\bI$ and the fixed-point equation $n_i/\mu_i=1+\Tr(\bC_i\overline\bG)$, we obtain
\begin{align*}
\sum_{j=1}^K\mu_j\bK_{ij}
&=\Tr\left(\bC_i\overline\bG\left(\sum_{j=1}^K\mu_j\bC_j\right)\overline\bG\right)\\
&=\Tr(\bC_i\overline\bG)-\lambda\Tr(\bC_i\overline\bG^2)\\
&=\frac{n_i}{\mu_i}-1-\lambda\Tr(\bC_i\overline\bG^2).
\end{align*}
For any $\bv\in\R^K$, the Cauchy-Schwartz inequality $2v_iv_j\leq(\mu_j/\mu_i)v_i^2+(\mu_i/\mu_j)v_j^2$, together with the nonnegativity and symmetry of $\bK$, yields
\begin{align*}
\bv^\top\bK\bv
&\leq\frac12\sum_{i,j=1}^K\bK_{ij}\left(\frac{\mu_j}{\mu_i}v_i^2+\frac{\mu_i}{\mu_j}v_j^2\right)\\
&=\sum_{i=1}^K\frac{v_i^2}{\mu_i}\sum_{j=1}^K\mu_j\bK_{ij}\\
&=\sum_{i=1}^K\left(\frac{n_i}{\mu_i^2}-\frac{1+\lambda\Tr(\bC_i\overline\bG^2)}{\mu_i}\right)v_i^2.
\end{align*}
Consequently, since $\lambda>0$, $\Tr(\bC_i\overline\bG^2)\geq0$, and $\mu_i>0$, every nonzero $\bv\in\R^K$ satisfies
\begin{align*}
\bv^\top\bL\bv = \bv^\top (\bD - \bK) \bv &=\sum_{i=1}^K\frac{n_i}{\mu_i^2}v_i^2-\bv^\top\bK\bv\\
&\geq\sum_{i=1}^K\frac{n_i}{\mu_i^2}v_i^2-\sum_{i=1}^K\left(\frac{n_i}{\mu_i^2}-\frac{1+\lambda\Tr(\bC_i\overline\bG^2)}{\mu_i}\right)v_i^2\\
&=\sum_{i=1}^K\frac{1+\lambda\Tr(\bC_i\overline\bG^2)}{\mu_i}v_i^2 > 0
\end{align*}
Thus $\bL$ is positive definite and hence invertible.

\subsection{Proof of Theorem \ref{thm:deteq-risk}}\label{sec:pfthm2}
\begin{proof}
    Conditionally on the designs,
    $$
        \widehat\btheta - \btheta_* = -\lambda\bG\btheta_* + \bG\sum_{i=1}^K \bX_i^\top\eps_i.
    $$
    Averaging over the independent and centered training noises, we get
    $$
        \sR(\widehat\btheta) = \sum_{k=1}^K \pi_k^*B_k + \sum_{i=1}^K \sigma_{\eps_i}^2\Tr(\bC_{\pi^*}\bG\bX_i^\top\bX_i\bG), \quad \text{where} \quad B_k = \lambda^2\Tr(\bA_*\bG\bC_k\bG).
    $$
    Using the second bound from Theorem \ref{thm:deteq-func-main} with $\bA := \bA_*$, we get for each $k$,
    $$
        |B_k-\lambda^2\be_k^\top\bD\bL^{-1}\tau_{\bA_*}| \leq C_{D,c_0}\lambda^2\eps_2V_{\bA_*} \quad \text{with probability at least } 1-\sum_{i=1}^K n_i^{-D}.
    $$
    Now we compute the deterministic equivalent $\be_k^\top\bD\bL^{-1}\tau_{\bA_*}$. Recall that $\bK := \bD-\bL$. We have
    $$
        \bD\bL^{-1} = \bI + \bK\bL^{-1}, \qquad \bK\be_k = \tau_{\bC_k}.
    $$
    Therefore, 
    \begin{align*}
        \be_k^\top\bD\bL^{-1}\tau_{\bA_*}
        &= \tau_{\bA_*}[k] + \tau_{\bC_k}^\top\bL^{-1}\tau_{\bA_*} \\
        &= \btheta_*^\top\overline\bG\bC_k\overline\bG\btheta_* + \tau_{\bA_*}^\top\bL^{-1}\tau_{\bC_k} \\
        &= \frac{1}{\lambda^2}\overline B_k.
    \end{align*}
    Summing over $k\in[K]$, we have
    $$
        \left|\sum_{k=1}^K \pi_k^*(B_k-\overline B_k)\right| \leq C_{D, c_0}\lambda^2\eps_2V_{\bA_*}.
    $$
    For the variance, using the third bound from Theorem \ref{thm:deteq-func-main} with $\bA = \bC_{\pi^*}$ and target dataset $i$,
    $$
        \left|\frac{1}{n_i}\Tr(\bC_{\pi^*}\bG\bX_i^\top\bX_i\bG)-\frac{1}{n_i}\be_i^\top\bL^{-1}\tau_{\bC_{\pi^*}}\right| \leq C_{D, c_0}\eps_{3,i}n_i^{-1}V_{\bC_{\pi^*}}.
    $$
    Multiplying by $n_i$ and summing over $i\in [K]$ gives
    $$
        \left|\sum_{i=1}^K\sigma_{\eps_i}^2(\Tr(\bC_{\pi^*}\bG\bX_i^\top\bX_i\bG)-\be_i^\top\bL^{-1}\tau_{\bC_{\pi^*}})\right| \leq C_{D, c_0}V_{\bC_{\pi^*}}\sum_{i=1}^K\sigma_{\eps_i}^2\eps_{3,i}.
    $$
    Direct calculation then gives
    \begin{align*}
        \sum_{i=1}^K \sigma_{\eps_i}^2\be_i^\top\bL^{-1}\tau_{\bC_*} 
        &= \sum_{k=1}^K\pi_k^*\sum_{i=1}^K\sigma_{\eps_i}^2\tau_{\bC_k}^\top\bL^{-1}\be_i \\
        &= \sum_{k=1}^K \pi_k^* \overline V_k.
    \end{align*}
    Taking union bounds then implies with probability at least $1-\sum_{i=1}^K n_i^{-D}$,
    $$
        |\sR(\widehat\btheta)-\overline\sR| \leq \eps_n, \qquad \text{where} \qquad \eps_n = C_{D, c_0}\left(\lambda^2\eps_2V_{\btheta_*\btheta_*^\top}+V_{\bC_{\pi^*}}\sum_{i=1}^K \sigma_{\eps_i}^2\eps_{3,i}\right).
    $$
\end{proof}

\section{Proof of Theorem~\ref{thm:minimax}}
\label{apx:pf-thm:minimax}

The source constraints satisfy
\begin{equation*}
\|\bC_i^{1/2-r_i}\btheta\|_2^2\asymp\sum_{j\geq1}j^{\alpha_i(2r_i-1)}\btheta[j]^2,\qquad i\in\{1,2\}.
\end{equation*}
By the definition of $s$,
\begin{equation*}
\max\{\alpha_1(2r_1-1),\alpha_2(2r_2-1)\}=2s-\alpha_1.
\end{equation*}
The larger exponent dominates the smaller one for every $j\geq1$. Consequently, if
\begin{equation*}
\bTheta_0(\rho)=\left\{\btheta:\sum_{j\geq1}j^{2s-\alpha_1}\btheta[j]^2\leq\rho^2\right\},
\end{equation*}
then there exist constants $c,C>0$ such that $\bTheta_0(cR)\subseteq\bTheta\subseteq\bTheta_0(CR)$. We may therefore work with $\bTheta_0(R)$, since fixed changes in the radius only change the constants in the bounds below.

For this ellipsoid, choose the diagonal operator $\bS$ with
$[\bS]_{j,j}=j^{2s-\alpha_1}$, so that the constraint is
$\|\bS^{1/2}\btheta\|_2\leq R$. The normalized covariances satisfy
\begin{equation*}
[\bH_1]_{j,j}
=[\bS^{-1/2}\bC_1\bS^{-1/2}]_{j,j}=j^{-2s},
\qquad
[\bH_2]_{j,j}
=[\bS^{-1/2}\bC_2\bS^{-1/2}]_{j,j}=j^{-(2s-\delta)}.
\end{equation*}
Since $s>0$ and $2s-\delta>0$, both normalized covariances are bounded. The fixed positive noise variances imply $E_0=\max_{i=1,2} \frac{R^2 \|\bH_i\|_{op}}{\sigma^2_{\eps_i}} = O(1)$. Thus, Theorem~\ref{thm:minimax-ellipsoid} applies with constants independent of the number of coordinates.

Define $\bQ=\bH_1$ and $\bM=n_1\sigma_{\eps_1}^{-2}\bH_1+n_2\sigma_{\eps_2}^{-2}\bH_2$. Recall that
\begin{equation*}
\mathcal L(\bM)=\sup_{\substack{\bA\succeq0\\\operatorname{Tr}(\bA)\leq R^2}}F(\bA),\qquad
F(\bA):=\operatorname{Tr}\!\left[\bQ\bA^{1/2}(\bI+\bA^{1/2}\bM\bA^{1/2})^{-1}\bA^{1/2}\right].
\end{equation*}
Since $\bQ, \bM$ are both diagonal, we first show that the optimal $\bA$ is diagonal. In particular, we show that for any feasible $\bA,$ $F(\operatorname{Diag}(\bA)) \geq F(\bA),$ where  $\operatorname{Diag}(\bA)$ denote the diagonal operator whose $j$-th diagonal entry is $[\bA]_{j,j}.$ 

Fix a coordinate $j$, and let $\mathbf e_j$ denote the corresponding standard basis vector. Because $\bM$ is diagonal and positive semidefinite,
\begin{equation*}
\bM\succeq[\bM]_{j,j}\mathbf e_j\mathbf e_j^\top.
\end{equation*}
Congruence by $\bA^{1/2}$ preserves this order, whereas inversion reverses the order of positive definite operators. Consequently,
\begin{align*}
(\bI+\bA^{1/2}\bM\bA^{1/2})^{-1}
&\preceq(\bI+[\bM]_{j,j}\bA^{1/2}\mathbf e_j\mathbf e_j^\top\bA^{1/2})^{-1}\\
&=\bI-\frac{[\bM]_{j,j}\bA^{1/2}\mathbf e_j\mathbf e_j^\top\bA^{1/2}}{1+[\bM]_{j,j}[\bA]_{j,j}},
\end{align*}
where the equality follows from the Sherman--Morrison identity and $\mathbf e_j^\top\bA\mathbf e_j=[\bA]_{j,j}$. Multiplying on both sides by $\bA^{1/2}$ and taking the $j$-th diagonal entry gives
\begin{align*}
\left[\bA^{1/2}(\bI+\bA^{1/2}\bM\bA^{1/2})^{-1}\bA^{1/2}\right]_{j,j}
&\leq[\bA]_{j,j}-\frac{[\bM]_{j,j}[\bA]_{j,j}^2}{1+[\bM]_{j,j}[\bA]_{j,j}}\\
&=\frac{[\bA]_{j,j}}{1+[\bM]_{j,j}[\bA]_{j,j}}.
\end{align*}

Since $\bQ$ is diagonal with nonnegative entries, summing the preceding inequality with weights $[\bQ]_{j,j}$ yields
\begin{align*}
F(\bA)
&=\sum_{j\geq1}[\bQ]_{j,j}\left[\bA^{1/2}(\bI+\bA^{1/2}\bM\bA^{1/2})^{-1}\bA^{1/2}\right]_{j,j}\\
&\leq\sum_{j\geq1}\frac{[\bQ]_{j,j}[\bA]_{j,j}}{1+[\bM]_{j,j}[\bA]_{j,j}}
=F(\operatorname{diag}(\bA)).
\end{align*}
The last equality follows by evaluating $F$ on a diagonal operator. In infinitely many coordinates, the inequality follows by first summing over $j\leq d$ and then letting $d\to\infty$; all summands are nonnegative, so this passage to the limit is justified by monotone convergence.

Moreover, $\bA\succeq0$ implies $[\bA]_{j,j}\geq0$, and taking the diagonal part preserves the trace:
\begin{equation*}
\operatorname{Diag}(\bA)\succeq0,\qquad
\operatorname{Tr}(\operatorname{Diag}(\bA))=\operatorname{Tr}(\bA)\leq R^2.
\end{equation*}
Thus, every feasible $\bA$ can be replaced by a feasible diagonal operator with an objective value at least as large. Since diagonal operators are also included in the original feasible set, restricting the supremum to diagonal $\bA$ leaves its value unchanged. Writing their diagonal entries as $a_j$, we obtain
\begin{equation*}
\mathcal L(\bM)=\sup_{\substack{a_j\geq0\\\sum_{j\geq1}a_j\leq R^2}}\sum_{j\geq1}\frac{[\bQ]_{j,j}a_j}{1+[\bM]_{j,j}a_j}.
\end{equation*}

Set $b_j=[\bQ]_{j,j}a_j$ and define
\begin{equation*}
I_j=\frac{[\bM]_{j,j}}{[\bQ]_{j,j}}=\frac{n_1}{\sigma_{\eps_1}^2}+\frac{n_2[\bC_2]_{j,j}}{\sigma_{\eps_2}^2[\bC_1]_{j,j}}\asymp n_1+n_2j^\delta.
\end{equation*}
Since $[\bQ]_{j,j}\asymp j^{-2s}$, Theorem~\ref{thm:minimax-ellipsoid} and the preceding ellipsoid comparison give
\begin{equation}
\sR^*(n_1,n_2)\asymp\sup_{\substack{b_j\geq0\\\sum_{j\geq1}j^{2s}b_j\leq R^2}}\sum_{j\geq1}\frac{b_j}{1+I_jb_j}.
\label{eq:minimax-power-variational}
\end{equation}

Next, we prove a upper and lower bound on $\sR^*(n_1,n_2).$
For $n_1+n_2>0$, define
\begin{equation*}
T_m=R^2(m+1)^{-2s},\qquad V_m=\sum_{j=1}^m I_j^{-1},\qquad m\in\mathbb N_0,
\end{equation*}
where $V_0=0$. We show that for any $m \in \mathbb N_0,$ \begin{equation*}
    \begin{split}
   \sR^*(n_1,n_2) \gtrsim \min\{V_m,T_m\}, \qquad\sR^*(n_1,n_2)\lesssim (V_m+T_m),
    \end{split}
\end{equation*} and picking the optimal $m$ on both side gives: \begin{equation*}
    \begin{split}
        \sum_{j\geq1}j^{2s}b_j\leq m^{2s}\min\{V_m,T_m\}\leq R^2.
    \end{split}
\end{equation*}

For the upper bound, pick any feasible sequence in \eqref{eq:minimax-power-variational} and any $m\geq0,$ use $\frac{b_j}{1+I_jb_j} \leq I_j^{-1}$ for $j \leq m,$ and $\frac{b_j}{1+I_jb_j} \leq b_j$ for $j \geq m$. Then, this  gives
\begin{equation*}
\sum_{j\geq1}\frac{b_j}{1+I_jb_j}\leq\sum_{j=1}^m I_j^{-1}+\sum_{j>m}b_j\leq V_m+T_m.
\end{equation*}
For the lower bound, we simply need to find a specific feasible $b_j$ that satisfies the lower bound. Fix $m\geq1$ and choose
\begin{equation*}
b_j=\begin{cases}
I_j^{-1}\min\{1,T_m/V_m\},&1\leq j\leq m,\\
0,&j>m.
\end{cases}
\end{equation*}
This choice is feasible because
\begin{equation*}
\sum_{j\geq1}j^{2s}b_j\leq m^{2s}\min\{V_m,T_m\}\leq R^2.
\end{equation*}
Moreover, $I_jb_j\leq1$, so
\begin{equation*}
\sum_{j\geq1}\frac{b_j}{1+I_jb_j}\geq\frac12\sum_{j=1}^m b_j=\frac12\min\{V_m,T_m\}.
\end{equation*}
We have therefore proved
\begin{equation}
\sup_{m\geq1}\min\{V_m,T_m\}\lesssim\sR^*(n_1,n_2)\lesssim\inf_{m\geq0}(V_m+T_m).
\label{eq:minimax-power-two-sided}
\end{equation}
Next, we show that the upper and lower bounds are actually matching. Let $m_*$ be the first positive integer for which $V_{m_*}\geq T_{m_*}$. Such an integer exists since $T_m$ is decreasing to $0$ in $m,$ $V_m$ is increasing in $m,$ and $T_1 = \Theta_{n_1,n_2}(1) \geq V_1 = o_{n_1,n_2}(1)$. The lower bound in \eqref{eq:minimax-power-two-sided} is at least $\min\{V_{m_*} , T_{m_*}\} = T_{m_*},$ as by definition $V_{m_*} \geq T_{m_*}.$ Now we show that in fact $\inf_{m\geq0}(V_m+T_m) \lesssim T_{m_*}.$ We have
\begin{equation*}
\inf_{m\geq0}(V_m+T_m)\leq V_{m_*-1}+T_{m_*-1}\leq2R^2m_*^{-2s}\leq2^{2s+1}T_{m_*},
\end{equation*} where we use the definition of $m_*$ and the fact that $T_{m_*-1}/T_{m_*} \leq 2^{2s} = \Theta(1).$
Consequently,
\begin{equation}
\sR^*(n_1,n_2)\asymp\inf_{m\in\mathbb N_0}\left\{R^2(m+1)^{-2s}+\sum_{j=1}^m\frac1{n_1+n_2j^\delta}\right\}.
\label{eq:minimax-power-cutoff}
\end{equation}

We now evaluate \eqref{eq:minimax-power-cutoff}. All cutoffs below are rounded to integers, which does not affect the rates. We use $n_2\asymp n^{\gamma_2}$ throughout.

\paragraph{Target-only and auxiliary-only rates.}
For target data alone, $V_m\asymp m/n$. Taking $m\asymp n^{1/(1+2s)}$ gives
\begin{equation*}
V_m\asymp T_m\asymp n^{-2s/(1+2s)}=n^{-\Gamma_{\rm tar}}.
\end{equation*}
For auxiliary data alone,
\begin{equation*}
V_m\asymp n_2^{-1}\sum_{j=1}^m j^{-\delta}\asymp\begin{cases}
n_2^{-1}m^{1-\delta},&\delta<1,\\
n_2^{-1}\log(m+1),&\delta=1,\\
n_2^{-1},&\delta>1.
\end{cases}
\end{equation*}
The corresponding balancing cutoffs are
\begin{equation*}
m\asymp\begin{cases}
n_2^{1/(1+2s-\delta)},&\delta<1,\\
(n_2/\log n_2)^{1/(2s)},&\delta=1,\\
n_2^{1/(2s)},&\delta>1.
\end{cases}
\end{equation*}
Substituting these choices gives
\begin{equation*}
\sR^*(0,n_2)\asymp\begin{cases}
n^{-2s\gamma_2/(1+2s-\delta)}=n^{-\Gamma_{\rm aux}},&\delta<1,\\
n^{-\gamma_2}\log n,&\delta=1,\\
n^{-\gamma_2},&\delta>1.
\end{cases}
\end{equation*}

\paragraph{Mixed data with $\gamma_2\leq\gamma_{\rm c}$.}
Take $m\asymp n^{1/(1+2s)}$. Since $\gamma_{\rm c}=1-\delta/(1+2s)$,
\begin{equation*}
n_2m^\delta\asymp n^{\gamma_2+\delta/(1+2s)}\lesssim n.
\end{equation*}
Every summand of $V_m$ is at most a constant times $1/n$. On the block $m/2\leq j\leq m$, we also have $j^\delta\asymp m^\delta$ and hence $I_j\asymp n$. Therefore
\begin{equation*}
V_m\asymp\frac mn\asymp T_m\asymp n^{-\Gamma_{\rm tar}}.
\end{equation*}
This argument holds for every fixed real $\delta$, including $\delta\leq0$, and includes the boundary $\gamma_2=\gamma_{\rm c}$.

\paragraph{Mixed data with $\delta<1$ and $\gamma_2>\gamma_{\rm c}$.}
Take $m\asymp n_2^{1/(1+2s-\delta)}$. Then
\begin{equation*}
\frac{n_2m^\delta}{n}\asymp n^{\gamma_2(1+2s)/(1+2s-\delta)-1}\longrightarrow\infty.
\end{equation*}
Since $I_j\gtrsim n_2j^\delta$, summing over all $j\leq m$ gives $V_m\lesssim n_2^{-1}m^{1-\delta}$. On $m/2\leq j\leq m$, we have $I_j\asymp n_2m^\delta$, which gives the reverse bound. Consequently,
\begin{equation*}
V_m\asymp n_2^{-1}m^{1-\delta}\asymp T_m\asymp n^{-2s\gamma_2/(1+2s-\delta)}=n^{-\Gamma_{\rm aux}}.
\end{equation*}

\paragraph{Mixed data with $\delta=1$ and $\gamma_2>\gamma_{\rm c}$.}
Take $m\asymp(n_2/\log n)^{1/(2s)}$, so $T_m\asymp n_2^{-1}\log n$. If $\gamma_2\geq1$, then $I_j\asymp n_2j$ for every $j\geq1$, and
\begin{equation*}
V_m\asymp n_2^{-1}\log(m+1)\asymp n_2^{-1}\log n\asymp T_m.
\end{equation*}
If $\gamma_{\rm c}<\gamma_2<1$, define $j_{\rm c}=n/n_2$. Then
\begin{equation*}
\frac{m}{j_{\rm c}}\asymp\frac{n^{\gamma_2/(2s)-(1-\gamma_2)}}{(\log n)^{1/(2s)}},\qquad \frac{\gamma_2}{2s}-(1-\gamma_2)=\frac{1+2s}{2s}(\gamma_2-\gamma_{\rm c})>0.
\end{equation*}
Thus $m/j_{\rm c}\to\infty$ and $\log(m/j_{\rm c})\asymp\log n$. Splitting the sum at $j_{\rm c}$ gives
\begin{equation*}
V_m\asymp\frac{j_{\rm c}}n+\frac1{n_2}\sum_{j_{\rm c}<j\leq m}\frac1j\asymp\frac{1+\log(m/j_{\rm c})}{n_2}\asymp n_2^{-1}\log n\asymp T_m.
\end{equation*}
Hence the risk is $\asymp n^{-\gamma_2}\log n$. At $\gamma_2=\gamma_{\rm c}$, the preceding target-data cutoff applies and gives $n^{-\Gamma_{\rm tar}}$ without a logarithmic factor.

\paragraph{Mixed data with $\delta>1$ and $\gamma_{\rm c}<\gamma_2<1$.}
Define $j_{\rm c}=(n/n_2)^{1/\delta}$ and take $m\asymp n^{(\delta-1+\gamma_2)/(2s\delta)}$. The exponent of $m/j_{\rm c}$ is
\begin{equation*}
\frac{\delta-1+\gamma_2}{2s\delta}-\frac{1-\gamma_2}{\delta}=\frac{(1+2s)(\gamma_2-\gamma_{\rm c})}{2s\delta}>0.
\end{equation*}
Thus $m>2j_{\rm c}$ for sufficiently large $n$. The terms with $j\leq j_{\rm c}$ contribute order $j_{\rm c}/n$. Since $\delta>1$, the remaining terms satisfy
\begin{equation*}
\sum_{j_{\rm c}<j\leq m}I_j^{-1}\lesssim n_2^{-1}\sum_{j>j_{\rm c}}j^{-\delta}\lesssim n_2^{-1}j_{\rm c}^{1-\delta}=\frac{j_{\rm c}}n.
\end{equation*}
It follows that
\begin{equation*}
V_m\asymp\frac{j_{\rm c}}n\asymp n^{-(\delta-1+\gamma_2)/\delta}\asymp T_m,
\end{equation*}
which proves the intermediate mixed-data rate.

\paragraph{Mixed data with $\delta>1$ and $\gamma_2\geq1$.}
In this case, $I_j\asymp n_2j^\delta$ for every $j\geq1$. Since $\sum_{j\geq1}j^{-\delta}<\infty$, taking $m\asymp n_2^{1/(2s)}$ gives
\begin{equation*}
V_m\asymp n_2^{-1}\asymp T_m\asymp n^{-\gamma_2}.
\end{equation*}
The preceding cases cover all parameter regimes and establish the claimed rates.

\section{Proof of Theorem~\ref{thm:scaling-optimal-ridge}}
\label{apx:pf-thm:scaling-optimal-ridge}

For the fixed power-law signal $\btheta_*[j]=j^{-\beta}$, recall that
\begin{equation*}
s=\frac{\alpha_1+2\beta-1}{2},\qquad j^{-\alpha_1}\btheta_*[j]^2=j^{-1-2s}.
\end{equation*}
In particular, $2s-\delta=\alpha_2+2\beta-1>0$. Recall $a=s\wedge\alpha_1$ and $b=s\wedge\alpha_2$. Write $c_{ij}=j^{-\alpha_i}$, and
\begin{equation*}
h_j=nc_{1j}+n_2 c_{2j},\qquad d_j(\lambda)=\mu_1c_{1j}+\mu_2c_{2j}+\lambda.
\end{equation*}
The fixed-point equations and deterministic resolvent are
\begin{equation}
\frac{n_i}{\mu_i}=1+\sum_{j\geq1}\frac{c_{ij}}{d_j(\lambda)},\qquad \overline\bG=\operatorname{Diag}\bigl(d_j(\lambda)^{-1}\bigr)_{j\geq1}.
\label{eq:optimal-ridge-fixed-point}
\end{equation}
Let $\boldsymbol\mu=(\mu_1,\mu_2)^\top$, $\bD=\operatorname{Diag}(n_i/\mu_i^2)_{i=1}^2$, and
\begin{equation*}
\bK_{i\ell}=\sum_{j\geq1}\frac{ c_{ij}c_{\ell j}}{d_j(\lambda)^2},\qquad \bL=\bD-\bK.
\end{equation*} Recall that the deterministic equivalence of the risk is given by \begin{equation*}
    \begin{split}
        \overline{\sR}_1 = &\lambda^2\left[\left\langle\btheta_*,\overline\bG\bC_1\overline\bG\btheta_*\right\rangle+\tau_{\btheta_*}^\top\bL^{-1}\tau_{\bC_1}\right]  + \sum_{i=1}^2\sigma_{\eps_i}^2\,\tau_{\bC_1}^\top\bL^{-1}\be_i.
    \end{split}
\end{equation*}

\paragraph{General proof recipe.} Define
\begin{equation}
B_0(\lambda)=\sum_{j\geq1}c_{1j}\btheta_*[j]^2\left(\frac{\lambda}{h_j+\lambda}\right)^2,\qquad V_0(\lambda)=\sum_{j\geq1}\frac{c_{1j}h_j}{(h_j+\lambda)^2}.
\label{eq:optimal-ridge-population-sums}
\end{equation}
Pick a positive sequence $\lambda_n$ such that
\begin{equation*}
\epsilon_n:=\max_{i\in\{1,2\}}\sum_{j\geq1}\frac{c_{ij}}{h_j+\lambda_n}=o(1).
\end{equation*}

We first show that \begin{equation}
\min\{B_0(\lambda_n),V_0(\lambda_n)\}\lesssim\inf_{\lambda>0}\overline\sR_1(\lambda)\lesssim B_0(\lambda_n)+V_0(\lambda_n) \asymp \max\{B_0(\lambda_n),V_0(\lambda_n)\}.
\label{eq:optimal-ridge-oracle-comparison}
\end{equation}

Then, for each regime, we construct an explicit feasible $\lambda_n,$ such that $B_0(\lambda_n) \asymp V_0(\lambda_n)$ and $\eps_n = o(1).$ This gives the tight scaling law of the optimal ridge risk.

\paragraph{Control $\bL^{-1}$ terms.}
We first control the order of terms that depend on $\bL^{-1}$ to mitigate the matrix inversion problem. 

First, we have the following identities from direct computations: \begin{equation*}
    \begin{split}
        &[\bD \bmu]_i = \frac{n_i}{\mu_i}, \quad [\bK \bmu]_i = \sum_{j \geq 1} \frac{c_{ij} \sum_\ell \mu_\ell c_{\ell j}}{d_j(\lambda)^2} = \sum_{j \geq 1} \frac{c_{ij} (d_j(\lambda) - \lambda))}{d_j(\lambda)^2} = \frac{n_i}{\mu_i} -1 - \lambda \tau_{\bI}[i], \\
        &\tau_{\bI}[i]=\sum_{j\geq1}\frac{c_{ij}}{d_j(\lambda)^2}.
    \end{split}
\end{equation*}

Together with  $\bL = \bD -\bK$ we have $\bL\boldsymbol\mu=\mathbf1+\lambda\tau_{\bI}$ 
 and 
\begin{equation*}
    \begin{split}
[\bD^{-1}\bK\boldsymbol\mu]_i = \mu_i - \frac{\mu_i^2}{n_i}(1 + \lambda\tau_\bI[i]) < \mu_i,
    \end{split}
\end{equation*} which means $\bD^{-1}\bK\boldsymbol\mu<\boldsymbol\mu.$ Note that  $\bD^{-1}\bK$ and $\bmu$ are both entrywise strictly positive, implying that the spectral radius of $\bD^{-1} \bK$ is strictly smaller than $1.$ To see this, denote $\bA = \bD^{-1} \bK$ temporarily and let $(\kappa,\bv)$ be an arbitrary eigenvalue-eigenvector pair. Define $t = \max_{i= 1,2} \frac{|v_i|}{\mu_i}.$  Clearly, for $i = 1,2,$  $|v_i| \leq t \mu_i,$ and assume w.l.o.g.\ $|v_1| = t \mu_1.$  Now we have: \begin{equation*}
    \begin{split}
        |\kappa| |v_1| = &|[\bA\bv]_1| \leq \sum_{\ell} \bA_{1 \ell} |\bv_\ell|  < t [\bA \bmu]_1 \leq t \mu_1 = |v_1|,
    \end{split}
\end{equation*} which implies that the absolute value of all the eigenvalues is strictly smaller than $1.$ Thus, the matrix $\bI - \bA$ is invertible, and the inverse can be expanded into the following Taylor series $(\bI - \bA)^{-1} = \sum_{\ell \geq 0} (\bA)^{\ell},$ whose convergence is guaranteed by the contraction of $\bA.$ Now note that by definition $\bL = \bD(\bI - \bD^{-1} \bK)$. Consequently,
\begin{equation*}
\bL^{-1}=\sum_{k\geq0}(\bD^{-1}\bT)^k\bD^{-1},\qquad (\bL^{-1})_{i\ell}\geq(\bD^{-1})_{i\ell}, \quad \forall i,\ell \in \{1, 2\}.
\end{equation*}
This means that $\bL^{-1}$ is entrywise non-negative and further, by the fact that $\bL \bmu = \mathbf1 + \lambda \tau_\bI$ and the fact that $\bL$ is invertible, we have  $ \bL^{-1}\mathbf1\leq\boldsymbol\mu$ entrywise.

 Now recall that the variance $\overline \sV_1 = \sum_{i=1}^2\sigma_{\eps_i}^2\,\tau_{\bC_1}^\top\bL^{-1}\be_i = \tau_{\bC_1}^\top\bL^{-1} \bsigma,$ where we define $\bsigma = [\sigma^2_{\eps_1}, \sigma^2_{\eps_2}]^\top.$ Using that $\bL^{-1} \geq \bD^{-1}, \bL^{-1} \mathbf1 \leq \mu$ entrywise and $\sigma^2_{\eps_1}, \sigma^2_{\eps_2}$ are constants, we have
\begin{equation}
C_{*,-}\sum_{i=1}^2\frac{\mu_i^2}{n_i}\tau_{\bC_1}[i]\leq\overline\sV_1(\lambda)\leq C_{*,+}\sum_{i=1}^2\mu_i\tau_{\bC_1}[i].
\label{eq:optimal-ridge-variance-bounds}
\end{equation}

Next, consider the term $\lambda^2 \tau_{\btheta_*}^\top\bL^{-1}\tau_{\bC_1}$. We have 
\begin{equation*}
0\leq\lambda^2\tau_{\btheta_*\btheta_*^\top}[i]=\sum_{j\geq1}c_{ij}\btheta_*[j]^2\left(\frac{\lambda}{d_j(\lambda)}\right)^2\leq \sum_{j\geq1}c_{ij}\btheta_*[j]^2 = \sum_{j \geq 1} j^{- \alpha_i - 2 \beta} < \infty.
\end{equation*}
By the entrywise nonnegativity of $\bL^{-1}$, this term is bounded above by a constant times $\overline\sV_1$. Thus, 
we have \begin{equation*}
    \begin{split}
        \overline{\sR}_1 \asymp & \lambda^2\left\langle\btheta_*,\overline\bG\bC_1\overline\bG\btheta_*\right\rangle  + \overline{\sV}_1
    \end{split}
\end{equation*}

uniformly over $\lambda>0$.

\paragraph{Reduction of the optimal ridge risk to two spectral sums.}
Next, we study the risk under the optimal ridge regularization, and we aim to show \eqref{eq:optimal-ridge-oracle-comparison}.

Denote $\tilde \sB(\lambda) = \lambda^2\left\langle\btheta_*,\overline\bG\bC_1\overline\bG\btheta_*\right\rangle$  and we show that under the above choice of $\lambda_n,$ $$\overline{\sV}_1(\lambda_n) \asymp V_0(\lambda_n), \quad \tilde \sB(\lambda_n) \asymp B_0(\lambda_n) .$$ To see this, it is sufficient to show $\mu_i(\lambda_n)/n_i = \Theta(1), i=1,2.$ Let $\eta=\min_i\mu_i(\lambda_n)/n_i\in(0,1]$. By the self-consistent equation and recalling $d_j(\lambda_n) = \mu_1 c_{1j} + \mu_2 c_{2j}+\lambda_n, h_j = n_1\mu_1  + n_2\mu_2 c_{2j},$ we have $d_j(\lambda_n)\geq\eta(h_j+\lambda_n)$, and \eqref{eq:optimal-ridge-fixed-point} gives
\begin{equation*}
\frac{\mu_i(\lambda_n)}{n_i}\geq\frac1{1+\epsilon_n/\eta_n}, 
\end{equation*} which is equivalent to $\eta_n\geq\frac{\eta_n}{\eta_n+\epsilon_n}$ and
hence $1-\epsilon_n\leq\eta_n\leq1$. This always implies  $d_j(\lambda_n)\asymp h_j+\lambda_n$  uniformly in $j$ as we have the upper bound $d_j(\lambda_n) \leq h_j + \lambda$ due to $\mu_i \leq n_i$ by definition. This gives that $\overline \sV_1 \asymp V_0(\lambda_n),$ and we obtain the upper bound \begin{equation*}
    \begin{split}
\inf_{\lambda>0}\overline\sR_1(\lambda)\leq \sR_1(\lambda_n) \asymp B_0(\lambda_n)+V_0(\lambda_n).
    \end{split}
\end{equation*}

Next we prove the lower bound on $\inf_{\lambda>0}\overline\sR_1(\lambda).$ For simplicity denote  $\tilde\sB(\lambda) := \lambda^2 \< \btheta_*, \overline{\bG} \bC_1 \overline{\bG} \btheta_*\>.$ Fix any feasible $\lambda_n,$ we aim to show that for all $\lambda \geq \lambda_n$, $ \overline \sR_1(\lambda) \geq \tilde \sB(\lambda)  \geq B_0(\lambda_n);$ and for all $0< \lambda\leq \lambda_n$, $\sR_1(\lambda) \geq \overline{\sV}_1(\lambda)  \geq V_0(\lambda_n),$ which implies the lower bound.

When $\lambda > \lambda_n,$ it is sufficient to show that $\tilde \sB(\lambda) > \tilde \sB(\lambda_n).$ We prove the identity via the monotonicity of $\tilde \sB(\lambda)$ w.r.t $\lambda.$ In particular, define $q_j(\lambda) = \sum_{i=1}^2 \frac{c_{ij}\mu_i(\lambda)}{\lambda},$ and we have: \begin{equation*}
    \begin{split}
        \tilde \sB(\lambda) = \sum_{j} c_{1j} \btheta_*[j]^2 \frac{1}{(1+q_j(\lambda))^2}.
    \end{split}
\end{equation*}
It is sufficient to show that $q_j(\lambda)$ is monotone decreasing in $\lambda.$ Hence, we compute the derivative of $\frac{\bmu(\lambda)}{\lambda},$ and have: \begin{equation*}
    \begin{split}
        \frac{\dd}{\dd \lambda} \frac{\bmu(\lambda)}{\lambda} = &\frac{\lambda \bmu' - \bmu}{\lambda^2} = \frac{\lambda \bL^{-1} \tau_{\bI} - \bL^{-1}(1+\lambda \tau_{\bI})}{\lambda^2} = - \frac{1}{\lambda^2} \bL^{-1} \mathbf 1 <0,
    \end{split} 
\end{equation*} where here we use the derivative computation: \begin{equation*}
    \begin{split}
        \bL \bmu' = \tau_{\bI},\quad \bL \bmu = \mathbf{1} + \lambda \pi_\bI.
    \end{split}
\end{equation*} To obtain the identity $ \bL \bmu' = \tau_{\bI},$ note that taking the derivative of the self-consistency equation gives \begin{equation*}
    \begin{split}
        - \frac{n_i}{\mu_i^2} \mu_i' =& - \sum_{j} \frac{c_{ij}}{d_j^2(\lambda)} d_j'(\lambda)  =  - \sum_{j} \frac{c_{ij}}{d_j^2(\lambda)} (1 + \sum_{\ell=1}^2\mu_\ell' c_{\ell j})= -\tau_\bI[i] - [\bK \bmu']_i.
    \end{split}
\end{equation*} Noting that $\frac{n_i}{\mu_i^2} \mu_i' = [\bD \bmu']_i$ and $\bL = \bD -\bK$ finishes the proof.

When $0<\lambda<\lambda_n$, we use the monotonicity of a lower bound on $\overline{\sV}_1(\lambda)$. By Cauchy--Schwarz, for every $j\geq1$,
\begin{equation*}
    \begin{split}
        \left(\sum_{i=1}^2\mu_i c_{ij}\right)^2
        &=\left(\sum_{i=1}^2\sqrt{n_i c_{ij}}\,\mu_i\sqrt{\frac{c_{ij}}{n_i}}\right)^2\\
        &\leq\left(\sum_{i=1}^2n_i c_{ij}\right)\left(\sum_{i=1}^2\frac{\mu_i^2}{n_i}c_{ij}\right)
        =h_j\sum_{i=1}^2\frac{\mu_i^2}{n_i}c_{ij}.
    \end{split}
\end{equation*}
Using \eqref{eq:optimal-ridge-variance-bounds}, we obtain
\begin{equation*}
    \begin{split}
        \overline{\sV}_1(\lambda)
        &\geq C_{*,-}\sum_{i=1}^2\frac{\mu_i^2}{n_i}\tau_{\bC_1}[i]\\
        &=C_{*,-}\sum_{j\geq1}\frac{c_{1j}}{d_j(\lambda)^2}\sum_{i=1}^2\frac{\mu_i^2}{n_i}c_{ij}\\
        &\geq C_{*,-}\sum_{j\geq1}\frac{c_{1j}}{h_j}
        \left(\frac{\sum_{i=1}^2\mu_i c_{ij}}{d_j(\lambda)}\right)^2\\
        &=C_{*,-}\sum_{j\geq1}\frac{c_{1j}}{h_j}
        \left(\frac{q_j(\lambda)}{1+q_j(\lambda)}\right)^2.
    \end{split}
\end{equation*}
Since $q_j(\lambda)$ is decreasing in $\lambda$ and $q\mapsto q/(1+q)$ is increasing for $q\geq0$, we have
\begin{equation*}
    \frac{q_j(\lambda)}{1+q_j(\lambda)}
    \geq\frac{q_j(\lambda_n)}{1+q_j(\lambda_n)},
    \qquad 0<\lambda\leq\lambda_n.
\end{equation*}
Moreover, the previously established comparisons $\mu_i(\lambda_n)\asymp n_i$ and $d_j(\lambda_n)\asymp h_j+\lambda_n$ imply
\begin{equation*}
    \frac{q_j(\lambda_n)}{1+q_j(\lambda_n)}
    =\frac{\sum_{i=1}^2\mu_i(\lambda_n)c_{ij}}{d_j(\lambda_n)}
    \asymp\frac{h_j}{h_j+\lambda_n}
\end{equation*}
uniformly in $j$. Consequently, for all $0<\lambda\leq\lambda_n$,
\begin{equation*}
    \begin{split}
        \overline{\sR}_1(\lambda)
        &\geq\overline{\sV}_1(\lambda)\\
        &\geq C_{*,-}\sum_{j\geq1}\frac{c_{1j}}{h_j}
        \left(\frac{q_j(\lambda_n)}{1+q_j(\lambda_n)}\right)^2\\
        &\gtrsim\sum_{j\geq1}\frac{c_{1j}h_j}{(h_j+\lambda_n)^2}
        =V_0(\lambda_n).
    \end{split}
\end{equation*}
Combining this bound with the lower bound for $\lambda\geq\lambda_n$ gives
\begin{equation*}
    \inf_{\lambda>0}\overline{\sR}_1(\lambda)
    \gtrsim\min\{B_0(\lambda_n),V_0(\lambda_n)\}.
\end{equation*}
Together with the upper bound proved above, this establishes \eqref{eq:optimal-ridge-oracle-comparison}.


\paragraph{Evaluation of the spectral sums.}
We next evaluate $B_0$, $V_0$, and $\epsilon_n$ under the power-law model. Recall that: \begin{equation*}
    \begin{split}
        B_0(\lambda)=\sum_{j\geq1}c_{1j}\btheta_*[j]^2\left(\frac{\lambda}{h_j+\lambda}\right)^2,\qquad V_0(\lambda)=\sum_{j\geq1}\frac{c_{1j}h_j}{(h_j+\lambda)^2}.
    \end{split}
\end{equation*} For the fixed power-law signal $\btheta_*[j]=j^{-\beta}$, recall that
\begin{equation*}
s=\frac{\alpha_1+2\beta-1}{2},\qquad j^{-\alpha_1}\btheta_*[j]^2=j^{-1-2s}.
\end{equation*}
In particular, $2s-\delta=\alpha_2+2\beta-1>0$. Recall $a=s\wedge\alpha_1$ and $b=s\wedge\alpha_2$. Define $c_{ij}=j^{-\alpha_i}$, and
\begin{equation*}
h_j=nc_{1j}+n_2 c_{2j},\qquad d_j(\lambda)=\mu_1c_{1j}+\mu_2c_{2j}+\lambda.
\end{equation*}

For $m\geq 1$, define
\begin{equation*}
    h_m=nm^{-\alpha_1}+n_2m^{-\alpha_2},
    \qquad \lambda(m)=h_m.
\end{equation*}
Thus, choosing a cutoff $m$ specifies the regularization
$\lambda_n=\lambda(m)$. All comparison constants below are independent of $n$ and $m$, with the other model parameters held fixed.

Since $h_j$ is decreasing in $j$, splitting the bias at $m\geq1$ gives
\begin{equation}
    \begin{split}
        B_0(h_m)
        &\asymp
        h_m^2\sum_{j\leq m}\frac{j^{-1-2s}}{h_j^2}
        +\sum_{j>m}j^{-1-2s}\\
        &\asymp
        h_m^2\sum_{j\leq m}\frac{j^{-1-2s}}{h_j^2}
        +m^{-2s}.
    \end{split}
    \label{eq:optimal-ridge-cutoff-bias}
\end{equation}
Similarly,
\begin{equation*}
    \begin{split}
        V_0(h_m)
        &\asymp
        \sum_{j\leq m}\frac{c_{1j}}{h_j}
        +h_m^{-2}\sum_{j>m}c_{1j}h_j,\\
        h_m^{-2}\sum_{j>m}c_{1j}h_j
        &\asymp
        \frac{nm^{1-2\alpha_1}
        +n_2m^{1-\alpha_1-\alpha_2}}{h_m^2}
        =\frac{m}{n+n_2m^\delta}.
    \end{split}
\end{equation*}
Note that $\sum_{j = m/2}^m\frac{c_{1j}}{h_j} = \sum_{j = m/2}^m\frac{j^{-\alpha_1}}{n j^{-\alpha_1} + n_2 j^{-\alpha_2} } \asymp \frac{m}{n + n_2 m^{\delta}}.$

Consequently,
\begin{equation}
    V_0(h_m)\asymp
    \sum_{j\leq m}\frac{1}{n+n_2j^\delta}.
    \label{eq:optimal-ridge-cutoff-variance}
\end{equation}

For the feasibility condition, define
\begin{equation*}
    \epsilon(m)
    :=\max_{i\in\{1,2\}}
    \sum_{j\geq1}\frac{c_{ij}}{h_j+h_m}.
\end{equation*}
Let $\alpha_{\min}=\min\{\alpha_1,\alpha_2\}$ and
$\delta_+=\max\{\delta,0\}$. Since
$\max_i c_{ij}=j^{-\alpha_{\min}}$ for every $j$, we have
\begin{equation}
    \begin{split}
        \epsilon(m)
        &=\sum_{j\geq1}\frac{j^{-\alpha_{\min}}}{h_j+h_m}\\
        &\asymp
        \sum_{j\leq m}\frac{j^{\delta_+}}{n+n_2j^\delta}
        +\frac{m^{1-\alpha_{\min}}}{h_m}\\
        &\asymp
        \frac{m^{1+\delta_+}}{n+n_2m^\delta}.
    \end{split}
    \label{eq:optimal-ridge-cutoff-trace}
\end{equation}
For the last comparison, note that $\frac{m^{1-\alpha_{\min}}}{h_m} \asymp \frac{m^{1+\delta_+}}{n+n_2m^\delta},$ and $\frac{j^{\delta_+}}{n+n_2j^\delta}$ is non-decreasing in $j$ which implies $ \sum_{j\leq m}\frac{j^{\delta_+}}{n+n_2j^\delta} \leq \frac{m^{1+\delta_+}}{n+n_2m^\delta}.$

We will retain the logarithmic factors arising at the critical source exponents. Write
\begin{equation*}
    \ell(t)=1+\log t,\qquad
    \ell_i(t)=1+\mathbf1_{\{s=\alpha_i\}}\log t,
    \qquad t\geq1.
\end{equation*}

Next, we discuss how different choices of $m$ give different scalings in each regime.

\paragraph{Spectral sums when $\delta>0$.}
In this case $\alpha_1>\alpha_2$ and $a\geq b$. Define
\begin{equation*}
    J=\max\left\{1,\left(\frac{n}{n_2}\right)^{1/\delta}\right\}.
\end{equation*}
Recall that
\begin{equation*}
    h_j=nj^{-\alpha_1}+n_2j^{-\alpha_2},
    \qquad
    \ell(t)=1+\log t,
    \qquad
    \ell_i(t)=1+\mathbf1_{\{s=\alpha_i\}}\log t.
\end{equation*}
We will repeatedly use the following consequence of
Lemma~\ref{lem:power-law-sums}: for $i\in\{1,2\}$ and $x\geq1$,
\begin{equation*}
    \sum_{j\leq x}j^{2\alpha_i-1-2s}
    \asymp
    \begin{cases}
        x^{2\alpha_i-2s},&s<\alpha_i,\\
        \ell(x),&s=\alpha_i,\\
        1,&s>\alpha_i.
    \end{cases}
\end{equation*}

When $n_2\leq n$, the target dataset dominates $h_j$ for $j\leq J$,
whereas the auxiliary dataset dominates for $j\geq J$.
When $n_2>n$, we have $J=1$ and
$h_j\asymp n_2j^{-\alpha_2}$ for every $j\geq1$.

For $1\leq m\leq J$ with $n_2\leq n$,
$h_m\asymp nm^{-\alpha_1}$, and
\eqref{eq:optimal-ridge-cutoff-bias} gives
\begin{equation*}
    B_0(h_m)
    \asymp
    m^{-2s}
    +m^{-2\alpha_1}\sum_{j\leq m}j^{2\alpha_1-1-2s}.
\end{equation*}
The second term has order
\begin{equation*}
    m^{-2\alpha_1}\sum_{j\leq m}j^{2\alpha_1-1-2s}
    \asymp
    \begin{cases}
        m^{-2s},&s<\alpha_1,\\
        m^{-2\alpha_1}\ell(m),&s=\alpha_1,\\
        m^{-2\alpha_1},&s>\alpha_1.
    \end{cases}
\end{equation*}
In each case, $m^{-2s}$ is bounded above by a constant times
this contribution. Consequently,
\begin{equation*}
    B_0(h_m)\asymp m^{-2a}\ell_1(m).
\end{equation*}

For $m\geq J$, splitting the first sum in
\eqref{eq:optimal-ridge-cutoff-bias} at $J$ gives
\begin{equation*}
    \begin{split}
        B_0(h_m)
        \asymp m^{-2s}
        +m^{-2\alpha_2}\left[
        \underbrace{
            J^{-2\delta}\sum_{j\leq J}j^{2\alpha_1-1-2s}
        }_{=:S_+(J)}
        +
        \underbrace{
            \sum_{J<j\leq m}j^{2\alpha_2-1-2s}
        }_{=:T_+(m,J)}
        \right].
    \end{split}
\end{equation*}
When $n_2\leq n$, the factor $J^{-2\delta}$ follows from
$(n_2/n)^2=J^{-2\delta}$.
When $n_2>n$, we have $J=1$, and the same display follows by
separating the term $j=1$ and using
$h_j\asymp n_2j^{-\alpha_2}$ throughout.

Applying Lemma~\ref{lem:power-law-sums} to the two sums separately,
we obtain
\begin{equation*}
    S_+(J)\asymp
    \begin{cases}
        J^{2\alpha_2-2s},&s<\alpha_1,\\
        J^{-2\delta}\ell(J),&s=\alpha_1,\\
        J^{-2\delta},&s>\alpha_1,
    \end{cases}
\end{equation*}
and
\begin{equation*}
    T_+(m,J)\asymp
    \begin{cases}
        m^{2\alpha_2-2s}-J^{2\alpha_2-2s},
        &s<\alpha_2,\\
        \log(m/J),
        &s=\alpha_2,\\
        J^{2\alpha_2-2s}-m^{2\alpha_2-2s},
        &s>\alpha_2.
    \end{cases}
\end{equation*}

We now compare these contributions.
If $s<\alpha_2$, then
\begin{equation*}
    S_+(J)+T_+(m,J)\lesssim m^{2\alpha_2-2s},
\end{equation*}
so the contribution multiplied by $m^{-2\alpha_2}$ is bounded
above by a constant times $m^{-2s}$.
Together with the tail term, this gives
$B_0(h_m)\asymp m^{-2s}$.

If $s=\alpha_2$, then $S_+(J)\asymp1$, and hence
\begin{equation*}
    B_0(h_m)
    \asymp
    m^{-2\alpha_2}\bigl[1+T_+(m,J)\bigr]
    \asymp
    m^{-2\alpha_2}\ell(m/J).
\end{equation*}
Here the constant contribution is retained even when $m=J$.

If $s>\alpha_2$, then
\begin{equation*}
    S_+(J)
    \asymp J^{-2(a-\alpha_2)}\ell_1(J).
\end{equation*}
Moreover,
\begin{equation*}
    \begin{split}
        T_+(m,J)
        &\lesssim J^{2\alpha_2-2s}
        \lesssim S_+(J),\\
        m^{-2s}
        &\leq m^{-2\alpha_2}J^{2\alpha_2-2s}
        \lesssim m^{-2\alpha_2}S_+(J).
    \end{split}
\end{equation*}
Thus the contribution from $S_+(J)$ controls the bias:
\begin{equation*}
    B_0(h_m)
    \asymp
    m^{-2\alpha_2}J^{-2(a-\alpha_2)}\ell_1(J).
\end{equation*}

Define
\begin{equation*}
    \ell_+(m,J)
    =1+\mathbf1_{\{s=\alpha_2\}}\log(m/J)
      +\mathbf1_{\{s=\alpha_1\}}\log J.
\end{equation*}
Combining the preceding cases gives
\begin{equation}
    B_0(h_m)\asymp
    \begin{cases}
        m^{-2a}\ell_1(m),
        &1\leq m\leq J,\ n_2\leq n,\\
        m^{-2b}J^{-2(a-b)}\ell_+(m,J),
        &m\geq J.
    \end{cases}
    \label{eq:optimal-ridge-positive-bias}
\end{equation}

For the variance, recall from
\eqref{eq:optimal-ridge-cutoff-variance} that
\begin{equation*}
    V_0(h_m)\asymp
    \sum_{j\leq m}\frac{1}{n+n_2j^\delta}.
\end{equation*}
If $1\leq m\leq J$ and $n_2\leq n$, then
$n_2j^\delta\leq n$ for every $j\leq m$, so
\begin{equation*}
    V_0(h_m)
    \asymp \frac1n\sum_{j\leq m}1
    \asymp \frac mn.
\end{equation*}
If $m\geq J$, we split the sum at $J$.
The first contribution satisfies
\begin{equation*}
    \sum_{j\leq J}\frac1{n+n_2j^\delta}
    \asymp n_2^{-1}J^{1-\delta}.
\end{equation*}
Indeed, when $n_2\leq n$, it has order
$J/n=n_2^{-1}J^{1-\delta}$.
When $n_2>n$, we have $J=1$, and this contribution equals
$(n+n_2)^{-1}\asymp n_2^{-1}$.

For the second contribution,
$n_2j^\delta\geq n$ whenever $j>J$, and therefore
\begin{equation*}
    \sum_{J<j\leq m}\frac1{n+n_2j^\delta}
    \asymp n_2^{-1}\sum_{J<j\leq m}j^{-\delta}.
\end{equation*}
Lemma~\ref{lem:power-law-sums} gives
\begin{equation*}
    \sum_{J<j\leq m}j^{-\delta}
    \asymp
    \begin{cases}
        m^{1-\delta}-J^{1-\delta},
        &0<\delta<1,\\
        \log(m/J),
        &\delta=1,\\
        J^{1-\delta}-m^{1-\delta},
        &\delta>1.
    \end{cases}
\end{equation*}
Consequently,
\begin{equation*}
    V_0(h_m)
    \asymp n_2^{-1}\left[
        J^{1-\delta}+\sum_{J<j\leq m}j^{-\delta}
    \right].
\end{equation*}
For $0<\delta<1$, adding the first contribution yields
\begin{equation*}
    J^{1-\delta}+\sum_{J<j\leq m}j^{-\delta}
    \asymp
    J^{1-\delta}
    +\bigl(m^{1-\delta}-J^{1-\delta}\bigr)
    \asymp m^{1-\delta}.
\end{equation*}
For $\delta=1$, the two contributions together have order
\begin{equation*}
    1+\sum_{J<j\leq m}j^{-1}
    \asymp 1+\log(m/J)=\ell(m/J).
\end{equation*}
For $\delta>1$, the second contribution satisfies
\begin{equation*}
    0\leq\sum_{J<j\leq m}j^{-\delta}
    \lesssim J^{1-\delta},
\end{equation*}
so it is bounded above by a constant times the first contribution.
We conclude that
\begin{equation}
    V_0(h_m)\asymp
    \begin{cases}
        m/n,
        &1\leq m\leq J,\ n_2\leq n,\\
        n_2^{-1}m^{1-\delta},
        &m\geq J,\ 0<\delta<1,\\
        n_2^{-1}\ell(m/J),
        &m\geq J,\ \delta=1,\\
        n_2^{-1}J^{1-\delta},
        &m\geq J,\ \delta>1.
    \end{cases}
    \label{eq:optimal-ridge-positive-variance}
\end{equation}
When $n_2\leq n$, the applicable expressions agree at $m=J$,
since $J/n=n_2^{-1}J^{1-\delta}$ and $\ell(1)=1$.

Finally, \eqref{eq:optimal-ridge-cutoff-trace} gives
\begin{equation}
    \epsilon(m)\asymp
    \frac{m^{1+\delta}}{n+n_2m^\delta}
    \asymp
    \begin{cases}
        m^{1+\delta}/n,
        &1\leq m\leq J,\ n_2\leq n,\\
        m/n_2,
        &m\geq J.
    \end{cases}
    \label{eq:optimal-ridge-positive-trace}
\end{equation}

\paragraph{Spectral sums when $\delta\leq0$.}
If $\delta=0$, then $\alpha_1=\alpha_2$, $a=b$, and
$h_j=(n+n_2)j^{-\alpha_1}$.
The cutoff formulas give
\begin{equation*}
    \begin{split}
        B_0(h_m)
        &\asymp m^{-2s}
        +m^{-2\alpha_1}\sum_{j\leq m}j^{2\alpha_1-1-2s},\\
        V_0(h_m)
        &\asymp \frac1{n+n_2}\sum_{j\leq m}1.
    \end{split}
\end{equation*}
The prefix-sum estimates above show that the second contribution
to the bias has order $m^{-2a}\ell_1(m)$ and controls $m^{-2s}$.
The variance sum has order $m/(n+n_2)$.
Together with \eqref{eq:optimal-ridge-cutoff-trace}, this gives
\begin{equation}
    B_0(h_m)\asymp m^{-2a}\ell_1(m),
    \qquad
    V_0(h_m)\asymp\epsilon(m)\asymp\frac{m}{n+n_2}.
    \label{eq:optimal-ridge-equal-spectra}
\end{equation}

Suppose next that $\delta<0$.
Here $\alpha_1<\alpha_2$ and $a\leq b$. Define
\begin{equation*}
    J=\max\left\{1,\left(\frac{n_2}{n}\right)^{1/(-\delta)}\right\}.
\end{equation*}
Again write $m=\lfloor m\rfloor$ and $J=\lfloor J\rfloor$
when evaluating interval sums.
When $n_2\geq n$, the auxiliary dataset dominates below $J$ and
the target dataset dominates above $J$.
When $n_2<n$, we have $J=1$ and
$h_j\asymp nj^{-\alpha_1}$ for every $j\geq1$.

For $1\leq m\leq J$ with $n_2\geq n$,
$h_m\asymp n_2m^{-\alpha_2}$, and
\eqref{eq:optimal-ridge-cutoff-bias} gives
\begin{equation*}
    B_0(h_m)
    \asymp
    m^{-2s}
    +m^{-2\alpha_2}\sum_{j\leq m}j^{2\alpha_2-1-2s}.
\end{equation*}
The second contribution satisfies
\begin{equation*}
    m^{-2\alpha_2}\sum_{j\leq m}j^{2\alpha_2-1-2s}
    \asymp
    \begin{cases}
        m^{-2s},&s<\alpha_2,\\
        m^{-2\alpha_2}\ell(m),&s=\alpha_2,\\
        m^{-2\alpha_2},&s>\alpha_2.
    \end{cases}
\end{equation*}
It controls the tail term $m^{-2s}$ in every case, so
\begin{equation*}
    B_0(h_m)\asymp m^{-2b}\ell_2(m).
\end{equation*}

For $m\geq J$, splitting at $J$ gives
\begin{equation*}
    \begin{split}
        B_0(h_m)
        \asymp m^{-2s}
        +m^{-2\alpha_1}\left[
        \underbrace{
            J^{2\delta}\sum_{j\leq J}j^{2\alpha_2-1-2s}
        }_{=:S_-(J)}
        +
        \underbrace{
            \sum_{J<j\leq m}j^{2\alpha_1-1-2s}
        }_{=:T_-(m,J)}
        \right].
    \end{split}
\end{equation*}
When $n_2\geq n$, the factor $J^{2\delta}$ follows from
$(n/n_2)^2=J^{2\delta}$.
When $n_2<n$, the same display follows by taking $J=1$,
separating the term $j=1$, and using
$h_j\asymp nj^{-\alpha_1}$ throughout.

Applying Lemma~\ref{lem:power-law-sums} to each sum gives
\begin{equation*}
    S_-(J)\asymp
    \begin{cases}
        J^{2\alpha_1-2s},&s<\alpha_2,\\
        J^{2\delta}\ell(J),&s=\alpha_2,\\
        J^{2\delta},&s>\alpha_2,
    \end{cases}
\end{equation*}
and
\begin{equation*}
    T_-(m,J)\asymp
    \begin{cases}
        m^{2\alpha_1-2s}-J^{2\alpha_1-2s},
        &s<\alpha_1,\\
        \log(m/J),
        &s=\alpha_1,\\
        J^{2\alpha_1-2s}-m^{2\alpha_1-2s},
        &s>\alpha_1.
    \end{cases}
\end{equation*}

If $s<\alpha_1$, then
\begin{equation*}
    S_-(J)+T_-(m,J)\lesssim m^{2\alpha_1-2s}.
\end{equation*}
Thus both contributions to the bias are bounded above by a
constant times $m^{-2s}$, and the tail term gives
$B_0(h_m)\asymp m^{-2s}$.

If $s=\alpha_1$, then $S_-(J)\asymp1$, and
\begin{equation*}
    B_0(h_m)
    \asymp
    m^{-2\alpha_1}\bigl[1+T_-(m,J)\bigr]
    \asymp
    m^{-2\alpha_1}\ell(m/J).
\end{equation*}

If $s>\alpha_1$, then
\begin{equation*}
    S_-(J)
    \asymp J^{-2(b-\alpha_1)}\ell_2(J).
\end{equation*}
Furthermore,
\begin{equation*}
    \begin{split}
        T_-(m,J)
        &\lesssim J^{2\alpha_1-2s}
        \lesssim S_-(J),\\
        m^{-2s}
        &\leq m^{-2\alpha_1}J^{2\alpha_1-2s}
        \lesssim m^{-2\alpha_1}S_-(J).
    \end{split}
\end{equation*}
Consequently,
\begin{equation*}
    B_0(h_m)
    \asymp
    m^{-2\alpha_1}J^{-2(b-\alpha_1)}\ell_2(J).
\end{equation*}

Define
\begin{equation*}
    \ell_-(m,J)
    =1+\mathbf1_{\{s=\alpha_1\}}\log(m/J)
      +\mathbf1_{\{s=\alpha_2\}}\log J.
\end{equation*}
Combining these estimates yields
\begin{equation}
    B_0(h_m)\asymp
    \begin{cases}
        m^{-2b}\ell_2(m),
        &1\leq m\leq J,\ n_2\geq n,\\
        m^{-2a}J^{-2(b-a)}\ell_-(m,J),
        &m\geq J.
    \end{cases}
    \label{eq:optimal-ridge-negative-bias}
\end{equation}

For the variance, first suppose that $1\leq m\leq J$ and
$n_2\geq n$. Since $n_2j^\delta\geq n$ for $j\leq m$,
\begin{equation*}
    V_0(h_m)
    \asymp
    n_2^{-1}\sum_{j\leq m}j^{-\delta}.
\end{equation*}
Here $1-\delta>0$, so Lemma~\ref{lem:power-law-sums} gives
\begin{equation*}
    \sum_{j\leq m}j^{-\delta}\asymp m^{1-\delta},
    \qquad
    V_0(h_m)\asymp n_2^{-1}m^{1-\delta}.
\end{equation*}

For $m\geq J$, split the variance sum as
\begin{equation*}
    V_0(h_m)\asymp
    \sum_{j\leq J}\frac1{n+n_2j^\delta}
    +\sum_{J<j\leq m}\frac1{n+n_2j^\delta}.
\end{equation*}
If $n_2\geq n$, the first contribution satisfies
\begin{equation*}
    \sum_{j\leq J}\frac1{n+n_2j^\delta}
    \asymp
    n_2^{-1}\sum_{j\leq J}j^{-\delta}
    \asymp
    n_2^{-1}J^{1-\delta}
    =\frac Jn,
\end{equation*}
where we used $n_2=nJ^{-\delta}$.
If $n_2<n$, then $J=1$, and the same contribution equals
$(n+n_2)^{-1}\asymp1/n=J/n$.

For the second contribution, $n_2j^\delta\leq n$ whenever $j>J$.
Therefore,
\begin{equation*}
    \sum_{J<j\leq m}\frac1{n+n_2j^\delta}
    \asymp
    \frac1n\sum_{J<j\leq m}1
    =\frac{m-J}{n}.
\end{equation*}
Adding the two contributions gives
\begin{equation*}
    V_0(h_m)
    \asymp
    \frac{J+m-J}{n}
    \asymp \frac mn.
\end{equation*}
This estimate also holds at $m=J$, when the second sum is empty.

Finally, since $\delta_+=0$,
\eqref{eq:optimal-ridge-cutoff-trace} gives
\begin{equation*}
    \epsilon(m)\asymp\frac{m}{n+n_2m^\delta},
\end{equation*}
which has the same two orders as the variance. Hence
\begin{equation}
    V_0(h_m)\asymp\epsilon(m)\asymp
    \begin{cases}
        n_2^{-1}m^{1-\delta},
        &1\leq m\leq J,\ n_2\geq n,\\
        m/n,
        &m\geq J.
    \end{cases}
    \label{eq:optimal-ridge-negative-variance}
\end{equation}

\paragraph{Feasibility of a balanced cutoff.}
Before choosing $m$, we show that any sequence satisfying
\begin{equation*}
    m\longrightarrow\infty,
    \qquad
    B_0(h_{m})\asymp V_0(h_{m})
\end{equation*}
also satisfies $\epsilon(m)=o(1)$. This verifies the feasibility condition for the constructions below.

First note that
\begin{equation*}
    2a-\delta
    =\min\{2s-\delta,\alpha_1+\alpha_2\}>0.
\end{equation*}
We also have $2b>\delta$ when $0<\delta\leq1$, and $2b>1$ when $\delta>1$.

If $\delta\leq0$, the preceding estimates give, with $m=m$,
\begin{equation*}
    \epsilon(m)\asymp V_0(h_m)
    \asymp B_0(h_m)
    \lesssim m^{-2a}\ell(m)=o(1).
\end{equation*}
Here we used $b\geq a$ and $J\geq1$ when $\delta<0$.

Suppose $\delta>0$. If $m\leq J$, then
\begin{equation*}
    \epsilon(m)\asymp m^\delta V_0(h_m)
    \asymp m^\delta B_0(h_m)
    \lesssim m^{\delta-2a}\ell(m)=o(1).
\end{equation*}
If $m\geq J$ and $0<\delta\leq1$, then
\begin{equation*}
    \begin{split}
        \epsilon(m)
        &\asymp m/n_2
        \lesssim m^\delta V_0(h_m)
        \asymp m^\delta B_0(h_m)\\
        &\lesssim m^{\delta-2b}\ell(m)=o(1).
    \end{split}
\end{equation*}
Finally, if $m\geq J$ and $\delta>1$, then
\begin{equation*}
    \begin{split}
        \epsilon(m)
        &\asymp mJ^{\delta-1}V_0(h_m)
        \asymp mJ^{\delta-1}B_0(h_m)\\
        &\lesssim
        m^{1-2b}J^{\delta-1-2(a-b)}\ell(m)\\
        &\leq
        m^{-\min\{2b-1,\,2a-\delta\}}\ell(m)
        =o(1).
    \end{split}
\end{equation*}
The last inequality follows from $1\leq J\leq m$.

Consequently, for every balanced sequence constructed below,
$\lambda_n=h_{m}$ is feasible, and
\eqref{eq:optimal-ridge-oracle-comparison} gives
\begin{equation}
    \overline\sR_1^*
    :=\inf_{\lambda>0}\overline\sR_1(\lambda)
    \asymp B_0(h_{m})
    \asymp V_0(h_{m}).
    \label{eq:optimal-ridge-balanced-cutoff}
\end{equation}

\paragraph{Optimal scaling when $\delta>0$.}
Define the threshold
\begin{equation*}
    \gamma_{\rm c}^{+}
    =1-\frac{\delta}{1+2a}.
\end{equation*}
Since $2a>\delta$, we have $0<\gamma_{\rm c}^{+}<1$.

First suppose $\gamma_2<\gamma_{\rm c}^{+}$. Let
\begin{equation*}
    u_n=n^{1/(1+2a)},
    \qquad
    m=u_n\ell_1(u_n)^{1/(1+2a)}.
\end{equation*}
Since  $\gamma_2<\gamma_{\rm c}^{+} < 1,$ we have $J = n^{\frac{1-\gamma_2}{\delta}},$ which implies
$u_n/J = n^{\frac{\gamma_2 -1}{\delta} - \frac{1}{1+2 a}} = o(1).$  Thus $m\leq J$ for all sufficiently large $n$, and
$\ell_1(m)\asymp\ell_1(u_n)$. Therefore,
\begin{equation*}
    \begin{split}
        B_0(h_{m})
        &\asymp m^{-2a}\ell_1(m)\\
        &\asymp
        n^{-2a/(1+2a)}
        \ell_1(u_n)^{1/(1+2a)}
        \asymp \frac{m}{n}
        \asymp V_0(h_{m}).
    \end{split}
\end{equation*}
Applying \eqref{eq:optimal-ridge-balanced-cutoff} gives
\begin{equation*}
    \overline\sR_1^*
    \asymp
    n^{-2a/(1+2a)}
    \ell_1(u_n)^{1/(1+2a)}.
\end{equation*}
In particular, the polynomial exponent is
\begin{equation*}
    \Gamma=\frac{2a}{1+2a}.
\end{equation*}

Next suppose $\gamma_2=\gamma_{\rm c}^{+}$. Here
\begin{equation*}
    J\asymp n^{1/(1+2a)},
    \qquad
    \frac{J}{n}\asymp J^{-2a}.
\end{equation*}
If $s\ne\alpha_1$, choosing $m=J$ immediately gives
\begin{equation*}
    B_0(h_{m})\asymp V_0(h_{m})
    \asymp \frac{J}{n}
    \asymp n^{-2a/(1+2a)}.
\end{equation*}

If $s=\alpha_1$, then $a=\alpha_1$, $b=\alpha_2$, and
$a-b=\delta$. Put $L_n=\ell(J)$ and write $m=Jt$ with $t\geq1$.
Equations \eqref{eq:optimal-ridge-positive-bias} and
\eqref{eq:optimal-ridge-positive-variance} give
\begin{equation*}
    B_0(h_{Jt})\asymp\frac{J}{n}t^{-2b}L_n,
    \qquad
    V_0(h_{Jt})\asymp\frac{J}{n}
    \begin{cases}
        t^{1-\delta},&0<\delta<1,\\
        \ell(t),&\delta=1,\\
        1,&\delta>1.
    \end{cases}
\end{equation*}
We therefore choose $m=Jt_n$, where
\begin{equation*}
    t_n=
    \begin{cases}
        L_n^{1/(1+2b-\delta)},&0<\delta<1,\\
        \bigl(L_n/\ell(L_n)\bigr)^{1/(2b)},&\delta=1,\\
        L_n^{1/(2b)},&\delta>1.
    \end{cases}
\end{equation*}
For $\delta=1$, we used
$\ell(t_n)\asymp\ell(L_n)$. These choices yield
\begin{equation*}
    \overline\sR_1^*
    \asymp B_0(h_{m})
    \asymp V_0(h_{m})
    \asymp \frac{J}{n}
    \begin{cases}
        L_n^{(1-\delta)/(1+2b-\delta)},&0<\delta<1,\\
        \ell(L_n),&\delta=1,\\
        1,&\delta>1.
    \end{cases}
\end{equation*}
Thus the boundary $\gamma_2=\gamma_{\rm c}^{+}$ has the same polynomial exponent
$\Gamma=2a/(1+2a)$. In particular, the fixed-signal critical case
$s=\alpha_1$, $\delta=1$, and $\gamma_2=\gamma_{\rm c}^{+}$
has an additional factor of order $\log\log n$.

It remains to consider $\gamma_2>\gamma_{\rm c}^{+}$. We treat the three variance regimes separately.

If $0<\delta<1$, define
\begin{equation*}
    u_n=
    \left(n_2J^{-2(a-b)}\right)^{1/(1+2b-\delta)},
    \qquad
    m=u_n\ell_+(u_n,J)^{1/(1+2b-\delta)}.
\end{equation*}
The denominator is positive because $2b>\delta$. Moreover,
\begin{equation*}
    \frac{u_n}{J}
    =
    \left(n_2J^{-(1+2a-\delta)}\right)^{1/(1+2b-\delta)}
    \longrightarrow\infty.
\end{equation*}
 Hence $m\geq J$ and
$\ell_+(m,J)\asymp\ell_+(u_n,J)$. By construction,
\begin{equation*}
    u_n^{-2b}J^{-2(a-b)}
    =n_2^{-1}u_n^{1-\delta}.
\end{equation*}
Consequently,
\begin{equation*}
    \begin{split}
        B_0(h_{m})
        &\asymp
        m^{-2b}J^{-2(a-b)}\ell_+(m,J)\\
        &\asymp
        n_2^{-1}u_n^{1-\delta}
        \ell_+(u_n,J)^{(1-\delta)/(1+2b-\delta)}\\
        &\asymp n_2^{-1}m^{1-\delta}
        \asymp V_0(h_{m}).
    \end{split}
\end{equation*}
Since
\begin{equation*}
    n_2^{-1}u_n^{1-\delta}
    =
    n_2^{-2b/(1+2b-\delta)}
    J^{-2(a-b)(1-\delta)/(1+2b-\delta)}
\end{equation*}
and
$J\asymp n^{(1-\gamma_2)_+/\delta}$, the polynomial exponent is
\begin{equation*}
    \Gamma=
    \frac{
        2b\gamma_2+
        \frac{2(1-\delta)}{\delta}(a-b)(1-\gamma_2)_+
    }{1+2b-\delta}.
\end{equation*}

If $\delta=1$, define
\begin{equation*}
    u_n=\left(n_2J^{-2(a-b)}\right)^{1/(2b)},
    \qquad
    m=u_n
    \left(\frac{\ell_+(u_n,J)}{\ell(u_n/J)}\right)^{1/(2b)}.
\end{equation*}
Again, $u_n/J$ grows as a positive power of $n$. The logarithmic correction preserves $m/J\to\infty$, and
\begin{equation*}
    \ell_+(m,J)\asymp\ell_+(u_n,J),
    \qquad
    \ell(m/J)\asymp\ell(u_n/J)\asymp\log n.
\end{equation*}
It follows that
\begin{equation*}
    \begin{split}
        B_0(h_{m})
        &\asymp
        m^{-2b}J^{-2(a-b)}\ell_+(m,J)\\
        &\asymp n_2^{-1}\ell(u_n/J)
        \asymp V_0(h_{m}).
    \end{split}
\end{equation*}
Therefore,
\begin{equation*}
    \overline\sR_1^*\asymp n^{-\gamma_2}\log n,
    \qquad \Gamma=\gamma_2.
\end{equation*}

If $\delta>1$, define
\begin{equation*}
    u_n=
    \left(n_2J^{\delta-1-2(a-b)}\right)^{1/(2b)},
    \qquad
    m=u_n\ell_+(u_n,J)^{1/(2b)}.
\end{equation*}
Here
\begin{equation*}
    \frac{u_n}{J}
    =
    \left(n_2J^{-(1+2a-\delta)}\right)^{1/(2b)}
    \longrightarrow\infty
\end{equation*}
at a polynomial rate. Thus $m\geq J$ and
$\ell_+(m,J)\asymp\ell_+(u_n,J)$. We obtain
\begin{equation*}
    \begin{split}
        B_0(h_{m})
        &\asymp
        m^{-2b}J^{-2(a-b)}\ell_+(m,J)\\
        &\asymp
        u_n^{-2b}J^{-2(a-b)}
        =n_2^{-1}J^{1-\delta}
        \asymp V_0(h_{m}).
    \end{split}
\end{equation*}
Hence
\begin{equation*}
    \overline\sR_1^*\asymp n_2^{-1}J^{1-\delta}
    \asymp
    \begin{cases}
        n^{-(\delta-1+\gamma_2)/\delta},
        &\gamma_{\rm c}^{+}<\gamma_2<1,\\
        n^{-\gamma_2},
        &\gamma_2\geq1.
    \end{cases}
\end{equation*}

\paragraph{Optimal scaling when $\delta=0$.}
In this case the two datasets have the same spectral decay. Let
\begin{equation*}
    u_n=(n+n_2)^{1/(1+2a)},
    \qquad
    m=u_n\ell_1(u_n)^{1/(1+2a)}.
\end{equation*}
Using \eqref{eq:optimal-ridge-equal-spectra}, we obtain
\begin{equation*}
    \begin{split}
        B_0(h_{m})
        &\asymp m^{-2a}\ell_1(m)\\
        &\asymp
        (n+n_2)^{-2a/(1+2a)}
        \ell_1(u_n)^{1/(1+2a)}\\
        &\asymp \frac{m}{n+n_2}
        \asymp V_0(h_{m}).
    \end{split}
\end{equation*}
Therefore the polynomial exponent is
\begin{equation*}
    \Gamma=\frac{2a\max\{1,\gamma_2\}}{1+2a}.
\end{equation*}

\paragraph{Optimal scaling when $\delta<0$.}
Define
\begin{equation*}
    \gamma_{\rm c}^{-}
    =1-\frac{\delta}{1+2b}>1.
\end{equation*}
First suppose $\gamma_2\leq\gamma_{\rm c}^{-}$. Let
\begin{equation*}
    u_n=
    \left(nJ^{-2(b-a)}\right)^{1/(1+2a)},
    \qquad
    m=u_n\ell_-(u_n,J)^{1/(1+2a)}.
\end{equation*}
We have $u_n\geq J$ for all sufficiently large $n$. Indeed, when $\gamma_2\leq1$, $J=1$, whereas for
$1<\gamma_2\leq\gamma_{\rm c}^{-}$,
\begin{equation*}
    \frac{u_n}{J}
    =
    \left(nJ^{-(1+2b)}\right)^{1/(1+2a)}
    \geq1.
\end{equation*}
The last inequality follows from
$J\leq n^{(\gamma_2-1)/(-\delta)}$ and the definition of
$\gamma_{\rm c}^{-}$. Moreover, $u_n\to\infty$ at a polynomial rate, including at $\gamma_2=\gamma_{\rm c}^{-}$.

Since $m\geq u_n\geq J$, we can use the second cases of
\eqref{eq:optimal-ridge-negative-bias} and
\eqref{eq:optimal-ridge-negative-variance}. The definition of
$\ell_-$ gives
$\ell_-(m,J)\asymp\ell_-(u_n,J)$, including when $u_n\asymp J$. Thus
\begin{equation*}
    \begin{split}
        B_0(h_{m})
        &\asymp
        m^{-2a}J^{-2(b-a)}\ell_-(m,J)\\
        &\asymp
        \frac{u_n}{n}\ell_-(u_n,J)^{1/(1+2a)}\\
        &=\frac{m}{n}
        \asymp V_0(h_{m}).
    \end{split}
\end{equation*}
Using
\begin{equation*}
    \frac{u_n}{n}
    =
    n^{-2a/(1+2a)}
    J^{-2(b-a)/(1+2a)},
    \qquad
    J\asymp n^{(\gamma_2-1)_+/(-\delta)},
\end{equation*}
we obtain
\begin{equation*}
    \Gamma=
    \frac{
        2a+\frac{2(b-a)}{-\delta}(\gamma_2-1)_+
    }{1+2a},
    \qquad \gamma_2\leq\gamma_{\rm c}^{-}.
\end{equation*}

Finally, suppose $\gamma_2>\gamma_{\rm c}^{-}$. Define
\begin{equation*}
    u_n=n_2^{1/(1+2b-\delta)},
    \qquad
    m=u_n\ell_2(u_n)^{1/(1+2b-\delta)}.
\end{equation*}
The strict inequality $\gamma_2>\gamma_{\rm c}^{-}$ implies that
$u_n/J\to0$ at a polynomial rate. Therefore $m\leq J$ for all sufficiently large $n$ and
$\ell_2(m)\asymp\ell_2(u_n)$. Since
$u_n^{1+2b-\delta}=n_2$, we have
\begin{equation*}
    \begin{split}
        B_0(h_{m})
        &\asymp m^{-2b}\ell_2(m)\\
        &\asymp
        u_n^{-2b}
        \ell_2(u_n)^{(1-\delta)/(1+2b-\delta)}\\
        &\asymp n_2^{-1}m^{1-\delta}
        \asymp V_0(h_{m}).
    \end{split}
\end{equation*}
Consequently,
\begin{equation*}
    \Gamma=\frac{2b\gamma_2}{1+2b-\delta},
    \qquad \gamma_2>\gamma_{\rm c}^{-}.
\end{equation*}

Every cutoff constructed above satisfies $m\to\infty$ and
$B_0(h_{m})\asymp V_0(h_{m})$. The feasibility argument therefore gives
$\epsilon_n=\epsilon(m)=o(1)$ in every regime, including the threshold cases. Applying
\eqref{eq:optimal-ridge-oracle-comparison} with
$\lambda_n=h_{m}$ proves the stated polynomial exponents. More precisely, all the displayed logarithmic factors lie between a positive constant and a constant times $1+\log n$, so
\begin{equation*}
    n^{-\Gamma}
    \lesssim \overline\sR_1^*
    \lesssim n^{-\Gamma}(1+\log n).
\end{equation*}
This completes the proof.

\subsection{A technical lemma}
\begin{lemma}[Power-law sums]
\label{lem:power-law-sums}
For every fixed $u\in\mathbb R$ and all $m\geq1$,
\begin{equation*}
    \sum_{j\leq m}j^{u-1}\asymp
    \begin{cases}
        m^u,&u>0,\\
        1+\log m,&u=0,\\
        1,&u<0.
    \end{cases}
\end{equation*}
Moreover, for $u<0$,
\begin{equation*}
    \sum_{j\geq m}j^{u-1}\asymp m^u,
\end{equation*}
whereas the tail sum diverges for $u\geq0$.

For integers $1\leq m_-<m_+$,
\begin{equation*}
    \sum_{m_-<j\leq m_+}j^{u-1}\asymp
    \begin{cases}
        m_+^u-m_-^u,&u>0,\\
        \log(m_+/m_-),&u=0,\\
        m_-^u-m_+^u,&u<0.
    \end{cases}
\end{equation*}
For real endpoints $1\leq m_-<m_+$, the same interval estimate holds with $m_\pm$ on the right-hand side replaced by $\lfloor m_\pm\rfloor$. The sum is zero if these two integers coincide.

All comparison constants may depend on $u$ but are independent of the summation endpoints.
\end{lemma}

\begin{proof}
We first consider integer endpoints. For every $j\geq2$ and
$x\in[j-1,j]$, we have $j/2\leq x\leq j$. Consequently,
\begin{equation*}
    2^{-|u-1|}j^{u-1}
    \leq \int_{j-1}^{j}x^{u-1}\,\dd x
    \leq 2^{|u-1|}j^{u-1}.
\end{equation*}
Summing these inequalities gives
\begin{equation*}
    \sum_{j\leq m}j^{u-1}
    \asymp 1+\int_1^m x^{u-1}\,\dd x.
\end{equation*}
If $u\ne0$, the integral equals $(m^u-1)/u$, whereas for $u=0$ it equals $\log m$. Hence
\begin{equation*}
    1+\int_1^m x^{u-1}\,\dd x
    \asymp
    \begin{cases}
        m^u,&u>0,\\
        1+\log m,&u=0,\\
        1,&u<0,
    \end{cases}
\end{equation*}
which proves the prefix-sum estimate.

Similarly, since $j\leq x\leq2j$ for $x\in[j,j+1]$ and $j\geq1$,
\begin{equation*}
    j^{u-1}\asymp\int_j^{j+1}x^{u-1}\,\dd x.
\end{equation*}
For $u<0$, summing over $j\geq m$ yields
\begin{equation*}
    \sum_{j\geq m}j^{u-1}
    \asymp\int_m^\infty x^{u-1}\,\dd x
    =\frac{m^u}{-u}
    \asymp m^u.
\end{equation*}
For $u\geq0$, the integral diverges, and the same comparison proves that the tail sum diverges.

For the interval sum, we obtain
\begin{equation*}
    \sum_{m_-<j\leq m_+}j^{u-1}
    \asymp\int_{m_-}^{m_+}x^{u-1}\,\dd x
    =
    \begin{cases}
        \dfrac{m_+^u-m_-^u}{u},&u\ne0,\\[6pt]
        \log(m_+/m_-),&u=0.
    \end{cases}
\end{equation*}
Separating the cases $u>0$ and $u<0$, and absorbing the fixed factor $1/|u|$ into the comparison constants, proves the stated estimate.

Finally, for real endpoints,
\begin{equation*}
    \sum_{j\leq m}j^{u-1}
    =\sum_{j\leq\lfloor m\rfloor}j^{u-1},
    \qquad
    \sum_{j\geq m}j^{u-1}
    =\sum_{j\geq\lceil m\rceil}j^{u-1}.
\end{equation*}
Since $\lfloor m\rfloor\asymp m\asymp\lceil m\rceil$ for $m\geq1$, the prefix and tail estimates remain valid. For the interval sum, the exact identity
\begin{equation*}
    \sum_{m_-<j\leq m_+}j^{u-1}
    =
    \sum_{\lfloor m_-\rfloor<j\leq\lfloor m_+\rfloor}j^{u-1}
\end{equation*}
gives the asserted extension.
\end{proof}

{

\section{Proof of Lemma \ref{lem:vanishing-approx}}
\label{apx:pf-lem:vanishing-approx}

Our goal is to prove that the approximation error is small relative to $\overline\sR_1(\lambda_n)$. We compare both $\sR(\widehat\btheta)$ and $\overline\sR_1(\lambda_n)$ with the same population ridge risk. Specifically, define
\begin{equation*}
\bG_0=(n\bC_1+n_2\bC_2+\lambda_n\bI)^{-1},\qquad B_0=\lambda_n^2\btheta_*^\top\bG_0\bC_1\bG_0\btheta_*,\qquad V_\sigma=\sum_{i=1}^2\sigma_{\eps_i}^2n_i\Tr(\bC_1\bG_0\bC_i\bG_0).
\end{equation*}
Here $B_0$ is defined in \eqref{eq:optimal-ridge-population-sums}, and $V_\sigma\asymp V_0$ in \eqref{eq:optimal-ridge-population-sums} because the noise variances are bounded above and below. Recall that the proof of Theorem~\ref{thm:scaling-optimal-ridge} constructs a sequence $\lambda_n=h_{m}$ such that 
\begin{equation*}
\epsilon_n\lesssim {-c_-}, \quad m \gtrsim n^{1/(1+2s)}.
\end{equation*}
for some fixed $c_->0$. Throughout the proof, we suppress the subscript $n$ when no confusion can arise.

To bound the approximation error, the triangle inequality gives
\begin{equation*}
|\sR(\widehat\btheta)-\overline\sR_1(\lambda_n)|\leq|\sR(\widehat\btheta)-(B_0+V_\sigma)|+|\overline\sR_1(\lambda_n)-(B_0+V_\sigma)|.
\end{equation*}
Accordingly, we will prove the following two estimates:
\begin{align}
|\sR(\widehat\btheta)-(B_0+V_\sigma)|
&\leq C_D(\log N)^c\left(\sqrt{\epsilon}+\sqrt{\lambda_n\epsilon}+\lambda_n\epsilon\right)(B_0+V_\sigma),
\label{eq:powerlaw-empirical-comparison-goal}\\
|\overline\sR_1(\lambda_n)-(B_0+V_\sigma)|
&\leq C(\epsilon+\lambda_n\epsilon)(B_0+V_\sigma).
\label{eq:powerlaw-deterministic-comparison-goal}
\end{align}
The first estimate holds with probability at least $1-N^{-D}$, whereas the second is deterministic.

We first explain why these two estimates imply the desired conclusion. 

We note that $\lambda_n\epsilon\asymp m^{1-\alpha_{\min}}$. Since the same domain maximizes $j^{-\alpha_i}$ for every $j$,
\begin{equation*}
\epsilon=\sum_{j\geq1}\frac{j^{-\alpha_{\min}}}{h_j+\lambda_n}.
\end{equation*}
For $j\leq m$, the inequality $h_j\geq\lambda_n(m/j)^{\alpha_{\min}}$ gives
\begin{equation*}
\sum_{j\leq m}\frac{j^{-\alpha_{\min}}}{h_j+\lambda_n}\leq\frac{m^{1-\alpha_{\min}}}{\lambda_n}.
\end{equation*}
For $j>m$, we have $\lambda_n\leq h_j+\lambda_n\leq2\lambda_n$, so
\begin{equation*}
\sum_{j>m}\frac{j^{-\alpha_{\min}}}{h_j+\lambda_n}\asymp\frac1{\lambda_n}\sum_{j>m}j^{-\alpha_{\min}}\asymp\frac{m^{1-\alpha_{\min}}}{\lambda_n}.
\end{equation*}

Dividing the triangle inequality by $\overline\sR_1(\lambda_n)$ and applying \eqref{eq:powerlaw-empirical-comparison-goal}--\eqref{eq:powerlaw-deterministic-comparison-goal}, we obtain
\begin{align*}
\frac{|\sR(\widehat\btheta)-\overline\sR_1(\lambda_n)|}{\overline\sR_1(\lambda_n)}
&\leq C_D(\log N)^c\left(\sqrt{\epsilon}+\sqrt{\lambda_n\epsilon}+\lambda_n\epsilon\right)+C(\epsilon+\lambda_n\epsilon)\\
&\leq C_D(\log N)^c\left(\sqrt{\epsilon}+m^{-(\alpha_{\min}-1)/2}\right),
\end{align*}
where in the last step, we use $\eps \leq \sqrt{\eps}, \lambda_n \eps \leq \sqrt{\lambda_n \eps,}$ since $\alpha_{\min}>1$, $m\geq n^{c_-}$, and $\epsilon\leq n^{-c_-}$, so  $\epsilon, \lambda_n\epsilon = o(1)$. And this finishes the proof.

It remains to establish the two comparison estimates. For \eqref{eq:powerlaw-empirical-comparison-goal}, we control the normalized empirical covariance and then bound the resulting perturbations of the population bias and variance. For \eqref{eq:powerlaw-deterministic-comparison-goal}, we use the fixed-point equations to compare $\mu_i$ with $n_i$, and then compare the deterministic bias and variance with $B_0$ and $V_\sigma$.

\paragraph{Proof of \eqref{eq:powerlaw-empirical-comparison-goal}.}
We start from
\begin{equation*}
|\sR(\widehat\btheta)-(B_0+V_\sigma)|\leq|\sB_1-B_0|+|\sV_1-V_\sigma|.
\end{equation*}
Define the normalized covariance errors
\begin{align*}
\bE&=\bG_0^{1/2}\left(\sum_{i=1}^2(\bX_i^\top\bX_i-n_i\bC_i)\right)\bG_0^{1/2},\\
\bE_\sigma&=\bG_0^{1/2}\left(\sum_{i=1}^2\sigma_{\eps_i}^2(\bX_i^\top\bX_i-n_i\bC_i)\right)\bG_0^{1/2}.
\end{align*}
We first derive deterministic bounds for the bias and variance differences under $\|\bE\|_{\op}\leq1/2$. We then establish, independently of this condition, a high-probability bound on $\|\bE\|_{\op}$ and $\|\bE_\sigma\|_{\op}$.

\emph{Bounding $\sB_1 - B_0$.}
The definition of $\bE$ gives
\begin{equation*}
\bG^{-1}=\bG_0^{-1/2}(\bI+\bE)\bG_0^{-1/2},\qquad \bG=\bG_0^{1/2}(\bI+\bE)^{-1}\bG_0^{1/2}.
\end{equation*}
When $\|\bE\|_{\op}\leq1/2$, we have $\bI+\bE\succeq\bI/2$, and hence
\begin{equation*}
\|(\bI+\bE)^{-1}\|_{\op}\leq2,\qquad \|(\bI+\bE)^{-1}-\bI\|_{\op}=\|(\bI+\bE)^{-1}\bE\|_{\op}\leq2\|\bE\|_{\op}.
\end{equation*}
Here we used $(\bI+\bE)^{-1}-\bI=-(\bI+\bE)^{-1}\bE$.

Since
\begin{equation*}
\sB_1=\|\lambda_n\bC_1^{1/2}\bG\btheta_*\|^2,\qquad B_0=\|\lambda_n\bC_1^{1/2}\bG_0\btheta_*\|^2,
\end{equation*}
the  triangle inequality implies
\begin{align*}
|\sqrt{\sB_1}-\sqrt{B_0}|
&\leq\lambda_n\|\bC_1^{1/2}(\bG-\bG_0)\btheta_*\|\\
&=\lambda_n\|\bC_1^{1/2}\bG_0^{1/2}((\bI+\bE)^{-1}-\bI)\bG_0^{1/2}\btheta_*\|\\
&\leq2\lambda_n\|\bE\|_{\op}\|\bC_1^{1/2}\bG_0^{1/2}\|_{\op}\|\bG_0^{1/2}\btheta_*\|.
\end{align*}
The remaining factors satisfy
\begin{equation*}
\|\bC_1^{1/2}\bG_0^{1/2}\|_{\op}^2\leq\Tr(\bC_1\bG_0),\qquad \|\bG_0^{1/2}\btheta_*\|^2\leq\lambda_n^{-1}\|\btheta_*\|^2.
\end{equation*}
Since $\beta>1/2$, $\|\btheta_*\|^2=\sum_{j\geq1}j^{-2\beta}<\infty$. Therefore,
\begin{equation*}
|\sqrt{\sB_1}-\sqrt{B_0}|^2\leq C\lambda_n\|\bE\|_{\op}^2\Tr(\bC_1\bG_0).
\end{equation*}
Expanding the difference of squares yields
\begin{align*}
|\sB_1-B_0|
&\leq2\sqrt{B_0}\,|\sqrt{\sB_1}-\sqrt{B_0}|+|\sqrt{\sB_1}-\sqrt{B_0}|^2\\
&\leq C\|\bE\|_{\op}\sqrt{\lambda_nB_0\Tr(\bC_1\bG_0)}
+C\lambda_n\|\bE\|_{\op}^2\Tr(\bC_1\bG_0).
\end{align*}

\emph{Bounding $\sV_1 - V_\sigma$.}
By the definitions of $\sV_1$ and $V_\sigma$, 
\begin{align*}
\sV_1-V_\sigma
&=\Tr\left[\bC_1\bG\left(\sum_{i=1}^2\sigma_{\eps_i}^2(\bX_i^\top\bX_i-n_i\bC_i)\right)\bG\right]\\
&\quad+\sum_{i=1}^2\sigma_{\eps_i}^2n_i\Tr\left[\bC_1(\bG-\bG_0)\bC_i\bG\right]\\
&\quad+\sum_{i=1}^2\sigma_{\eps_i}^2n_i\Tr\left[\bC_1\bG_0\bC_i(\bG-\bG_0)\right].
\end{align*}
For the first term, by the definition of $\bE_\sigma,$
\begin{align*}
&\Tr\left[\bC_1\bG\left(\sum_{i=1}^2\sigma_{\eps_i}^2(\bX_i^\top\bX_i-n_i\bC_i)\right)\bG\right]\\
&\qquad=\Tr\Bigl[\bC_1\bigl(\bG_0^{1/2}(\bI+\bE)^{-1}\bG_0^{1/2}\bigr)
\bigl(\bG_0^{-1/2}\bE_\sigma\bG_0^{-1/2}\bigr) \bigl(\bG_0^{1/2}(\bI+\bE)^{-1}\bG_0^{1/2}\bigr)\Bigr]\\
&\qquad=\Tr\left[\bC_1\bG_0(\bI+\bE)^{-1}\bE_\sigma(\bI+\bE)^{-1}\right]\\
&\qquad\leq  4\|\bE_\sigma\|_{\op}\Tr(\bC_1\bG_0).
\end{align*}

For the other two terms, observe that
\begin{equation*}
0\preceq\bG_0^{1/2}\left(\sum_{i=1}^2\sigma_{\eps_i}^2n_i\bC_i\right)\bG_0^{1/2}
\preceq\left(\max_i\sigma_{\eps_i}^2\right)(\bI-\lambda_n\bG_0)
\preceq\left(\max_i\sigma_{\eps_i}^2\right)\bI.
\end{equation*}
Substituting
\begin{equation*}
\bG-\bG_0=\bG_0^{1/2}\left((\bI+\bE)^{-1}-\bI\right)\bG_0^{1/2}
\end{equation*}
into the second term, and using cyclicity of the trace, gives
\begin{align*}
&\left|\sum_{i=1}^2\sigma_{\eps_i}^2n_i\Tr\left[\bC_1(\bG-\bG_0)\bC_i\bG\right]\right|\\
&\quad\leq\Tr(\bC_1\bG_0)\|(\bI+\bE)^{-1}-\bI\|_{\op}
\left\|\bG_0^{1/2}\left(\sum_{i=1}^2\sigma_{\eps_i}^2n_i\bC_i\right)\bG_0^{1/2}\right\|_{\op}
\|(\bI+\bE)^{-1}\|_{\op}\\
&\quad\leq4\left(\max_i\sigma_{\eps_i}^2\right)\|\bE\|_{\op}\Tr(\bC_1\bG_0).
\end{align*}
The third term is bounded in the same way, with no final inverse factor:
\begin{equation*}
\left|\sum_{i=1}^2\sigma_{\eps_i}^2n_i\Tr\left[\bC_1\bG_0\bC_i(\bG-\bG_0)\right]\right|
\leq2\left(\max_i\sigma_{\eps_i}^2\right)\|\bE\|_{\op}\Tr(\bC_1\bG_0).
\end{equation*}
Combining the three terms and using the bounded noise variances, we obtain
\begin{equation*}
|\sV_1-V_\sigma|\leq C\left(\|\bE\|_{\op}+\|\bE_\sigma\|_{\op}\right)\Tr(\bC_1\bG_0).
\end{equation*}

\emph{Bounding the common trace factor.}
We next show that $\Tr(\bC_1\bG_0)\lesssim V_0$. Recall
\begin{equation*}
\Tr(\bC_1\bG_0)=\sum_{j\geq1}\frac{j^{-\alpha_1}}{h_j+\lambda_n},\qquad V_0=\sum_{j\geq1}\frac{j^{-\alpha_1}h_j}{(h_j+\lambda_n)^2}.
\end{equation*}
For $j\leq m$, we have $h_j\geq\lambda_n$, so
\begin{equation*}
\sum_{j\leq m}\frac{j^{-\alpha_1}}{h_j+\lambda_n}\leq2\sum_{j\leq m}\frac{j^{-\alpha_1}h_j}{(h_j+\lambda_n)^2}\leq2V_0.
\end{equation*}
For the tail,
\begin{equation*}
\sum_{j>m}\frac{j^{-\alpha_1}}{h_j+\lambda_n}\leq\frac1{\lambda_n}\sum_{j>m}j^{-\alpha_1}\lesssim\frac{m^{1-\alpha_1}}{\lambda_n}.
\end{equation*}
On the block $m/2\leq j\leq m$, we have $h_j\asymp\lambda_n$ and $j^{-\alpha_1}\asymp m^{-\alpha_1}$. Hence
\begin{equation*}
V_0\geq\sum_{m/2\leq j\leq m}\frac{j^{-\alpha_1}h_j}{(h_j+\lambda_n)^2}\gtrsim\frac{m^{1-\alpha_1}}{\lambda_n}.
\end{equation*}
Combining the head and tail estimates proves $\Tr(\bC_1\bG_0)\lesssim V_0$.

Substituting this comparison into the preceding bias and variance bounds gives
\begin{align*}
|\sB_1-B_0|
&\leq C\|\bE\|_{\op}\sqrt{\lambda_nB_0V_0}
+C\lambda_n\|\bE\|_{\op}^2V_0\\
&\leq C\left(\sqrt{\lambda_n}\|\bE\|_{\op}+\lambda_n\|\bE\|_{\op}^2\right)(B_0+V_0),\\
|\sV_1-V_\sigma|
&\leq C\left(\|\bE\|_{\op}+\|\bE_\sigma\|_{\op}\right)V_0.
\end{align*}
Here we used $2\sqrt{B_0V_0}\leq B_0+V_0$. Since $V_\sigma\asymp V_0$, we conclude that, whenever $\|\bE\|_{\op}\leq1/2$,
\begin{equation*}
\frac{|\sR(\widehat\btheta)-(B_0+V_\sigma)|}{B_0+V_\sigma}
\leq C\left(\|\bE\|_{\op}+\|\bE_\sigma\|_{\op}+\sqrt{\lambda_n}\|\bE\|_{\op}+\lambda_n\|\bE\|_{\op}^2\right).
\end{equation*}

\emph{Bounding the normalized covariance errors.}
It remains to prove, directly from Assumption~\ref{asm:distr}, that
\begin{equation*}
\max\{\|\bE\|_{\op},\|\bE_\sigma\|_{\op}\}\leq C_D(\log N)^c\sqrt{\epsilon}
\end{equation*}
with probability at least $1-N^{-D}$. We first  define
\begin{equation*}
\bM_{i,r}=\bG_0^{1/2}\bx_{i,r}\bx_{i,r}^\top\bG_0^{1/2},\qquad
\bZ_{i,r}=\bM_{i,r}\mathbf1_{\{\bx_{i,r}^\top\bG_0\bx_{i,r}\leq C_A\epsilon(\log N)^{c_1}\}}.
\end{equation*} and  decompose
\begin{equation*}
\bE=\sum_{i,r}(\bZ_{i,r}-\E\bZ_{i,r})
+\sum_{i,r}(\bM_{i,r}-\bZ_{i,r})
-\sum_{i,r}\E[\bM_{i,r}-\bZ_{i,r}].
\end{equation*}

We control the first term via matrix Bernstein inequality. By the definition of $\epsilon$, $\Tr(\bC_i\bG_0)=\sum_{j\geq1}\frac{j^{-\alpha_i}}{h_j+\lambda_n}\leq\epsilon.$ Since each $\bZ_{i,r}$ depends only on $\bx_{i,r}$, they are still independent with each other. 
By Assumption~\ref{asm:distr}, using
$\|\bC_i^{1/2}\bG_0\bC_i^{1/2}\|_F\leq\Tr(\bC_i\bG_0)\leq\epsilon$ gives 
\begin{equation*}
\P\left(\bx_{i,r}^\top\bG_0\bx_{i,r}>C_A \epsilon(\log N)^{c_1}\right)\leq N^{-A},
\end{equation*}
for any $A > 0$ with some constants $C_A.$
Thus, we have$\|\bM_{i,r}\|_{\op}=\bx_{i,r}^\top\bG_0\bx_{i,r},
$ meaning that with probability $1 - N^{-A},$
\begin{equation*}
\|\bZ_{i,r}-\E\bZ_{i,r}\|_{\op}\leq C_A\epsilon(\log N)^{c_1}.
\end{equation*}

For the second moment, note that 
\begin{equation*}
\bM_{i,r}^2=(\bx_{i,r}^\top\bG_0\bx_{i,r})\bM_{i,r}.
\end{equation*}
Consequently,
\begin{align*}
\bZ_{i,r}^2
&=(\bx_{i,r}^\top\bG_0\bx_{i,r})\bM_{i,r}
\mathbf1_{\{\bx_{i,r}^\top\bG_0\bx_{i,r}\leq C_A\epsilon(\log N)^{c_1}\}}\\
&\preceq C_A\epsilon(\log N)^{c_1}\bM_{i,r}.
\end{align*}
Using $\E\bM_{i,r}=\bG_0^{1/2}\bC_i\bG_0^{1/2}$, we obtain
\begin{align*}
\sum_{i,r}\E[(\bZ_{i,r}-\E\bZ_{i,r})^2] &\preceq\sum_{i,r}\E[\bZ_{i,r}^2] \preceq C_A\epsilon(\log N)^{c_1}
\bG_0^{1/2}\left(\sum_{i=1}^2n_i\bC_i\right)\bG_0^{1/2}.
\end{align*}
We show that
\begin{equation*}
\left\|\bG_0^{1/2}\left(\sum_i n_i\bC_i\right)\bG_0^{1/2}\right\|_{\op}\in[1/2,1],
\qquad
\Tr\left[\bG_0^{1/2}\left(\sum_i n_i\bC_i\right)\bG_0^{1/2}\right]\asymp m.
\end{equation*}

To see this, by definition,
\begin{equation*}
\bG_0^{1/2}\left(\sum_{i=1}^2n_i\bC_i\right)\bG_0^{1/2}
=\bI-\lambda_n\bG_0
=\operatorname{Diag}\left(\frac{h_j}{h_j+\lambda_n}\right)_{j\geq1}.
\end{equation*}
The operator norm of the above lies in $[1/2,1]$ as each of  its diagonal entries is at most one, and the first one is at least $1/2$ because $h_1\geq h_m=\lambda_n$, and thus its trace is of order $m$. Indeed, the first $\lfloor m\rfloor$ entries are at least $1/2$, while
\begin{equation*}
\sum_{j>m}\frac{h_j}{h_j+\lambda_n}\leq\sum_{j>m}(m/j)^{\alpha_{\min}}\lesssim m,
\end{equation*}
where we used $h_j\leq\lambda_n(m/j)^{\alpha_{\min}}$ for $j>m$ and $\alpha_{\min}>1$.

Thus, $\bG_0^{1/2}\left(\sum_i n_i\bC_i\right)\bG_0^{1/2}$ has intrinsic dimension $O(m)$, and by applying the intrinsic-dimension matrix Bernstein inequality \citep[Theorem~7.3.1]{tropp2015introduction} and recalling $\log m=O(\log N)$,
\begin{equation*}
\left\|\sum_{i,r}(\bZ_{i,r}-\E\bZ_{i,r})\right\|_{\op}
\leq C_D\left(\sqrt{\epsilon}(\log N)^{(c_1+1)/2}+\epsilon(\log N)^{c_1+1}\right),
\end{equation*}
with probability at least $1-N^{-(D+2)}$.

Next,  we control the remaining two terms. A union bound gives
\begin{equation*}
\P\left(\bM_{i,r}\neq\bZ_{i,r}\text{ for some }i,r\right)\leq N^{1-A}.
\end{equation*}
On the complementary event, the decomposition reduces to
\begin{equation*}
\bE=\sum_{i,r}(\bZ_{i,r}-\E\bZ_{i,r})
-\sum_{i,r}\E[\bM_{i,r}-\bZ_{i,r}].
\end{equation*}

To bound the remaining deterministic term, fix a deterministic vector $\bv$ and write $q=\bv^\top\bG_0^{1/2}\bC_i\bG_0^{1/2}\bv$. Assumption~\ref{asm:distr} gives
\begin{equation*}
\P\left(\left|(\bv^\top\bG_0^{1/2}\bx_{i,r})^2-q\right|\geq tq\right)\leq C\exp(-ct^{1/\beta}),\qquad t>0,
\end{equation*}
where $\beta>0$ is the fixed tail exponent supplied by the assumption. Since $\E[(\bv^\top\bG_0^{1/2}\bx_{i,r})^2]=q$, integrating this tail bound yields
\begin{align*}
\E[(\bv^\top\bG_0^{1/2}\bx_{i,r})^4]
&=q^2+\E\left[\left((\bv^\top\bG_0^{1/2}\bx_{i,r})^2-q\right)^2\right]\\
&\leq q^2+2Cq^2\int_0^\infty t\exp(-ct^\eta)\,dt\\
&\leq C'(\bv^\top\bG_0^{1/2}\bC_i\bG_0^{1/2}\bv)^2.
\end{align*}
 Cauchy--Schwarz inequality gives
\begin{align*}
\bv^\top\E[\bM_{i,r}-\bZ_{i,r}]\bv
&=\E\left[(\bv^\top\bG_0^{1/2}\bx_{i,r})^2
\mathbf1_{\{\bx_{i,r}^\top\bG_0\bx_{i,r}>C_A\epsilon(\log N)^{c_1}\}}\right]\\
&\leq\left(\E[(\bv^\top\bG_0^{1/2}\bx_{i,r})^4]\right)^{1/2}
\P\left(\bx_{i,r}^\top\bG_0\bx_{i,r}>C_A\epsilon(\log N)^{c_1}\right)^{1/2}\\
&\leq CN^{-A/2}\bv^\top\bG_0^{1/2}\bC_i\bG_0^{1/2}\bv.
\end{align*}
Since this holds for every $\bv$, summing over the samples yields
\begin{equation*}
\mathbf0\preceq\sum_{i,r}\E[\bM_{i,r}-\bZ_{i,r}]
\preceq CN^{-A/2}\bG_0^{1/2}\left(\sum_i n_i\bC_i\right)\bG_0^{1/2}
\preceq CN^{-A/2}\bI.
\end{equation*}
Combining this estimate with the Bernstein bound, we conclude that
\begin{equation*}
\|\bE\|_{\op}
\leq C_D\left(\sqrt{\epsilon}(\log N)^{(c_1+1)/2}+\epsilon(\log N)^{c_1+1}\right)+CN^{-A/2},
\end{equation*}
with probability at least $1-N^{-(D+2)}-N^{1-A}$.

Finally, since
\begin{equation*}
N\epsilon\geq\sum_i n_i\Tr(\bC_i\bG_0)
=\Tr\left[\bG_0^{1/2}\left(\sum_i n_i\bC_i\right)\bG_0^{1/2}\right]\geq\frac12,
\end{equation*}
we have $N^{-A/2}=o(\sqrt{\epsilon})$. Also, $\epsilon\leq n^{-c_-}$ implies $\epsilon\leq1$ for sufficiently large $n$, so increasing the fixed logarithmic exponent gives
\begin{equation*}
\|\bE\|_{\op}\leq C_D(\log N)^c\sqrt{\epsilon}.
\end{equation*}
The same argument applies to
\begin{equation*}
\bE_\sigma=\sum_{i,r}\sigma_{\eps_i}^2(\bM_{i,r}-\E\bM_{i,r}),
\end{equation*}
using the truncated matrices $\sigma_{\eps_i}^2\bZ_{i,r}$. Since the noise variances are uniformly bounded, the single-matrix norm bound, the matrix variance bound, and the discarded-mean bound change only by constant factors. Taking a union bound and using $A>D+4$, we obtain
\begin{equation*}
\max\{\|\bE\|_{\op},\|\bE_\sigma\|_{\op}\}
\leq C_D(\log N)^c\sqrt{\epsilon},
\end{equation*}
with probability at least $1-N^{-D}$ for sufficiently large $n$.

\paragraph{Proof of \eqref{eq:powerlaw-deterministic-comparison-goal}.}
For $\bA_*=\btheta_*\btheta_*^\top$, the definition of the deterministic equivalent gives
\begin{align*}
\overline\sR_1(\lambda_n)-(B_0+V_\sigma)
&=\left(\lambda_n^2\tau_{\bA_*}[1]-B_0\right)+\left(\overline\sV_1-V_\sigma\right)\\
&\quad+\lambda_n^2\tau_{\bA_*}^\top\bL^{-1}\tau_{\bC_1}.
\end{align*}
We will show that the first two terms are bounded by $C\epsilon B_0$ and $C\epsilon V_\sigma$, respectively, while the last term is bounded by $C\lambda_n\epsilon V_\sigma$.

For the first term, recall that
\begin{align*}
B_0&=\lambda_n^2\btheta_*^\top\bG_0\bC_1\bG_0\btheta_*, \quad 
\tau_{\bA_*}[1]=\Tr(\bA_*\overline\bG\bC_1\overline\bG)
=\btheta_*^\top\overline\bG\bC_1\overline\bG\btheta_*.
\end{align*}

We first compare $\bG_0$ and $\bG.$  Let $\eta=\min_{i=1,2}\frac{\mu_i}{n_i}\in(0,1].$
Since
\begin{equation*}
\sum_i\mu_i\bC_i+\lambda_n\bI\succeq\eta(n\bC_1+n_2\bC_2+\lambda_n\bI),
\end{equation*}
we have $\overline\bG\preceq\eta^{-1}\bG_0$. Consequently, the fixed-point equations imply
\begin{equation*}
\frac{\mu_i}{n_i}=\frac1{1+\Tr(\bC_i\overline\bG)}\geq\frac1{1+\epsilon/\eta}.
\end{equation*}
Taking the minimum over $i$ gives $\eta\geq\eta/(\eta+\epsilon)$. Since $\eta>0$, this implies $\eta+\epsilon\geq1$, and therefore
\begin{equation*}
1-\epsilon\leq\frac{\mu_i}{n_i}\leq1,
\end{equation*}
which implies
\begin{equation*}
\bG_0\preceq\overline\bG\preceq(1-\epsilon)^{-1}\bG_0,\qquad \Tr(\bC_i\overline\bG)\leq\frac{\epsilon}{1-\epsilon}.
\end{equation*}

Since $\bC_1$, $\bG_0$, and $\overline\bG$ are diagonal in the same basis, the comparison $\bG_0\preceq\overline\bG\preceq(1-\epsilon)^{-1}\bG_0$ gives
\begin{equation*}
B_0\leq\lambda_n^2\tau_{\bA_*}[1]\leq(1-\epsilon)^{-2}B_0.
\end{equation*}
For sufficiently large $n$, $\epsilon\leq1/2$, so
\begin{equation*}
0\leq\lambda_n^2\tau_{\bA_*}[1]-B_0
\leq\left((1-\epsilon)^{-2}-1\right)B_0
=\frac{2\epsilon-\epsilon^2}{(1-\epsilon)^2}B_0
\leq8\epsilon B_0.
\end{equation*}

To control the variance term, recall that $\bL=\bD-\bK$, where
\begin{equation*}
\bD_{ii}=\frac{n_i}{\mu_i^2},\qquad \bK_{ij}=\Tr(\bC_i\overline\bG\bC_j\overline\bG).
\end{equation*}
Define
\begin{equation*}
\bsigma=(\sigma_{\eps_1}^2,\sigma_{\eps_2}^2)^\top,\qquad w_i=\frac{(\bL^{-1}\bsigma)_i}{\mu_i}.
\end{equation*}
The equation $\bL(\bL^{-1}\bsigma)=\bsigma$ becomes
\begin{equation*}
w_i=\frac{\mu_i}{n_i}\sigma_{\eps_i}^2+\sum_{j=1}^2\frac{\mu_i\mu_j}{n_i}\bK_{ij}w_j.
\end{equation*}
The coefficients in the sum are nonnegative, and their row sums satisfy
\begin{align*}
\sum_{j=1}^2\frac{\mu_i\mu_j}{n_i}\bK_{ij}
&=\frac{\mu_i}{n_i}\Tr\left[\bC_i\overline\bG\left(\sum_{j=1}^2\mu_j\bC_j\right)\overline\bG\right]\\
&\leq\Tr(\bC_i\overline\bG)\\
&\leq C\epsilon.
\end{align*}
Here we used $\sum_j\mu_j\bC_j\preceq\overline\bG^{-1}$ and $\mu_i/n_i\leq1$. It follows that
\begin{equation*}
\max_i|w_i|\leq\max_i\sigma_{\eps_i}^2+C\epsilon\max_i|w_i|,
\end{equation*}
and hence $\max_i|w_i|=O(1)$. Substituting this bound back into the equation for $w_i$ gives
\begin{equation*}
w_i=\frac{\mu_i}{n_i}\sigma_{\eps_i}^2+O(\epsilon).
\end{equation*}
Multiplying by $\mu_i$, and using $\mu_i/n_i=1+O(\epsilon)$, yields
\begin{equation*}
(\bL^{-1}\bsigma)_i=\frac{\mu_i^2}{n_i}\sigma_{\eps_i}^2+O(\mu_i\epsilon)=n_i\sigma_{\eps_i}^2(1+O(\epsilon)).
\end{equation*}
The last equality uses the positive lower bound on the noise variances.

Since $\bG_0, \bG$ are both diagonal, we have
\begin{equation*}
\tau_{\bC_1}[i]=(1+O(\epsilon))\Tr(\bC_1\bG_0\bC_i\bG_0).
\end{equation*}
Consequently,
\begin{align*}
\overline\sV_1
&=\tau_{\bC_1}^\top\bL^{-1}\bsigma\\
&=(1+O(\epsilon))\sum_{i=1}^2\sigma_{\eps_i}^2n_i\Tr(\bC_1\bG_0\bC_i\bG_0)\\
&=(1+O(\epsilon))V_\sigma.
\end{align*}
Thus $|\overline\sV_1-V_\sigma|\leq C\epsilon V_\sigma$.

It remains to bound the correction $\lambda_n^2\tau_{\bA_*}^\top\bL^{-1}\tau_{\bC_1}$. The rescaled linear system for $w_i$ has nonnegative coefficients and row sums smaller than one. Its Neumann series therefore has nonnegative entries. Applying the same argument to any nonnegative right-hand side and undoing the positive diagonal rescaling shows that $\bL^{-1}$ is entrywise nonnegative.

For each $i$, we have
\begin{align*}
\lambda_n^2\tau_{\bA_*}[i]
&=\lambda_n^2\btheta_*^\top\overline\bG\bC_i\overline\bG\btheta_*\\
&\leq\lambda_n^2\|\overline\bG\|_{\op}\Tr(\bC_i\overline\bG)\|\btheta_*\|^2\\
&\leq C\lambda_n\epsilon,
\end{align*}
where we used $\|\overline\bG\|_{\op}\leq\lambda_n^{-1}$ and $\|\btheta_*\|^2<\infty$. Since the noise variances are bounded below, this means that $\lambda_n^2\tau_{\bA_*}\leq C\lambda_n\epsilon\,\bsigma$ entrywise. By symmetry and entrywise nonnegativity of $\bL^{-1}$,
\begin{equation*}
0\leq\lambda_n^2\tau_{\bA_*}^\top\bL^{-1}\tau_{\bC_1}\leq C\lambda_n\epsilon\,\bsigma^\top\bL^{-1}\tau_{\bC_1}=C\lambda_n\epsilon\,\overline\sV_1\leq C\lambda_n\epsilon V_\sigma.
\end{equation*}
Combining the three terms in the deterministic risk decomposition proves
\begin{equation*}
|\overline\sR_1(\lambda_n)-(B_0+V_\sigma)|\leq C\epsilon B_0+C\epsilon V_\sigma+C\lambda_n\epsilon V_\sigma\leq C(\epsilon+\lambda_n\epsilon)(B_0+V_\sigma).
\end{equation*}
This establishes \eqref{eq:powerlaw-deterministic-comparison-goal} and concludes the proof.

}

\section{Proofs of deterministic equivalences}
\label{app:proof-deteq}

\subsection{Heuristic derivation of deterministic equivalences}

First of all, it is easy to heuristicly derive the deterministic equivalence for the first order functional $\Tr(\bA\bG) \approx \Tr(\bA \overline \bG).$ Next we heuristicly derive the deterministic equivalence for the second order functional of resolvent.

\paragraph{Deterministic equivalence of $\Tr(\bA\bG\bB\bG)$.}
Let
\begin{equation*}
\bH=\sum_{i=1}^K \bX_i^\top\bX_i+\lambda\bI.
\end{equation*}
For a deterministic PSD matrix $\bB$,
\begin{equation*}
\Tr(\bA\bG\bB\bG)=-\left.\frac{\dd}{\dd t}\Tr\left(\bA(\bH+t\bB)^{-1}\right)\right|_{t=0}.
\end{equation*} 
The deterministic equivalence of $\Tr(\bA \bG)$ gives
\begin{equation*}
\Tr\left(\bA(\bH+t\bB)^{-1}\right)\approx\Tr(\bA\overline\bG(t)),
\end{equation*}
where
\begin{equation*}
\overline\bG(t)=\left(\sum_{i=1}^K \mu_i(t)\bC_i+\lambda\bI+t\bB\right)^{-1},\qquad \mu_i(t)=\frac{n_i}{1+\Tr(\bC_i\overline\bG(t))}.
\end{equation*}
Differentiating the resolvent gives
\begin{equation*}
\frac{\dd}{\dd t}\overline\bG(t)=-\overline\bG(t)\left(\sum_{j=1}^K \frac{\dd\mu_j(t)}{\dd t}\bC_j+\bB\right)\overline\bG(t).
\end{equation*}
Differentiating the fixed-point equations and evaluating at $t=0$ yields
\begin{equation*}
\frac{n_i}{\mu_i^2}\dot\mu_i-\sum_{j=1}^K\Tr(\bC_i\overline\bG \bC_j\overline\bG)\dot\mu_j=\Tr(\bC_i\overline\bG\bB\overline\bG).
\end{equation*}
Therefore,
\begin{equation*}
\bL\dot{\bmu}=\tau_{\bB}.
\end{equation*}
A final differentiation gives
\begin{align*}
-\left.\frac{\dd}{\dd t}\Tr(\bA\overline\bG(t))\right|_{t=0}
&=\Tr(\bA\overline\bG\bB\overline\bG)+\tau_{\bA}^\top\dot{\bmu}\\
&=\Tr(\bA\overline\bG\bB\overline\bG)+\tau_{\bA}^\top\bL^{-1}\tau_{\bB}.
\end{align*}
This proves the first second-order formula in Theorem~\ref{thm:deteq-risk}.

\paragraph{Deterministic equivalence of $\Tr(\bA\bG\bX_k^\top\bX_k\bG)$.}

Set $\bS_k=\bX_k^\top\bX_k$. Then
\begin{equation*}
\Tr(\bA\bG\bS_k\bG)=-\left.\frac{\dd}{\dd t}\Tr\left(\bA(\bH+t\bS_k)^{-1}\right)\right|_{t=0}.
\end{equation*}
Let $\beta_i(t)=1+t\mathbf 1[i=k]$ and $r_i(t)=\beta_i(t)\mu_i(t)$. The first-order deterministic equivalent is
\begin{equation*}
\overline\bG_k(t)=\left(\sum_{i=1}^K  r_i(t)\bC_i+\lambda\bI\right)^{-1},\qquad \mu_i(t)=\frac{n_i}{1+\beta_i(t)\Tr(\bC_i\overline\bG_k(t))}.
\end{equation*}
Equivalently,
\begin{equation*}
\frac{n_i}{r_i(t)}=\frac{1}{\beta_i(t)}+\Tr(\bC_i\overline\bG_k(t)).
\end{equation*}
Differentiating at $t=0$ gives
\begin{equation*}
\frac{n_i}{\mu_i^2}\dot r_i-\sum_{j=1}^K\Tr(\bC_i\overline\bG \bC_j\overline\bG)\dot r_j=\mathbf 1[i=k],
\end{equation*}
and hence
\begin{equation*}
\bL\dot\br=\be_k.
\end{equation*}
Since
\begin{equation*}
-\left.\frac{\dd}{\dd t}\Tr(\bA\overline\bG_k(t))\right|_{t=0}=\tau_{\bA}^\top\dot\br,
\end{equation*}
we obtain
\begin{equation*}
\Tr(\bA\bG\bS_k\bG)\approx\tau_{\bA}^\top\bL^{-1}\be_k.
\end{equation*}

\paragraph{Bias and variance.}

We apply the first second-order formula with the rank-one matrix $\bA=\btheta_*\btheta_*^\top$ and $\bB=\bC_k$. Multiplication by $\lambda^2$ gives
\begin{equation*}
\overline\sB_k=\lambda^2\left[\left\langle\btheta_*,\overline\bG\bC_k\overline\bG\btheta_*\right\rangle+\tau_{\btheta_*}^\top\bL^{-1}\tau_{\bC_k}\right].
\end{equation*}
Applying the second formula with $\bA=\bC_k$ and summing the independent noise contributions gives
\begin{equation*}
\overline\sV_k=\sum_{i=1}^K\sigma_{\eps_i}^2\,\tau_{\bC_k}^\top\bL^{-1}\be_i.
\end{equation*}
Weighting the domain-specific risks by $\pi_k^*$ gives the main results in  Theorem~\ref{thm:deteq-risk}.

\subsection{Main theorem and proof sketch}

We first present the main theorem and the proof sketch, and then specify the proof of each deterministic equivalence in the following subsections.

Before stating the main results, we first recall some notation below. 

Let $\bX_i = [\bx_{i,1}^\top, \dots, \bx_{i, n_i}^\top]^\top \in \R^{n_i \times d}$ ($i \in [K]$) be $K$ datasets,  where $\E[\bx_{i,r} \bx_{i,r}^\top] = \bC_i.$ Define $\bG = (\sum_{i=1}^K  \bX_i^\top \bX_i + \lambda)^{-1},$ and $\overline \bG = (\sum_{i=1}^K  \mu_i \bC_i + \lambda)^{-1}.$ Here, $\mu_i$ is the fixed point of the self-consistence equation \begin{equation*}
    \begin{split}
        \frac{n_i}{\mu_i} = 1+ \Tr( \bC_i \overline{\bG}), \quad i \in [K]. 
    \end{split}
\end{equation*} 

For later convenience, we define the following matrices that relate to the deterministic equivalence and the stability of the fixed point. Define a $K \times K$ diagonal matrix $\bD$ with $\bD_{ii} = \frac{n_i}{ \mu_i^2},$ a $K \times K$ diagonal matrix $\bK$ with $\bK_{ij} = \Tr(\bC_i \overline{\bG} \bC_j \overline{\bG})$, and $\bL = \bD - \bK$. Now, we state the main theorem that provides the deterministic equivalence of the three functionals appearing in Theorem \ref{thm:deteq-risk}.

\begin{theorem} 
   \label{thm:deteq-func-main} 
   Assume that, for $i \in [K]$, 
   $\bC_i$ satisfies Assumption \ref{asm:distr} and these matrices commute.
   For each $\bC_i$, following \cite[Appendix A.1]{misiakiewicz2024non}, we define  
   \begin{equation}\label{eq:defnu}
       \begin{split}
           r_{\bC_i}(n_i) &= \max\left\{n_i, \max_{0\leq m<\min(n_i,p_i)}\frac{\sum_{j>m}\xi_{i,j}}{\xi_{i,m+1}}\right\}, \\
           \nu^i_\lambda(n_i) &:= 1 + \frac{\xi_{\lfloor\eta n_i\rfloor} r_{\bC_i}(n_i) \sqrt{\log r_{\bC_i}(n_i)}}{\lambda}, \qquad \nu := \max_{i\in[K]}\nu_{\lambda}^i(n_i).
       \end{split}
   \end{equation} 
   We further assume $\nu_\lambda^i(n_i)$ satisfies Assumption \ref{asm:nu} for all $i\in[K]$. For any PSD matrix $\bA,$ define $\mathbf{\tau}_A \in \R^{K}$ with $\mathbf{\tau}_A[k] = \Tr(\bA \overline{\bG} \bC_k \overline{\bG}).$ Then, we have with probability at least $1-\sum_{i=1}^K n_i^{-D}$, 
   \begin{equation*}
       \begin{split}
            |\Tr(\bA \bG) - \Tr(\bA \overline{\bG})| &\leq  C_D \eps_1 \Tr(\bA \overline{\bG}), \\
            |\Tr(\bA \bG \bC_k \bG) - \be_k^\top \bD \bL^{-1} \tau_\bA| &\leq C_D \eps_2 V_\bA, \\
            |\Tr(\bA \bG \bX_k^\top \bX_k \bG) - \be_k^\top  \bL^{-1} \tau_\bA| &\leq C_D \eps_3 V_\bA,
       \end{split}
   \end{equation*}
   where 
   \begin{equation*}
        \begin{aligned}
            \eps_1
            &= e_K\left(\nu^6+\nu^2\log^{\beta+1/2}(N)\right), \\
            \eps_2
            &= e_K\left(\frac{\nu^{14}}{n_k}
            +\nu^8e_K^2\log^{4\beta+3/2} (N)\right), \\
            \eps_3
            &= e_K\left(\nu^{14}+\nu^5\log^{3\beta+3/2}(N)\right), \\
            V_{\bA}
            &= \sum_{i=1}^K n_i \Tr(\bA\overline{\bG}\bC_i\overline{\bG}), \qquad e_K = \left(\sum_{i=1}^K n_i^{-1}\right)^{1/2}.
            \end{aligned}
    \end{equation*}
\end{theorem}

\paragraph{Connections to \cite{misiakiewicz2024non}.} The three functionals studied by \cite[Appendix A]{misiakiewicz2024non} are
\[
\Tr(\bA\bM),\qquad \Tr(\bA\bM^2),\qquad n_k^{-1}\Tr(\bA\bM\bX_k^\top\bX_k\bM).
\]
They enjoy a multiplicative error bound whose coefficient is $O(n^{-1/2})$. In our case, under fixed $K$ and comparable dataset sizes $n_i \asymp n$, all three $\eps_1$, $\eps_2$, and $\eps_3$ also scale as $O(n^{-1/2})$.

To prove Theorem \ref{thm:deteq-func-main}, we decompose the error into the deterministic bias part and the martingale part. For example, for $\Tr(\bA \bG),$ we have: \begin{equation*}
    \begin{split}
         |\Tr(\bA \bG) - \Tr(\bA \overline{\bG})| \leq  \underbrace{|\Tr(\bA \bG) - \E[\Tr(\bA \bG)] |}_{\text{deterministic part}} + \underbrace{|\E[\Tr(\bA \bG)] - \Tr(\bA \overline{\bG})|}_{\text{martingale part}}.
    \end{split}
\end{equation*}

\subsection{Deterministic equivalence of \texorpdfstring{$\Tr(\bA \bG)$}{Tr(A G)}}

\subsubsection{Deterministic part of \texorpdfstring{$\Tr(\bA\bG)$}{Tr(AG)}}
Write
$$
    e_K = \left(\sum_{i=1}^K \frac{1}{n_i}\right)^{1/2}, \qquad \nu = \max_{i} \nu_\lambda^i(n).
$$
Without loss of generality, we may assume $e_K \leq 1$. By the deterministic bounds in Lemma \ref{lem:new1}, we have
$$
    t_i := \Tr(\bC_i\overline\bG) \leq C\nu.
$$
As in \cite[Appendix A.3.1]{misiakiewicz2024non}, we may assume $\bA = \btheta\btheta^\top$. For dataset $\bX_i$, define 
$$
    \bH_i = \bG_{-(i,1)}, \qquad \kappa_i = \E[\Tr(\bC_i\bH_i)], \qquad \overline\bG_- = \left(\sum_{\ell=1}^K \frac{n_\ell}{1+\kappa_\ell}\bC_\ell + \lambda\right)^{-1},
$$
where $\bG_{-(i,1)}$ is obtained from $\bG$ by leave-one-out on the $i$-th data point.
We will bound $|\E[\Tr(\bA\bG)-\Tr(\bA\overline\bG)|$ by splitting it into
\begin{equation}\label{eq:fo1}
    |\E[\Tr(\bA\bG)-\Tr(\bA\overline\bG)| \leq |\Tr(\bA(\E[\bG]-\overline\bG_-))| + |\Tr(\bA(\overline\bG_--\overline\bG))|.
\end{equation}
For the first term in \eqref{eq:fo1}, 
$$
    \E[\bG]-\overline\bG_- = \sum_{i=1}^Kn_i\E\left[\bG\left(\frac{\bC_i}{1+\kappa_i}-\bx\bx^\top\right)\overline\bG_-\right],
$$
with $\bx = \bx_{i,1}$.
Setting $z_i = \bx^\top\bH_i\bx$ and following the approach in \cite[Appendix A.3.1]{misiakiewicz2024non}, we decompose
$$
    \bG\left(\frac{\bC_i}{1+\kappa_i}-\bx\bx^\top\right)\overline\bG_- = \Delta_{i,1}+\Delta_{i,2}+\Delta_{i,3},
$$
with
\begin{equation*}
    \begin{split}
        \Delta_{i,1} &= \frac{\bH_i(\bC_i-\bx\bx^\top)\overline\bG_-}{1+\kappa_i}, \\
        \Delta_{i,2} &= \frac{\bH_i\bx\bx^\top\overline\bG_-(z_i-\kappa_i)}{(1+\kappa_i)(1+z_i)}, \\
        \Delta_{i,3} &= -\frac{\bH_i\bx\bx^\top\bH_i\bC_i\overline\bG_-}{1+z_i}.
    \end{split}
\end{equation*}
The independence between $\bx$ and $\bH_i$ implies $\E[\Delta_{i,1}] = 0$. Write $z_i - \kappa_i$ as
$$
    z_i-\kappa_i = (z_i-\Tr(\bC_i\bH_i)) + (\Tr(\bC_i\bH_i)-\E[\Tr(\bC_i\bH_i)]).
$$
We control the first term by applying Lemma \ref{lem:new-MS24-lem2} conditionally on $\bH_i$. For the second term, we use the first-order martingale estimation established in the next section. Altogether, we get
$$
    \|z_i-\kappa_i\|_{L^q} \leq C_q\left(\frac{\nu}{\sqrt{n_i}}+\nu^2e_Kt_i\right) \leq C_q \nu^3e_K.
$$
Using 
$$
    \overline\bG_- \preceq C\nu\overline\bG, \qquad \overline\bG_-\bC_i\overline\bG_- \preceq C\frac{\nu^2}{n_i}\overline\bG,
$$
Lemma \ref{lem:new-MS24-lem2} to the squared linear forms, the basic deleted-resolvent moment bounds in Lemma \ref{lem:new1}, and Hölder's inequality, we get
\begin{equation*}
    \begin{split}
        \|\btheta^\top\bH_i\bx\|_{L^q} &\leq C_q \frac{\nu}{\sqrt{n_i}}\sqrt{\Tr(\bA\overline\bG)}, \\
        \|\bx^\top\overline\bG_-\btheta\|_{L^q} &\leq C_q \frac{\nu}{\sqrt{n_i}}\sqrt{\Tr(\bA\overline\bG)}, \\
        \|\bx^\top\bH_i\bC_i\overline\bG_-\btheta\|_{L^q} &\leq C_q \frac{\nu^2}{n_i^{3/2}}\sqrt{\Tr(\bA\overline\bG)}.
    \end{split}
\end{equation*}
By Hölder, 
\begin{equation*}
    \begin{split}
        n_i|\E[\Tr(\bA\Delta_{i,2})]| \leq Cn_i \frac{\nu}{\sqrt{n_i}}\frac{\nu}{\sqrt{n_i}} (\nu^3e_K)\Tr(\bA\overline\bG) = C\nu^5e_K\Tr(\bA\overline\bG), \\
        n_i|\E[\Tr(\bA\Delta_{i,3})]| \leq Cn_i \frac{\nu}{\sqrt{n_i}}\frac{\nu^2}{n_i^{3/2}}\Tr(\bA\overline\bG) = C\frac{\nu^3}{n_i}\Tr(\bA\overline\bG).
    \end{split}
\end{equation*}
Combining the previous results, we get, for the first term in \eqref{eq:fo1},
$$
    |\E[\Tr(\bA(\E[\bG]-\overline\bG_-))] \leq C(\nu^5e_K+\nu^3e_K^2)\Tr(\bA\overline\bG) \leq C\nu^5e_K\Tr(\bA\overline\bG).
$$
Now we need to bound the second term in \eqref{eq:fo1}. Defining
$$
    \varepsilon = \max_i \frac{|\kappa_i-t_i|}{1+t_i},
$$
we have
\begin{align*}
    |\Tr(\bA(\overline\bG_--\overline\bG))|
    &\leq \varepsilon\sum_{i=1}^K\frac{n_i}{1+\kappa_i}\Tr(\bA\overline\bG_-\bC_i\overline\bG) \\
    &= \varepsilon\Tr(\bA(\bI-\lambda\overline{\bG}_-)\overline\bG) \\
    &\leq \varepsilon\Tr(\bA\overline\bG).
\end{align*}
Now we need to upper bound $\varepsilon$. Using the one-observation restoration estimate established in Lemma \ref{lem:new1}, together with the preceding bounds, we obtain
\begin{align*}
    |\kappa_i-t_i|
    &\leq |\E[\Tr(\bC_i(\bH_i-\bG))]+ |\Tr(\bC_i(\E[\bG]-\overline\bG_-))| + |\Tr(\bC_i(\overline\bG_--\overline\bG))| \\
    &\leq C\frac{\nu^2}{n_i}t_i + C\nu^5 e_Kt_i + \varepsilon t_i.
\end{align*}
By definition, there is $i \in [K]$ such that $|\kappa_i-t_i| = \varepsilon(1+t_i)$. Substituting it to the previous inequality yields
$$
    \varepsilon \leq C\left(\frac{\nu^2}{n_i}+\nu^5e_K\right)t_i \leq C\nu^6e_K.
$$
Therefore, for the second term on \eqref{eq:fo1},
$$
    |\Tr(\bA(\overline\bG_--\overline\bG))| \leq C\nu^6e_K\Tr(\bA\overline\bG).
$$
Combining these two terms, 
$$
    |\E[\Tr(\bA\bG)]-\Tr(\bA\overline\bG)| \leq C\nu^6e_K\Tr(\bA\overline\bG).
$$

\subsubsection{Martingale part of \texorpdfstring{$\Tr(\bA\bG)$}{Tr(AG)}} 
Here we enumerate all $N = \sum_{i=1}^Kn_i$ observations. Define $\mathcal{F}_\ell$ be the first $\ell$ observations, and write $\E_\ell[\cdot] = \E[\cdot | \mathcal{F}_\ell]$. Assume the $\ell$-th feature is $\bX_{i,r}$. Defining $\bH_\ell = \bG_{-(i,r)}$, we have
$$
    \Tr(\bA\bG)-\E[\Tr(\bA\bG)] = \sum_{\ell=1}^{N}(\E_\ell-\E_{\ell-1})(\Tr(\bA\bG)-\Tr(\bA\bH_\ell)) =: \sum_{\ell=1}^{N}(\E_\ell-\E_{\ell-1})\xi_\ell.
$$
Sherman-Morrison gives
$$
    \xi_\ell = \Tr(\bA(\bG-\bH_\ell)) = -\frac{\bx_\ell^\top\bH_\ell\bA\bH_\ell\bx_\ell}{1+\bx_\ell^\top\bH_\ell\bx_\ell}.
$$
Applying Lemma \ref{lem:new-MS24-lem2} conditionally on $\bH_i$, together with the basic deleted-resolvent moment bounds in Lemma \ref{lem:new1}, gives
$$
    \|\xi_\ell\|_{L^q} \leq C_q \frac{\nu^2}{n_i}\Tr(\bA\overline\bG).
$$
Therefore,
$$
    \|\Tr(\bA\bG)-\E[\Tr(\bA\bG)]\|_{L^q} \leq C_q \left(\sum_{j=1}^K n_j\frac{\nu^4}{n_j^2}\right)^{1/2}\Tr(\bA\overline\bG) = C_q\nu^2e_K\Tr(\bA\overline\bG).
$$
The same argument applies when one observation is already removed, since the resulting row increments involve resolvents with at most two observations removed. For the probability bound, we use the same approach as in \cite[Appendix A.3.2]{misiakiewicz2024non}. 
Set $n_{\min}=\min_i n_i$. The conditional concentration bound in Lemma \ref{lem:new-MS24-lem2} and the probabilistic bound of the basic deleted-resolvent give
\begin{align*}
\bx_\ell^\top\bH_\ell\bA\bH_\ell\bx_\ell
&\le C_M\log^\eta(N)\Tr(\bA\bH_\ell\bC_i\bH_\ell)\\
&\le C_M\frac{\nu^2}{n_i}\log^\eta(N)
\Tr(\bA\overline\bG),
\end{align*}
except with probability at most $C_Mn_{\min}^{-M}$.
Here we used
\[
\Tr(\bA\bH_\ell\bC_i\bH_\ell)
\le
\|\bC_i^{1/2}\bH_\ell\bC_i^{1/2}\|_{\op}
\|\overline\bG^{-1/2}\bH_\ell\overline\bG^{-1/2}\|_{\op}
\Tr(\bA\overline\bG).
\]
Consequently,
\[
\P(|\xi_\ell|>R_i)\le C_Mn_{\min}^{-M},
\qquad
R_i:=C_M\frac{\nu^2}{n_i}\log^\beta(N)
\Tr(\bA\overline\bG).
\]
Define
\[
\widetilde\xi_\ell
=\xi_\ell\mathbf1_{\{|\xi_\ell|\le R_i\}},
\qquad
\widetilde d_\ell
=(\E_\ell-\E_{\ell-1})\widetilde\xi_\ell,
\qquad
d_\ell=(\E_\ell-\E_{\ell-1})\xi_\ell.
\]
Then $|\widetilde d_\ell|\le2R_i$ almost surely, and
Cauchy--Schwarz gives
\[
\E[|\xi_\ell|\mathbf1_{\{|\xi_\ell|>R_i\}}]
\le C_M\frac{\nu^2}{n_i}n_{\min}^{-M/2}
\Tr(\bA\overline\bG).
\]
Summing and applying Markov's inequality, using
$e_K\ge n_{\min}^{-1/2}$, for sufficiently large $M$ and
$n_{\min}$, with probability at least
$1-\frac12\sum_i n_i^{-D}$,
\[
\left|\sum_{\ell=1}^N(d_\ell-\widetilde d_\ell)\right|
\le C\nu^2e_K\Tr(\bA\overline\bG).
\]
By Azuma--Hoeffding, with probability at least
$1-\frac12\sum_i n_i^{-D}$,
\[
\left|\sum_{\ell=1}^N\widetilde d_\ell\right|
\le C_D\sqrt{\log(N)\sum_i n_iR_i^2}
\le C_D\nu^2e_K\log^{\eta+1/2}(N)
\Tr(\bA\overline\bG).
\]
Combining with the deterministic bound, with probability at least
$1-\sum_i n_i^{-D}$,
\[
|\Tr(\bA\bG)-\Tr(\bA\overline\bG)|
\le C_De_K\left(\nu^6+\nu^2\log^{\eta+1/2}(N)\right)
\Tr(\bA\overline\bG).
\]

\subsection{Deterministic equivalence of \texorpdfstring{$\Tr(\bA \bG \bC_k \bG)$}{Tr(A G Ck G)}}

We follow the same idea as in \cite[Appendix A.5]{misiakiewicz2024non}. First, we split the term again in deterministic part and martingale part. 

\subsubsection{Deterministic part of \texorpdfstring{$\Tr(\bA\bG\bC_k\bG)$}{Tr(AGCkG)}}

Denote $\kappa_i=\E[\Tr(\bC_i\bG_{-(i,1)})]$. For the row $\bx=\bx_{i,r}$ under consideration, set
\[
\bH=\bG_{-(i,r)},\qquad
q_i=\Tr(\bA\overline\bG_-\bC_i\overline\bG_-),
\qquad s_i=1+\kappa_i,
\]
and
\[
c_k=\bx^\top\bH\bC_k\bH\bC_i\overline\bG_-\btheta.
\]
By the first-order deterministic approximation and the one-observation restoration estimate established in the previous subsections, we have
\[
\delta:=\max_i\frac{|\kappa_i-t_i|}{1+t_i}
\le C\nu^6e_K.
\]
By the resolvent identity and commutativity,
\[
-\delta\overline\bG
\preceq\overline\bG_--\overline\bG
\preceq\delta\overline\bG.
\]
Moreover, since $t_i\le C\nu$ and $\kappa_i\ge0$,
\[
\overline\bG_-\preceq C\nu\overline\bG.
\]
Consequently,
\[
q_i\le C\nu^2\tau_{\bA}[i],
\qquad
|q_i-\tau_{\bA}[i]|
\le C\nu\delta\,\tau_{\bA}[i].
\]
Setting
$$
    T_{\bA,k} = \E[\Tr(\bA\bG\bC_k\bG)],
$$
we can decompose $T_{\bA,k}$ into
$$
    T_{\bA,k} = \Tr(\bA\overline\bG_-\bC_k\overline\bG_-) + \btheta^\top(\E[\bG]-\overline\bG_-)\bC_k\overline\bG_-\btheta + \sum_{i=1}^{K}n_i\E\left[\btheta^\top\overline\bG_-\left(\frac{\bC_i}{1+\kappa_i}-\bx\bx^\top\right)\bG\bC_k\bG\btheta\right],
$$
where $\bx$ has the same distribution as the features in $\bX_i$. Defining
$$
    \bH = \bG_{-(i,r)}, \qquad z = \bx^\top\bH\bx, \qquad f = 1+z,
$$
we write
$$
    \left(\frac{\bC_i}{1+\kappa_i}-\bx\bx^\top\right)\bG = \frac{(\bC_i-\bx\bx^\top)\bH}{1+\kappa_i}+\frac{\bx\bx^\top\bH(z-\kappa_i)}{(1+\kappa_i)f}-\frac{\bC_i\bH\bx\bx^\top\bH}{(1+\kappa_i)f},
$$
where we imitate the idea of decomposing the deterministic part into $\Delta_1, \Delta_2, \Delta_3$ as in \citep{misiakiewicz2024non}.
Furthermore, defining
$$
    a = \btheta^\top\bH\bx, \quad b = \bx^\top\overline\bG_-\btheta, \quad w=\bx^\top\bH\bC_k\bH\bx, \quad c=\bx^\top\bH\bC_i\overline\bG_-\btheta, 
$$
$$
    \quad a_k = \bx^\top\bH\bC_k\bH\btheta, \quad c_k = \bx^\top\bH\bC_k\bH\bC_i\overline\bG_-\btheta,
$$
we have
\begin{align*}
    \E\left[\btheta^\top\overline\bG_-\left(\frac{\bC_i}{1+\kappa_i}-\bx\bx^\top\right)\bG\bC_k\bG\btheta\right] = \underbrace{\E\left[\frac{abw}{f^2}\right]}_{\text{leading term}} + \underbrace{\E\left[\frac{(z-\kappa_i)ba_k}{(1+\kappa_i)f} - \frac{ca_k+ac_k}{(1+\kappa_i)f}+\frac{acw}{s_if^2}\right]}_{\text{remainder term}}.
\end{align*}
Therefore, next we want to give a deterministic approximation of the leading term, while proving the remainder terms are all small.

\paragraph{Bounding the leading term.} Setting $h_i^- = \E[\Tr(\bC_i\bH\bC_k\bH)]$, we claim
$$
    \E\left[\frac{abw}{f^2}\right] \approx \frac{q_ih_i^-}{(1+\kappa_i)^2}.
$$
To establish this result, we first show that we can replace the denominator from $f^2$ to $(1+\kappa_i)^2$. Recall that
$$
    \|z-\kappa_i\|_{L^p} \leq C_p\delta_i, \qquad \delta_i = C\left(\nu\sqrt{\frac{t_i}{n_i}}+\nu^2e_Kt_i\right) \leq C\nu^3e_K.
$$
Since $f, 1+\kappa_i \geq 1$, we can bound $f^{-2}-(1+\kappa_i)^{-2}$ as
$$
    |f^{-2}-(1+\kappa_i)^{-2}| \leq 2|z-\kappa_i|.
$$
Applying Lemma \ref{lem:new-MS24-lem2} conditionally on $\bH$, followed by Lemma \ref{lem:new2}, gives
$$
    \|a\|_{L^p}\leq C_p\sqrt{\mathcal{Q}_\bA[i]}, \qquad \|b\|_{L^p}\le C_p\sqrt{q_i}\le C_p\nu\sqrt{\tau_{\bA}[i]},
$$
where
$$
    \mathcal{Q}_\bA[i] = \tau_\bA[i]+\nu^3m_iW_\bA, \qquad W_\bA = \sum_{j=1}^K n_jt_j\tau_\bA[j].
$$
The same conditional moment estimate and the basic deleted-resolvent bounds in Lemma \ref{lem:new1} also give
\begin{align*}
\|w\|_{L^p}
&\le C_p\|\Tr(\bC_i\bH\bC_k\bH)\|_{L^p}\\
&\le C_p
\left\|\|\bC_k^{1/2}\bH\bC_k^{1/2}\|_{\op}\right\|_{L^{2p}}
\|\Tr(\bC_i\bH)\|_{L^{2p}}
\le C_p\frac{\nu^2}{n_k}.
\end{align*}
In particular, $h_i^-\le C\nu^2/n_k$.
H\"older then implies
\[
\left|\E[abw(f^{-2}-(1+\kappa_i)^{-2})]\right|
\le C\sqrt{\tau_{\bA}[i]\mathcal Q_{\bA}[i]}
\,\delta_i\frac{\nu^3}{n_k},
\]
Then, we show that we can replace $w$ by $h_i^-$. Decompose
$$
    w-h_i^- = (w-\Tr(\bC_i\bH\bC_k\bH)) + (\Tr(\bC_i\bH\bC_k\bH)-h_i^-).
$$

For the first term, set
$\bB=\bC_i^{1/2}\bH\bC_k\bH\bC_i^{1/2}$.
Since
\[
\|\bB\|_{\op}\le a_i(\bH)a_k(\bH),
\qquad
\Tr(\bB)\le a_k(\bH)\Tr(\bC_i\bH),
\]
Lemma \ref{lem:new-MS24-lem2} and $\|\bB\|_F^2\le\|\bB\|_{\op}\Tr(\bB)$ give
\[
\|w-\Tr(\bC_i\bH\bC_k\bH)\|_{L^p}
\le C_p\frac{\nu^2}{n_k\sqrt{n_i}}.
\]
For the second term, apply the row-deletion martingale argument
to $\Tr(\bC_i\bH\bC_k\bH)$.
For an additional row from dataset $j$, the conditional increment
is bounded by
\[
C_p a_j(\bR)a_k(\bR)\Tr(\bC_i\bR)
\bigl(1+\Tr(\bC_j\bR)\bigr),
\]
where $\bR$ has both rows removed.
Lemma \ref{lem:new1} therefore gives an unconditional increment
bound $C_p\nu^4/(n_jn_k)$, and hence
\[
\|\Tr(\bC_i\bH\bC_k\bH)-h_i^-\|_{L^p}
\le C_p\frac{\nu^4e_K}{n_k}.
\]
Consequently,
\[
\|w-h_i^-\|_{L^p}\le C_p\frac{\nu^4e_K}{n_k}, \qquad |\E[ab(w-h_i^-)]|
\le C\sqrt{\tau_{\bA}[i]\mathcal Q_{\bA}[i]}
\frac{\nu^5e_K}{n_k}.
\]
Finally, we show that we can approximate $\E[ab]$ by $q_i$. A direct calculation gives
$$
    \E[ab]-q_i = \btheta^\top(\E[\bG]-\overline\bG_-)\bC_i\overline\bG_-\btheta + \btheta^\top\E[\bH-\bG]\bC_i\overline\bG_-\btheta.
$$
For the second term, Sherman-Morrison gives
$$
    \btheta^\top(\bH-\bG)\bC_i\overline\bG_-\btheta = \frac{ac}{f}.
$$
Therefore, we have
$$
     \left|\btheta^\top(\E[\bH-\bG])\bC_i\overline\bG_-\btheta\right| \leq C\frac{\nu^2}{n_i}
     \sqrt{\tau_{\bA}[i]\mathcal Q_{\bA}[i]}.
$$
For the first term, 
$$
    \left|\btheta^\top(\E[\bG]-\overline\bG_-)\bC_i\overline\bG_-\btheta\right| \leq C\nu^2m_iJ_\bA,
$$
where
$$
    m_i = \|\overline\bG^{1/2}\bC_i\overline\bG^{1/2}\|_\op, \qquad J_\bA = \sum_{j=1}^K n_j\sqrt{\tau_\bA[j]\mathcal{Q}_\bA[j]}\left(\delta_j+\frac{\nu}{n_j}\right).
$$
Therefore,
$$
    |\E[ab]-q_i| \leq C\left(\nu^2m_iJ_\bA+\frac{\nu^2}{n_i}\sqrt{\tau_\bA[i]\mathcal{Q}_\bA[i]}\right).
$$
Combining the three approximations,
\begin{align*}
    \left|\E\left[\frac{abw}{f^2}\right]-\frac{q_ih_i^-}{(1+\kappa_i)^2}\right|
    &\leq \left|\E[abw(f^{-2}-(1+\kappa_i)^{-2})]\right| + \left|\frac{1}{(1+\kappa_i)^2}\E[ab(w-h_i^-)]\right| + \left|\frac{h_i^-}{(1+\kappa_i)^2}(\E[ab]-q_i)\right| \\
    &\leq \frac{C\sqrt{\tau_{\bA}[i]\mathcal Q_{\bA}[i]}}{n_k}
    \left(\nu^3\delta_i+\nu^5e_K+\frac{\nu^4}{n_i}\right)
    +\frac{C\nu^4m_i}{n_k}J_{\bA}.
\end{align*}
\paragraph{Bounding the remainder terms.} For the remainder terms, applying the same conditional moment estimates and deleted-resolvent bounds as above, we obtain
$$
    \|a_k\|_{L^p} \leq C_p\nu\sqrt{\frac{m_k\mathcal{Q}_\bA[k]}{n_i}}, \qquad \|c_k\|_{L^p}\leq C_p\frac{\nu^3m_k}{n_i}\sqrt{\tau_\bA[i]}.
$$
Therefore, we can control the terms appearing in the remainder terms by
\begin{equation*}
    \begin{split}
        n_i\E[|(z-\kappa_i)ba_k|] &\leq C\nu^2\delta_i\sqrt{n_im_k\tau_\bA[i]\mathcal{Q}_\bA[k]}, \\
        n_i\E[|ca_k|] &\leq C\nu^3\sqrt{m_k\tau_\bA[i]\mathcal{Q}_\bA[k]/n_i}, \\
        n_i\E[|ac_k|] &\leq C\nu^3m_k\sqrt{\tau_\bA[i]\mathcal{Q}_\bA[i]}, \\
        n_i\E[|acw|] &\leq C\frac{\nu^4}{n_k}\sqrt{\tau_{\bA}[i]\mathcal Q_{\bA}[i]}.
    \end{split}
\end{equation*}
Thus, we can write
$$
    T_{\bA, k} = q_k +\sum_{i=1}^K \frac{n_i}{(1+\kappa_i)^2}q_ih_i^- + \widetilde r_{\bA, k},
$$
where $\widetilde r_{\bA, k}$ is the controlled error. By the same covariance-weighted row-increment estimate used above,
\[
|h_i^--T_{\bC_i,k}|
\le C\frac{\nu^4}{n_in_k},
\qquad
h_i^-+T_{\bC_i,k}\le C\frac{\nu^2}{n_k}.
\]
Also, since $(1+t_i)/(1+\kappa_i)\le C\nu$,
\[
\left|
\frac{n_i}{(1+\kappa_i)^2}-\frac{\mu_i^2}{n_i}
\right|
\le C\nu^2\delta\,\frac{\mu_i^2}{n_i}.
\]
Together with the bounds on $q_i$, this gives
\[
\left|
\frac{n_iq_i}{(1+\kappa_i)^2}
-\frac{\mu_i^2}{n_i}\tau_{\bA}[i]
\right|
\le Cn_i\nu^4\delta\,\tau_{\bA}[i].
\]
Therefore, we can write
\begin{equation}\label{eq:coupled-original}
    T_{\bA, k} = \tau_\bA[k] + \sum_{i=1}^K \frac{\mu_i^2}{n_i}\tau_\bA[i] T_{\bC_i, k} + r_{\bA, k}.
\end{equation}
Now, we want to bound $\widetilde r_{\bA,k}$ and $r_{\bA,k}$.
Recall that
\[
\mathcal Q_{\bA}[i]\le C\frac{\nu^5}{n_i}V_{\bA},
\qquad
\sqrt{\tau_{\bA}[i]\mathcal Q_{\bA}[i]}
\le C\frac{\nu^3}{n_i}V_{\bA},
\qquad
J_{\bA}\le C\nu^6e_KV_{\bA}.
\]
Substituting these bounds into the preceding leading-term
and remainder estimates gives
\[
|\widetilde r_{\bA,k}|
\le C\frac{\nu^{11}e_K}{n_k}V_{\bA}.
\]
The replacement of $q_k$, the coefficients, and $h_i^-$
contributes at most
\[
C\nu\delta\,\tau_{\bA}[k]
+C\frac{\nu^6\delta}{n_k}V_{\bA}
+C\frac{\nu^6}{n_k}\sum_i\tau_{\bA}[i].
\]
Using $\delta\le C\nu^6e_K$, $\nu\ge1$, and
$e_K^2\le\sqrt K\,e_K$, we get
\[
\begin{aligned}
|r_{\bA,k}|
&\le C\frac{V_{\bA}}{n_k}
\left(
\nu^{11}e_K+\nu\delta+\nu^6\delta+\nu^6e_K^2
\right)\\
&\le C\frac{\nu^{12}e_K}{n_k}V_{\bA}.
\end{aligned}
\]
Having established a bound on $T_{\bA, k}$, now we consider the coupled system. Define $h_j = T_{\bC_j,k} = \E[\Tr(\bC_j\bG\bC_k\bG)]$, and write
$$
    \bh = 
    \begin{pmatrix}
        h_1 \\
        h_2 \\
        \vdots \\
        h_K
    \end{pmatrix}.
$$
Setting $\bA = \bC_j$, we have for all $j\in [K]$,
\begin{equation}\label{eq:coupled}
    h_j = \tau_{\bC_j}[k] + \sum_{i=1}^{K} \frac{\mu_i^2}{n_i}\tau_{\bC_j}[i]h_i + r_{\bC_j, k}.
\end{equation}
Recall that
$$
    \bK_{j,i} = \Tr(\bC_j\overline\bG\bC_i\overline\bG), \qquad \bD = \mathrm{diag}\left(\frac{n_1}{\mu_1^2}, \cdots , \frac{n_K}{\mu_K^2}\right), \qquad \bL = \bD-\bK.
$$
Then, we can write \eqref{eq:coupled} as
$$
    \bh = \bK\be_k + \bK\bD^{-1}\bh + \br.
$$
Substituting it back to \eqref{eq:coupled-original}, we get
$$
    T_{\bA, k} = \underbrace{\tau_\bA[k]+\tau_\bA^\top\bL^{-1}\tau_{\bC_k}}_{\text{deterministic equivalent}} + \underbrace{r_{\bA,k}+\sum_{1\leq i,j\leq K}\tau_\bA[i](\bL^{-1})_{i,j}\br_{\bC_j,k}}_{\text{remaining error}}.
$$
We finally have
\[
\begin{aligned}
\left|\E[\Tr(\bA\bG\bC_k\bG)]
-\be_k^\top\bD\bL^{-1}\tau_{\bA}\right|
&\le C\frac{\nu^{12}e_K}{n_k}
\left(
V_{\bA}
+\sum_{i,j}\tau_{\bA}[i](\bL^{-1})_{ij}V_{\bC_j}
\right).
\end{aligned}
\]
Now we use established upper bounds to control the terms appearing in the above expression.
Since
\[
    \sum_i n_i\bC_i
    =\sum_i\mu_i(1+t_i)\bC_i
    \preceq C\nu\sum_i\mu_i\bC_i
    \preceq C\nu\overline\bG^{-1},
\]
we have
\[
    V_{\bC_j}
    =\Tr\left(\bC_j\overline\bG
    \left(\sum_i n_i\bC_i\right)\overline\bG\right)
    \le C\nu t_j\le C\nu^2.
\]
Furthermore, because
\[
    \bL^{-1}\ge0\ \text{entrywise},
    \qquad
    \bL\boldsymbol\mu=\mathbf1+\lambda\tau_{\bI},
    \qquad
    \bL^{-1}\mathbf1\le\boldsymbol\mu,
\]
we have
\[
    \begin{aligned}
    \sum_{i,j}\tau_{\bA}[i](\bL^{-1})_{ij}V_{\bC_j}
    &\le C\nu^2\tau_{\bA}^{\top}\bL^{-1}\mathbf1\\
    &\le C\nu^2\sum_i\mu_i\tau_{\bA}[i]
    \le C\nu^2V_{\bA}.
    \end{aligned}
\]
Combining these bounds, we obtain
\[
\left|\E[\Tr(\bA\bG\bC_k\bG)]
-\be_k^\top\bD\bL^{-1}\tau_{\bA}\right|
\le C\frac{\nu^{14}e_K}{n_k}V_{\bA}.
\]

\subsubsection{Martingale part of \texorpdfstring{$\Tr(\bA\bG\bC_k\bG)$}{Tr(AGCkG)}}
For $i \in [K]$ and $r \in [n_i]$, define
$$
    \bH = \bG_{-i,r}, \qquad s = \bx_{i,r}^\top\bH\bx_{i,r}.
$$
We use Sherman-Morrison to get
\begin{equation}\label{eq:temp1}
    \Tr(\bA\bG\bC_k\bG) - \Tr(\bA\bH\bC_k\bH)
    = \underbrace{-\frac{2\bx^\top\bH\bA\bH\bC_k\bH\bx}{1+s}}_{\text{cross term}}+\underbrace{\frac{(\bx^\top\bH\bA\bH\bx)(\bx^\top\bH\bC_k\bH\bx)}{(1+s)^2}}_{\text{product term}}.
\end{equation}
Define
$$
    t_j = \Tr(\bC_j\overline\bG), \qquad W_\bA = \sum_{j=1}^{K}n_jt_j\tau_\bA[j], \qquad Q_j = \tau_\bA[j] + \nu^3\|\overline\bG^{1/2}\bC_j\overline\bG^{1/2}\|_\op W_\bA.
$$
By Lemma \ref{lem:new2}, we have
$$
    \|\Tr(\bA\bH\bC_j\bH)\|_{L^p} \leq C_pQ_j.
$$
Now we need to bound the cross term and the product term respectively in the decomposition of \eqref{eq:temp1}.

For the cross term, as in \cite[Appendix A.3.1]{misiakiewicz2024non}, it suffices to assume $\bA = \btheta\btheta^\top$. Setting $a_j(\bH) = \|\bC_j^{1/2}\bH\bC_j^{1/2}\|_\op$,  we can bound $\|a_j(\bH)\|_{L^p}\leq C_p\nu/n_j$ with Lemma \ref{lem:new1}. Using Lemma \ref{lem:new-MS24-lem2} and Hölder's inequality, we have
\begin{align*}
    \|\bx^\top\bH\bA\bH\bC_k\bH\bx\|_{L^p(\bx|\bH)} 
    &\leq C_p\sqrt{a_i(\bH)a_k(\bH)\Tr(\bA\bH\bC_i\bH)\Tr(\bA\bH\bC_k\bH)}.
\end{align*}
Taking the remaining expectation and applying Hölder, we get
$$
    \|\bx^\top\bH\bA\bH\bC_k\bH\bx\|_{L^p} \leq C_p\nu\sqrt{\frac{Q_iQ_k}{n_in_k}}.
$$
For the product term, similarly we have
$$
    \|\bx^\top\bH\bA\bH\bx\bx^\top\bH\bC_k\bH\bx\|_{L^p(\bx|\bH)} \leq C_p\Tr(\bA\bH\bC_i\bH)\Tr(\bC_i\bH\bC_k\bH).
$$
We have already obtained a bound on $\Tr(\bA\bH\bC_i\bH)$. For $\Tr(\bC_i\bH\bC_k\bH)$,
$$
    \Tr(\bC_i\bH\bC_k\bH) \leq a_i(\bH)\Tr(C_k\bH) \wedge a_k(\bH)\Tr(\bC_i\bH).
$$
Taking moments and applying Hölder's inequality, together with the same basic deleted-resolvent bounds, gives
$$
    \|\Tr(\bC_i\bH\bC_k\bH)\|_{L^p} \leq C_p\nu^2 \left(\frac{t_k}{n_i} \wedge \frac{t_i}{n_k}\right).
$$
Combining the previous results and setting
$$
    b_i = \nu\sqrt{\frac{Q_iQ_k}{n_in_k}} + \nu^2Q_i\left(\frac{t_i}{n_k}\wedge\frac{t_k}{n_i}\right),
$$
we have
$$
    \|\Tr(\bA\bG\bC_k\bG)-\Tr(\bA\bH\bC_k\bH)\|_{L^p} \leq C_pb_i.
$$
As before, we construct a martingale difference sequence and write
\begin{align*}
    \Tr(\bA\bG\bC_k\bG) - \E[\Tr(\bA\bG\bC_k\bG)]
    &= \sum_{\ell=1}^{N} (\E_\ell-\E_{\ell-1})\Tr(\bA\bG\bC_k\bG) \\
    &= \sum_{\ell=1}^{N} (\E_\ell-\E_{\ell-1})(\Tr(\bA\bG\bC_k\bG)-\Tr(\bA\bG_{-\ell}\bC_k\bG_{-\ell})),
\end{align*}
where $N = \sum_{k=1}^{K}n_k$ and the $\ell$-th feature we removed belongs to $\bX_{i,r}$. For convenience, we denote
$$
    \Delta_\ell = \Tr(\bA\bG\bC_k\bG) - \Tr(\bA\bG_{-\ell}\bC_k\bG_{-\ell}), \qquad d_\ell = (\E_\ell-\E_{\ell-1})\Delta_\ell.
$$
By the previous bound on $\Tr(\bA\bG\bC_k\bG)-\Tr(\bA\bG_-\bC_k\bG_-)$, we have
$$
    \|(\E_\ell-\E_{\ell-1})(\Tr(\bA\bG\bC_k\bG)-\Tr(\bA\bG_{-\ell}\bC_k\bG_{-\ell}))\|_{L^p} \leq 2\|\Delta_\ell\|_{L^p} \leq C_pb_i.
$$

By the martingale moment inequality, we have
$$
    \|\Tr(\bA\bG\bC_k\bG)-\E[\Tr(\bA\bG\bC_k\bG)]\|_{L^p} \leq C_p\left(\sum_{i=1}^K n_ib_i^2\right)^{1/2} \leq C_p\nu^8e_K^3V_\bA.
$$
For the probability bound, set $n_{\min}=\min_i n_i$.
Applying Lemma \ref{lem:new-MS24-lem2} conditionally on $\bH$ and the probabilistic results in Lemma \ref{lem:new1} and \ref{lem:new2} gives
\[
\P(|\Delta_\ell|>R_i)\le C_Mn_{\min}^{-M},
\qquad
R_i:=C_M\log^{4\beta+1}(N)b_i,
\]
where observation $\ell$ belongs to dataset $i$. Define
\[
\widetilde\Delta_\ell
=\Delta_\ell\mathbf1_{\{|\Delta_\ell|\le R_i\}},
\qquad
\widetilde d_\ell
=(\E_\ell-\E_{\ell-1})\widetilde\Delta_\ell.
\]
Then $|\widetilde d_\ell|\le2R_i$ almost surely.
By Cauchy--Schwarz,
\[
\E[|\Delta_\ell|\mathbf1_{\{|\Delta_\ell|>R_i\}}]
\le C_Mb_i n_{\min}^{-M/2}.
\]
The previously established bounds imply
\[
b_i\le CV_{\bA}\left(
\frac{\nu^6}{n_in_k}+\frac{\nu^8}{n_i^2}
\right),
\qquad
\sum_i n_ib_i\le C_K\nu^8e_K^2V_{\bA}.
\]
Thus, summing and applying Markov's inequality, using
$e_K\ge n_{\min}^{-1/2}$, gives
\[
\P\left(
\left|\sum_{\ell=1}^N(d_\ell-\widetilde d_\ell)\right|
>\nu^8e_K^3V_{\bA}
\right)
\le C_{M,K}n_{\min}^{-(M-1)/2}
\le\frac12\sum_i n_i^{-D}
\]
for sufficiently large $M$ and $n_{\min}$.

By Azuma--Hoeffding, with probability at least
$1-\frac12\sum_i n_i^{-D}$,
\begin{align*}
\left|\sum_{\ell=1}^N\widetilde d_\ell\right|
&\le C_D\sqrt{\log(N)\sum_i n_iR_i^2}\\
&\le C_D\log^{4\eta+3/2}(N)
\left(\sum_i n_ib_i^2\right)^{1/2}\\
&\le C_D\nu^8e_K^3\log^{4\eta+3/2}(N)V_{\bA}.
\end{align*}
Therefore, with probability at least $1-\sum_i n_i^{-D}$,
\[
\left|\Tr(\bA\bG\bC_k\bG)
-\E[\Tr(\bA\bG\bC_k\bG)]\right|
\le C_D\nu^8e_K^3\log^{4\beta+3/2}(N)V_{\bA}.
\]

Combined with the deterministic part of $\Tr(\bA\bG\bC_k\bG)$, we get
\[
    \left|\Tr(\bA\bG\bC_k\bG)
    -\be_k^\top\bD\bL^{-1}\tau_\bA\right| 
    \leq 
    C_De_K
    \left(\frac{\nu^{14}}{n_k}+\nu^8e_K^2\log^{4\beta+3/2}(N)\right)V_\bA.
\]

\subsection{Deterministic equivalence of \texorpdfstring{$\Tr(\bA \bG \bX_k^\top \bX_k \bG)$}{Tr(A G XkT Xk G)}}

\subsubsection{Deterministic part of \texorpdfstring{$\Tr(\bA \bG \bX_k^\top \bX_k \bG)$}{Tr(A G XkT Xk G)}}
Let $\bx = \bx_{k,1}$ and $\bH$ be the resolvent with this feature removed. Set
$$
    z = \bx^\top\bH\bx, \qquad \kappa_k = \E[\Tr(\bC_k\bH)], \qquad y = \bx^\top\bH\bA\bH\bx.
$$
Then, we have
$$
    \E[\Tr(\bA\bG\bX_k^\top\bX_k\bG)] = n_k \E\left[\frac{y}{(1+z)^2}\right].
$$
First, we show we can replace the denominator by $(1+\kappa_k)^2$ up to a small error. We have
$$
    \left|\frac{1}{(1+z)^2}-\frac{1}{(1+\kappa_k)^2}\right| \leq 2|z-\kappa_k|.
$$
By the same conditional quadratic-form and first-order martingale estimates used in the deterministic part of $\Tr(\bA\bG)$, we have
$$
    \|z-\kappa_k\|_{L^q} \leq C_q\left(\frac{\nu}{\sqrt{n_k}} + \nu^2e_Kt_k\right) \leq C_q\nu^3e_K.
$$
Moreover, applying Lemma \ref{lem:new-MS24-lem2} conditionally on $\bH$, followed by Lemma \ref{lem:new2}, gives
$$
    \|y\|_{L^q} \leq C_q \|\Tr(\bA\bH\bC_k\bH)\|_{L^q} \leq C_q\frac{\nu^5}{n_k}\sum_{i=1}^Kn_i\tau_\bA[i].
$$
Therefore,
$$
    \left|\E[\Tr(\bA\bG\bX_k^\top\bX_k\bG)] - \frac{n_k}{(1+\kappa_k)^2}\E[\Tr(\bA\bH\bC_k\bH)]\right| \leq C\nu^8e_K\sum_{i=1}^K n_i\tau_\bA[i].
$$
Now, we show we can approximate $\Tr(\bA\bH\bC_k\bH)$ by $\Tr(\bA\bG\bC_k\bG)$, by noticing that
$$
    \Tr(\bA\bH\bC_k\bH)-\Tr(\bA\bG\bC_k\bG) = 2\frac{\bx^\top\bH\bA\bH\bC_k\bH\bx}{1+z}-\frac{(\bx^\top\bH\bA\bH\bx)(\bx^\top\bH\bC_k\bH\bx)}{(1+z)^2}.
$$
Applying the conditional moment estimates of Lemma \ref{lem:new-MS24-lem2}, we have
\begin{equation*}
    \begin{split}
        \E_\bx[|\bx^\top\bH\bA\bH\bC_k\bH\bx|] &\leq C\|\bC_k^{1/2}\bH\bC_k^{1/2}\|_\op \Tr(\bA\bH\bC_k\bH), \\
        \E_\bx[(\bx^\top\bH\bA\bH\bx)(\bx^\top\bH\bC_k\bH\bx)] &\leq C\Tr(\bA\bH\bC_k\bH)\Tr(\bC_k\bH\bC_k\bH).
    \end{split}
\end{equation*}
Again, by the basic deleted-resolvent moment bounds in Lemma \ref{lem:new1}, together with
$$
    \Tr(\bC_k\bH\bC_k\bH) \leq \|\bC_k^{1/2}\bH\bC_k^{1/2}\|_\op\Tr(\bC_k\bH), \qquad \|\|\bC_k^{1/2}\bH\bC_k^{1/2}\|_\op\|_{L^q} \leq C_q\frac{\nu}{n_k}, $$
    $$
 \|\Tr(\bC_k\bH)\|_{L^q} \leq C_q\nu,
$$
we have
$$
    \frac{n_k}{(1+\kappa_k)^2}\E[|\Tr(\bA\bH\bC_k\bH)-\Tr(\bA\bG\bC_k\bG)|] \leq C\frac{\nu^7}{n_k^2}\sum_{i=1}^Kn_i\tau_\bA[i] \leq C\nu^8e_K\sum_{i=1}^Kn_i\tau_\bA[i].
$$
Therefore,
$$
    \left|\E[\Tr(\bA\bG\bX_k^\top\bX_k\bG)] - \frac{n_k}{(1+\kappa_k)^2}\E[\Tr(\bA\bG\bC_k\bG)]\right| \leq C\nu^8e_K\sum_{i=1}^Kn_i\tau_\bA[i].
$$
Finally, using the deterministic equivalent of $\Tr(\bA\bG\bC_k\bG)$,
$$
    |\E[\Tr(\bA\bG\bC_k\bG)]-\be_k^\top\bD\bL^{-1}\tau_\bA| \leq C\frac{\nu^{14}e_K}{n_k}V_\bA,
$$
we get
\begin{align*}
    &\quad \; |\E[\Tr(\bA\bG\bX_k^\top\bX_k\bG)]-\be_k^\top\bL^{-1}\tau_\bA| \\
    &\leq \underbrace{\left|\E[\Tr(\bA\bG\bX_k^\top\bX_k\bG)] - \frac{n_k}{(1+\kappa_k)^2}\E[\Tr(\bA\bG\bC_k\bG)]\right|}_{(I)} + \underbrace{\left|\frac{n_k}{(1+\kappa_k)^2}\E[\Tr(\bA\bG\bC_k\bG)] - \be_k^\top\bL^{-1}\tau_\bA\right|}_{(II)}.
\end{align*}
For $(I)$, we have
$$
    (I) \leq C\nu^8e_KV_\bA.
$$
For $(II)$, similar to the deterministic part of $\Tr(\bA\bG\bC_k\bG)$, and the coefficient comparison established in the preceding subsection, we have
\begin{align*}
    (II)
    &\leq \frac{n_k}{(1+\kappa_k)^2}\left|\E[\Tr(\bA\bG\bC_k\bG)] - \be_k^\top\bD\bL^{-1}\tau_\bA\right| + \left|\frac{n_k}{(1+\kappa_k)^2}-\frac{\mu_k^2}{n_k}\right|\be_k^\top\bD\bL^{-1}\tau_\bA \\
    &\leq Cn_k\frac{\nu^{14}e_K}{n_k}V_\bA + C\nu^8e_K\be_k^\top\bL^{-1}\tau_\bA, \\
    &\leq C\nu^{14}e_KV_\bA,
\end{align*}
where we use $\bD_{k,k} = n_k/\mu_k^2$, $\be_k^\top\bL^{-1}\tau_\bA \leq V_\bA$, and $\nu\geq 1$. Combining the previous bounds, we obtain
$$
    \left|\E[\Tr(\bA\bG\bX_k^\top\bX_k\bG)]
    -\be_k^\top\bL^{-1}\tau_\bA\right|
    \leq C\nu^{14}e_KV_\bA.
$$

\subsubsection{Martingale part of \texorpdfstring{$\Tr(\bA \bG \bX_k^\top \bX_k \bG)$}{Tr(A G XkT Xk G)}}
For convenience, we denote $\bS_k = \bX_k^\top\bX_k$. For $\bx = \bx_{i,r}$, we set
$$
    \bS_k = \bS_{k,-} + {\bf1}_{\{i=k\}}\bx\bx^\top.
$$
Note that $\Tr(\bA\bH\bS_{k,-}\bH)$ is independent of $\bx$. Setting $z = \bx^\top\bH\bx$, we get
\begin{equation*}
    \begin{split}
        \xi
        &:= \Tr(\bA\bG\bS_k\bG)-\Tr(\bA\bH\bS_{k,-}\bH) \\
        &= -2\frac{\bx^\top\bH\bA\bH\bS_{k,-}\bH\bx}{1+z}+\frac{(\bx^\top\bH\bA\bH\bx)(\bx^\top\bH\bS_{k,-}\bH\bx)}{(1+z)^2} + {\bf1}_{\{i=k\}}\frac{\bx^\top\bH\bA\bH\bx}{(1+z)^2}.
    \end{split}
\end{equation*}
Since $0 \preceq \bS_{k,-} \preceq \bH^{-1}$, we have $\bx^\top\bH\bS_{k,-}\bH\bx \leq z$. This shows that the sum of the last two terms is bounded by $2\bx^\top\bH\bA\bH\bx$, whose moments, by Lemma \ref{lem:new-MS24-lem2} conditionally on $\bH$, followed by Lemma \ref{lem:new2}, satisfy
$$
    \|\bx^\top\bH\bA\bH\bx\|_{L^q} \leq C_q \frac{\nu^5}{n_i}\sum_{j=1}^{K}n_j\tau_\bA[j].
$$
For the cross term, conditioning on the remaining observations and applying Lemma \ref{lem:new-MS24-lem2} and Cauchy--Schwarz give
\begin{align*}
    \|\bx^\top\bH\bA\bH\bS_{k,-}\bH\bx\|_{L^q(\bx|\cdot)}
    &\leq \|\bx^\top\bH\bA\bH\bx\|_{L^q(\bx|\cdot)}^{1/2} \|\bx^\top\bH\bS_{k,-}\bH\bA\bH\bS_{k,-}\bH\bx\|_{L^q(\bx|\cdot)}^{1/2} \\
    &\leq C_q\sqrt{\Tr(\bA\bH\bC_i\bH)\Tr(\bA\bH\bS_{k,-}\bH\bC_i\bH\bS_{k,-}\bH)} \\
    &\leq C_q\sqrt{\|\bC_i^{1/2}\bH\bC_i^{1/2}\|_\op\Tr(\bA\bH\bC_i\bH)\Tr(\bA\bH\bS_{k,-}\bH)}.
\end{align*}
With the basic deleted-resolvent moment bounds, Lemma \ref{lem:new2}, and the empirical-covariance estimate
\begin{equation*}
    \begin{split}
        \|\bC_i^{1/2}\bH\bC_i^{1/2}\|_\op \leq C_i\frac{\nu}{n_i}, \qquad &\|\Tr(\bA\bH\bC_i\bH)\|_{L^q} \leq C_q\frac{\nu^5}{n_i}\sum_{j=1}^K n_j\tau_\bA[j], \\
        \|\Tr(\bA\bH\bS_{k,-}\bH)\|_{L^q} &\leq C_q\nu^3\sum_{j=1}^K n_j\tau_\bA[j],
    \end{split}
\end{equation*}
we have
$$
    \|\bx^\top\bH\bA\bH\bS_{k,-}\bH\bx\|_{L^q} \leq C_q\frac{\nu^5}{n_i}\sum_{j=1}^Kn_j\tau_\bA[j].
$$
Here, the moment bound on $\Tr(\bA\bH\bS_{k,-}\bH)$ holds because
\begin{align*}
    \Tr(\bA\bH\bS_{k,-}\bH)
    &\leq 2\Tr(\bA\overline\bG\bS_{k,-}\overline\bG)
    +2\|\overline\bG^{-1/2}\bH\overline\bG^{-1/2}\|_{\op}\|\bY\|_F^2 \\
    &\leq C_q(n_k\tau_\bA[k]+\nu^2W_\bA) \\
    &\leq C_q\nu^3V_\bA.
\end{align*}

Now we obtain the bound
$$
    \|\xi\|_{L^q} \leq C_q\frac{\nu^5}{n_i}\sum_{j=1}^Kn_j\tau_\bA[j].
$$
The rest of the proof follows similarly. We again define the martingale difference sequence, and get
\begin{align*}
    \|\Tr(\bA\bG\bS_k\bG)-\E[\Tr(\bA\bG\bS_k\bG)]\|_{L^q}
    &=\left\|\sum_{\ell=1}^K(\E_\ell-\E_{\ell-1})\Tr(\bA\bG\bS_k\bG)\right\|_{L^q} \\
    &\leq C_q\left(\sum_{\ell=1}^K\|(\E_\ell-\E_{\ell-1})\Tr(\bA\bG\bS_k\bG)\|_{L^q}^2\right)^{1/2} \\
    &\leq C_q\nu^5\sum_{j=1}^Kn_k\tau_\bA[j]\left(\sum_{j=1}^K\frac{1}{n_j}\right)^{1/2} \\
    &= C_q\nu^5e_K\sum_{j=1}^Kn_j\tau_\bA[j].
\end{align*}
For a probabilistic bound, set $n_{\min}=\min_i n_i$. Similar to the martingale part of $\Tr(\bA\bG\bC_k\bG)$, we have
\[
    \P(|\xi_\ell|>R_i)\le C_Mn_{\min}^{-M},
    \qquad
    R_i:=C_M\frac{\nu^5}{n_i}\log^{3\eta+1}(N)V_{\bA}.
\]
Define
\[
    \widetilde\xi_\ell
    =\xi_\ell\mathbf1_{\{|\xi_\ell|\le R_i\}},
    \qquad
    \widetilde d_\ell
    =(\E_\ell-\E_{\ell-1})\widetilde\xi_\ell.
\]
Then $|\widetilde d_\ell|\le2R_i$ almost surely, and
Cauchy--Schwarz gives
\[
    \E[|\xi_\ell|\mathbf1_{\{|\xi_\ell|>R_i\}}]
    \le C_M\frac{\nu^5}{n_i}n_{\min}^{-M/2}V_{\bA}.
\]
Summing and applying Markov's inequality, for sufficiently
large $M$ and $n_{\min}$, with probability at least
$1-\frac12\sum_i n_i^{-D}$,
\[
    \left|\sum_{\ell=1}^N(d_\ell-\widetilde d_\ell)\right|
    \le C\nu^5e_KV_{\bA}.
\]
By Azuma--Hoeffding, with probability at least
$1-\frac12\sum_i n_i^{-D}$,
\[
    \left|\sum_{\ell=1}^N\widetilde d_\ell\right|
    \le C_D\sqrt{\log(N)\sum_{i=1}^K n_iR_i^2}
    \le C_D\nu^5e_K\log^{3\eta+3/2}(N)V_{\bA}.
\]
Therefore, with probability at least $1-\sum_i n_i^{-D}$,
\[
    |\Tr(\bA\bG\bX_k^\top\bX_k\bG)
    -\E[\Tr(\bA\bG\bX_k^\top\bX_k\bG)]|
    \le C_D\nu^5e_K\log^{3\eta+3/2}(N)V_{\bA}.
\]
Combining the previously established moment and deterministic
bounds, we obtain
\[
    \|\Tr(\bA\bG\bX_k^\top\bX_k\bG)
    -\be_k^\top\bL^{-1}\tau_{\bA}\|_{L^q}
    \le C_q\nu^{14}e_KV_{\bA},
\]
and, with probability at least $1-\sum_i n_i^{-D}$,
\[
    |\Tr(\bA\bG\bX_k^\top\bX_k\bG)
    -\be_k^\top\bL^{-1}\tau_{\bA}|
    \le C_De_K\left(
    \nu^{14}+
    \nu^5\log^{3\eta+3/2}(N)
    \right)V_{\bA}.
\]

\subsection{Technical lemmas}

\begin{lemma}
    \label{lem:new-MS24-lem1} Under  Assumption \ref{asm:distr}, define $\bM_i = \bC_i^{1/2} \bG \bC_i^{1/2}.$ For any $i \in [K],$ we have, with probability at least $1-n_i^{-D}$,
    \begin{equation*}
        \begin{split}
            \|\bM_i\|_{op} \leq C_D \frac{\nu_{\lambda,i}(n_i)}{ n_i} , \quad \|\bM_i\|_{F} \leq C_D \frac{\nu_{\lambda,i}(n_i)}{\sqrt{n_i}} , \quad \Tr(\bM_i) \leq C_D \nu_{\lambda,i}(n_i). 
        \end{split}
    \end{equation*}
\end{lemma}

\begin{proof}[Proof of Lemma \ref{lem:new-MS24-lem1}]

The result directly follows from \citep[Proof of Lemma 1]{misiakiewicz2024non}, by noting that \begin{equation*}
    \begin{split}
        \bM_i = \bC_i^{1/2} \left(\sum_{i=1}^K \bX_i^\top \bX_i +  \lambda\right)^{-1} \bC_i^{1/2} \preceq \bC_i^{1/2} \left(\bX_i^\top \bX_i + \lambda\right)^{-1} \bC_i^{1/2}.
    \end{split} 
\end{equation*} 
\end{proof}

Next, we state \cite[Lemma 2]{misiakiewicz2024non}. 

\begin{lemma}\label{lem:new-MS24-lem2}
    Assume the features $\bx_1, \bx_2$ are independent and satisfy Assumption \ref{asm:distr}. Then for all PSD matrices $\bB$ independent of $\bx_1, \bx_2$, with probability at least $1-n^{-D}$ over the randomness of $\bx_1, \bx_2$, we have
    \begin{equation*}
        \begin{split}
            |\bx_1^\top \bB \bx_1 - \Tr(\bC\bB) | 
            &\leq C_{x, D} \cdot 
            \log^\beta(n)\|\bC^{1/2}\bB\bC^{1/2}\|_F, \\
            |\bx_1^\top\bB\bx_2| 
            &\leq C_{x, D} \cdot \log^\beta(n)\left(\|\bC^{1/2}\bB\bC^{1/2}\|_F + 
            \|\bC^{1/2}\bB\bC\bB\bC^{1/2}\|_F^{1/2}\right).
        \end{split}
    \end{equation*}
    Moreover, for all $q\in\mathbb{N}$, we have
    \begin{equation*}
        \begin{split}
            \E_{\bx_1}\left[|\bx_1^\top\bB\bx_1 - \Tr(\bC\bB)|^q\right]^{1/q} 
            &\leq C_{x, q} \cdot 
            \|\bC^{1/2}\bB\bC^{1/2}\|_F, \\
            \E_{\bx_1, \bx_2}\left[|\bx_1^\top\bB\bx_2|^q\right]^{1/q} 
            &\leq C_{x, q} \left(\|\bC^{1/2}\bB\bC^{1/2}\|_F + 
            \|\bC^{1/2}\bB\bC\bB\bC^{1/2}\|_F^{1/2}\right).
        \end{split}
    \end{equation*}
\end{lemma}
\begin{proof}
    The proof is the same as in \cite[Lemma 2]{misiakiewicz2024non}.
\end{proof}

\begin{lemma}\label{lem:new1}
    Under Assumption \ref{asm:distr} and \ref{asm:nu}, let $\bG_-$ be the resolvent after at most 2 features respectively from $\bX_\ell$ and $\bX_j$ have been removed (here $\ell$ and j could be the same). Then, we have
    $$
        \Tr(\bC_i\overline\bG) \leq C\nu, \qquad \|\overline\bG^{1/2}\bC_i\overline\bG^{1/2}\|_\op \leq \frac{C\nu}{n_i},
    $$
    $$
        \|\Tr(\bC_i\bG_-)-\Tr(\bC_i\overline\bG)\|_{L^q} \leq C_q\left(\nu^2\left(\frac{1}{n_\ell}+\frac{1}{n_j}\right)+\nu^6\log^{C+1/2}(eN)\sum_{k=1}^{K}n_k^{-1/2}\right) \Tr(\bC_i\overline\bG).
    $$
    Moreover, for every fixed $M > 0$, with probability at least $1-\sum_{j=1}^K n_j^{-M}$,
    $$
        \|\overline\bG^{-1/2}\bG_-\overline\bG^{-1/2}\|_\op \leq C_M\nu, \qquad \|\bC_i^{1/2}\bG_-\bC_i^{1/2}\|_\op \leq C_M\frac{\nu}{n_i}.
    $$
    Therefore, for every deterministic PSD matrix $\bA$, 
    $$
        \Tr(\bA\bG_-) \leq C_M\nu\Tr(\bA\overline\bG) \qquad \text{with probability at least} \; 1-\sum_{j=1}^{K} n_j^{-M}.
    $$
\end{lemma}
\begin{proof}
    Write
        \[
        \nu=\nu_{\lambda,*},\qquad
        e_K=\left(\sum_{k=1}^K n_k^{-1}\right)^{1/2},
        \qquad t_i=\Tr(\bC_i\overline\bG).
        \]
    First, setting $h_i=d\wedge\lfloor\eta n_i\rfloor$, we have
    \begin{align*}
        n_i
        &=\mu_i+\Tr(\mu_i\bC_i\overline\bG)\\
        &\le \mu_i+
        \sum_{\ell=1}^d
        \frac{\mu_i\xi_\ell(\bC_i)}
        {\mu_i\xi_\ell(\bC_i)+\lambda}\\
        &\le \mu_i+h_i+
        \frac{\mu_i}{\lambda}\sum_{\ell>h_i}\xi_\ell(\bC_i)
        \le \nu\mu_i+\eta n_i.
    \end{align*}
    Consequently,
    \[
        \mu_i\ge \frac{1-\eta}{\nu}n_i,\qquad
        t_i\le \frac{\nu}{1-\eta}-1\le C\nu,\qquad
        \|\overline\bG^{1/2}\bC_i\overline\bG^{1/2}\|_{\op}
        \le \mu_i^{-1}\le \frac{C\nu}{n_i}.
    \]
    We next prove the random resolvent bounds.
    Fix $|\mathcal S|\le2$ and write $\bG_{-\mathcal S}:= \bG_-$ to be explicit on the features removed from the resolvent.
    Since
    \[
        \bG_-\preceq
        \left(\lambda\bI+
        \sum_{(i,r)\notin\mathcal S}
        \bx_{i,r}\bx_{i,r}^{\top}\right)^{-1},
    \]
    the single-dataset moment estimates in
    \cite[Lemma 4(b)]{misiakiewicz2024non}, applied to the
    remaining observations, give
    \[
        \left\|
        \|\bC_i^{1/2}\bG_-\bC_i^{1/2}\|_{\op}
        \right\|_{L^q}
        \le C_q\frac{\nu}{n_i},
        \qquad
        \|\Tr(\bC_i\bG_-)\|_{L^q}\le C_q\nu.
    \]
    Here the same estimates apply after at most two deletions
    under our standing assumptions.

    Moreover, $\lambda\bG_-\preceq\bI$ and $\mu_i\le n_i$.
    Thus,
    \begin{align*}
        \|\overline\bG^{-1/2}\bG_-\overline\bG^{-1/2}\|_{\op}
        &=\|\bG_-^{1/2}\overline\bG^{-1}\bG_-^{1/2}\|_{\op}\\
        &\le 1+\sum_{i=1}^K
        \mu_i\|\bG_-^{1/2}\bC_i\bG_-^{1/2}\|_{\op}\\
        &=1+\sum_{i=1}^K
        \mu_i\|\bC_i^{1/2}\bG_-\bC_i^{1/2}\|_{\op}.
    \end{align*}
    Taking $L^q$ norms proves
    \[
        \left\|
        \|\overline\bG^{-1/2}\bG_-\overline\bG^{-1/2}\|_{\op}
        \right\|_{L^q}
        \le C_q\nu,
    \]
    where the constant may depend on $K$.

    Now add back one removed observation
    $\bu=\bx_{j,r}$, and set
    $\bG_+^{-1}=\bG_-^{-1}+\bu\bu^\top$.
    By Sherman--Morrison,
    \[
        0\le\Tr(\bC_i(\bG_--\bG_+))
        =\frac{\bu^\top\bG_-\bC_i\bG_-\bu}
       {1+\bu^\top\bG_-\bu}
        \le \bu^\top\bG_-\bC_i\bG_-\bu.
    \]
    Conditional on $\bG_-$, the quadratic-form moment bound gives
    \[
        \|\Tr(\bC_i(\bG_--\bG_+))\|_{L^q(\bu\mid\bG_-)}
        \le C_q\Tr(\bC_j\bG_-\bC_i\bG_-).
    \]
    Also,
    \[
        \Tr(\bC_j\bG_-\bC_i\bG_-)
        \le
        \|\bC_j^{1/2}\bG_-\bC_j^{1/2}\|_{\op}
        \|\overline\bG^{-1/2}\bG_-\overline\bG^{-1/2}\|_{\op}
    t_i.
    \]
    Taking $L^q$ norms, applying H\"older and the preceding
    bounds at exponent $2q$, we obtain
    \[
        \|\Tr(\bC_i(\bG_--\bG_+))\|_{L^q}
        \le C_q\frac{\nu^2}{n_j}t_i.
    \]
    Repeating this for each removed observation yields
    \[
        \|\Tr(\bC_i(\bG_{-\mathcal S}-\bG))\|_{L^q}
        \le C_q\nu^2
        \left(\sum_{(j,r)\in\mathcal S}\frac1{n_j}\right)t_i.
    \]
    Finally, the previously established first-order moment
    bound gives
    \[
        \|\Tr(\bC_i(\bG-\overline\bG))\|_{L^q}
        \le C_q\nu^6e_Kt_i.
    \]
    The triangle inequality therefore implies
    \[
        \|\Tr(\bC_i\bG_-)-\Tr(\bA\overline\bG)\|_{L^q}
        \le C_q
        \left(
        \nu^2\sum_{(j,r)\in\mathcal S}\frac1{n_j}
        +\nu^6e_K
        \right)\Tr(\bA\overline\bG).
    \]
    For the probabilistic bound, applying the single-dataset resolvent estimates in Lemma \ref{lem:new-MS24-lem1} to the retained observations and taking the union bounds over the datasets, we get 
    $$
        \|\bC_i^{1/2}\bG_-\bC_i^{1/2}\|_\op \leq C_M\frac{\nu}{n_i}.
    $$
    Since $\lambda\bG_-\preceq\bI$ and $\mu_i \leq n_i$, we get
    \begin{align*}
        \|\overline\bG^{-1/2}\bG_-\overline\bG^{-1/2}\|_{\op}
        &=\|\bG_-^{1/2}\overline\bG^{-1}\bG_-^{1/2}\|_{\op}\\
        &\leq 1+\sum_i\mu_i
        \|\bC_i^{1/2}\bG_-\bC_i^{1/2}\|_{\op}
        \leq C_D\nu.
    \end{align*}
    The trace bound then follows immediately.
\end{proof}

\begin{lemma}\label{lem:new2}
    Under the assumptions of Lemmas \ref{lem:new-MS24-lem2}
    and \ref{lem:new1}, let $\bG_-$ denote the resolvent after a
    fixed set of at most two observations has been removed.
    For deterministic PSD matrices $\bA,\bB$ and every fixed $q\ge1$,
    \[
        \|\Tr(\bA\bG_-\bB\bG_-)\|_{L^q}
        \le C_q\left(
        \Tr(\bA\overline\bG\bB\overline\bG)
        +\nu^3
        \|\overline\bG^{1/2}\bB\overline\bG^{1/2}\|_{\op}
        W_{\bA}\right),
    \]
    where
    \[
        t_i=\Tr(\bC_i\overline\bG),\qquad
        \tau_{\bA}[i]=\Tr(\bA\overline\bG\bC_i\overline\bG),
        \qquad
        W_{\bA}=\sum_{i=1}^K n_it_i\tau_{\bA}[i].
    \]
    Moreover, for every fixed $D > 0$, with probability at least $1-\sum_{i=1}^K n_i^{-D}$,
    \[
        \Tr(\bA\bG_-\bB\bG_-)
        \leq C_D\log^{2\beta+1}(N)
        \left(
        \Tr(\bA\overline\bG\bB\overline\bG)
        +\nu^3
        \|\overline\bG^{1/2}\bB\overline\bG^{1/2}\|_{\op}
        W_{\bA}
        \right).
    \]
    \end{lemma}
\begin{proof}
    We adopt the same notation as in the proof of Lemma \ref{lem:new1}. Defining
    \[
        \bS_-:=\bG_-^{-1}-\lambda\bI
        =\sum_{(i,r)\notin\mathcal S}\bx_{i,r}\bx_{i,r}^{\top},
    \]
    we can write
    $$
        \Tr(\bA\bG_-\bB\bG_-) = \|\bB^{1/2}(\overline\bG\bA^{1/2}+\bG_-\overline\bG^{-1/2}\bY)\|_F^2,
    $$
    where
    $$
        \bY = \overline\bG^{1/2}\left(\sum_{k=1}^{K}\mu_k\bC_k-\bS_-\right)\overline\bG\bA^{1/2}.
    $$
    Next, we want to bound $\bY$. From \cite[Lemma 2]{misiakiewicz2024non}, we have for all $(i,r)$,
    \begin{align*}
        \left\|\|\overline\bG^{1/2}\bx_{i,r}\bx_{i,r}^\top\overline\bG\bA^{1/2}\|_F\right\|_{L^p} 
        &\leq \left\|\bx_{i,r}^\top\overline\bG\bx_{i,r}\right\|_{L^{2p}}^{1/2}\cdot \left\|\bx_{i,r}\overline\bG\bA\overline\bG\bu\right\|_{L^{2p}}^{1/2} \\
        &\leq C_p\sqrt{t_i\tau_\bA[i]}.
    \end{align*}
    Therefore, we also have
    $$
        \left\|\|\overline\bG^{1/2}(\bx_{i,r}\bx_{i,r}^\top-\bC_i)\overline\bG\bA^{1/2}\|_F\right\|_{L^p} \leq C_p\sqrt{t_i\tau_\bA[i]}.
    $$
    Since we want to obtain a bound on $\bY$, we first consider controlling 
    $$
        \bY - \E[\bY] = -\sum_{(i,r)} \overline\bG^{1/2}(\bx_{i,r}\bx_{i,r}^\top - \bC_i)\overline\bG\bA^{1/2} =: -\sum_{(i,r)}\bZ_{i,r.}
    $$
    By symmetrization and the Hilbert-space Khintchine--Kahane
inequality \citep[Chapter~4 and Section~6.1]{ledoux1991probability},
applied to the independent centered matrices $\bZ_{i,r}$
with the Frobenius norm, , for $p \geq 2$,
    \begin{align*}
        \left\|\left\|\sum_{(i,r)}\bZ_{i,r}\right\|_F\right\|_{L^p}
        &\leq C\sqrt{p}\left(\sum_{(i,r)}\left\|\|\bZ_{i,r}\|_F\right\|_{L^p}^2\right)^{1/2} \\
        &\leq C_p\left(\sum_{i=1}^{K}n_it_i\tau_\bA[i]\right)^{1/2} \\
        &= C_p\sqrt{W_\bA}.
    \end{align*}
    Next, we need to bound $\E[\bY]$. Define $d_i$ as the number of observations deleted from $\bX_i$. By assumption, we have $d_i \in \{0, 1, 2\}$ for all $i \in [K]$. We can write $\E[\bY]$ as
    $$
        \E[\bY] = -\overline\bG^{1/2}\sum_{i=1}^{K}(n_i-\mu_i)\bC_i\overline\bG\bA^{1/2} + \sum_{i=1}^{K}d_i\overline\bG^{1/2}\bC_i\overline\bG\bA^{1/2}.
    $$
    Define
    $$
        \bH_i = \overline\bG^{1/2}\bC_i\overline\bG^{1/2}, \qquad \bR = \overline\bG^{1/2}\bA^{1/2}, \qquad \bP = \sum_{i=1}^{K}(n_i-\mu_i)\bH_i.
    $$
    Then, we can write
    $$
        -\overline\bG^{1/2}\sum_{i=1}^{K}(n_i-\mu_i)\bC_i\overline\bG\bA^{1/2} = -\bP\bR.
    $$
    Since $n_i-\mu_i=\mu_it_i$ and
    \[
        \sum_{i=1}^K\mu_i\bH_i
        =\bI-\lambda\overline\bG\preceq\bI,
    \]
    we have
    \[
        0\preceq\bP
        =\sum_{i=1}^K\mu_it_i\bH_i
        \preceq \max_i t_i\,\bI
        \preceq C\nu\bI.
    \]
    Consequently,
    \begin{align*}
        \|\bP\bR\|_F^2
        &\le C\nu\Tr(\bR^\top\bP\bR)\\
        &=C\nu\sum_{i=1}^K\mu_it_i\tau_{\bA}[i]\\
        &\le C\nu W_{\bA}.
    \end{align*}
    Let $d_i$ be the number of observations removed from dataset $i$. Then
    \[
        0\le d_i\le n_i,\qquad \sum_{i=1}^K d_i\le2,
    \]
    and we can write
    $$
        \sum_{i=1}^{K}d_i\overline\bG^{1/2}\bC_i\overline\bG\bA^{1/2} = \sum_{i=1}^K d_i\bH_i\bR.
    $$
    For each $i$,
    $$
        \|\bH_i\bR\|_F^2 = \Tr(\bR^\top\bH_i^2\bR) \leq t_i\Tr(\bR^\top\bH_i\bR) = t_i\tau_\bA[i].
    $$
    Since there are at most two observations being deleted, using the inequality
    $$
        \|\bH_i\bR+\bH_j\bR\|_F^2 \leq 2\|\bH_i\bR\|_F^2 + 2\|\bH_j\bR\|_F^2,
    $$
    we have
    $$
        \left\|\sum_{i=1}^Kd_i\bH_i\bR\right\|_F^2 \leq 2\sum_{i=1}^Kd_it_i\tau_\bA[i] \leq 2W_\bA.
    $$
    Combined with previous results, we get
    $$
        \|\E[\bY]\|_F^2 \leq C\nu W_\bA.
    $$
    Therefore, we bound the $L^p$ norm of $\|\bY\|_F$ as
    $$
        \|\|\bY\|_F\|_{L^p}^2 \leq \left(\|\|\bY-\E[\bY]\|_{F}\|_{L^p} + \|\|\E[\bY]\|_F\|_{L^p}\right)^2 \leq C_p\nu W_\bA.
    $$
    With this bound on $\bY$, we can control $\Tr(\bA\bG_-\bB\bG_-)$. A direct computation gives
    \begin{align*}
        \Tr(\bA\bG_-\bB\bG_-)
        &= \|\bB^{1/2}\bG_-\bA^{1/2}\|_F^2 \\
        &\leq 2\|\bB^{1/2}\overline\bG\bA^{1/2}\|_F^2 + 2\|\bB^{1/2}\bG_-\overline\bG^{-1/2}\|_\op^2\|\bY\|_F^2 \\
        &\leq 2\Tr(\bA\overline\bG\bB\overline\bG) + 2\|\overline\bG^{1/2}\bB\overline\bG^{1/2}\|_\op\|\overline\bG^{-1/2}\bG_-\overline\bG^{-1/2}\|_\op^2\|\bY\|_F^2.
    \end{align*}
    Finally, we apply Hölder's inequality on the last term and get
    $$
        \|\|\overline\bG^{-1/2}\bG_-\overline\bG^{-1/2}\|_\op^2\|\bY\|_F^2\|_{L^q} \leq \|\|\overline\bG^{-1/2}\bG_-\overline\bG^{-1/2}\|_\op\|_{L^{4q}}^{2}\|\|\bY\|_F\|_{L^{4q}}^2 \leq C_q\nu^3W_\bA.
    $$
    For the probabilistic result the idea is similar to the martingale parts of the functionals. First, we have
    $$
        \P(\|\bZ_{i,r}\|_F \geq R_i) \leq C_MN^{-M}, \qquad R_i = C_M\log^\beta(N)\sqrt{t_i\tau_\bA[i]}.
    $$
    Set $\widetilde \bZ_{i,r} = \bZ_{i,r}{\bf 1}_{\{\|Z_{i,r}\|_F\leq L_i\}}$. Since $\E[\bZ_{i,r}] = 0$, 
    $$
        \left\|\sum_{i,r}\E[\widetilde \bZ_{i,r}]\right\|_F \leq C_MN^{-M/2}\sum_{i=1}^K n_i\sqrt{t_i\tau_\bA[i]} \leq C_MN^{(1-M)/2}\sqrt{W_\bA} \leq C_M\sqrt{W_A}
    $$
    for $M$ sufficiently large. Bounded-differences concentration then gives with probability at least $1-N^{-M}$,
    $$
        \left\|\sum_{i,r}(\widetilde \bZ_{i,r}-\E[\widetilde \bZ_{i,r}])\right\|_F \leq C_M\sqrt{\log(N)\sum_{i=1}^Kn_iR_i^2} \leq C_M\log^{\beta+1/2}(N)\sqrt{W_\bA}.
    $$
    Applying union bounds and choosing $M$ sufficiently large, we have with probability at least $1-\frac{1}{2}\sum_{i=1}^Kn_i^{_D}$,
    $$
        \|\bY-\E[\bY]\|_F \leq C_D\log^{\beta+1/2}(N)\sqrt{W_\bA}.
    $$
    Along with $\|\E[\bY]\|_F^2 \leq C\nu W_\bA$, we get
    $$
        \|\bY\|_F^2 \leq C_D\nu\log^{2\beta+1}(N)W_\bA.
    $$
    Therefore, we have with probability at least $1-\sum_{i=1}^K n_i^{-D}$,
    \[
        \Tr(\bA\bG_-\bB\bG_-)
        \leq C_D\log^{2\beta+1}(N)
        \left(
        \Tr(\bA\overline\bG\bB\overline\bG)
        +\nu^3
        \|\overline\bG^{1/2}\bB\overline\bG^{1/2}\|_{\op}
        W_{\bA}
        \right).
    \]
\end{proof}

\section{Omitted parts in Section \ref{sec:num}}
\label{apx:experiments}

\subsection{Details of numerical experiments}
\label{apx:detail-experiments}
\paragraph{Datasets and data preprocessing.}
We use SlimPajama \citep{cerebras2023slimpajama} as the
pretraining corpus. The target dataset consists of documents from the
\texttt{StackExchange} domain. We consider the following constructions
of the auxiliary dataset:
\begin{itemize}
    \item In the case where the mixture improves the scaling law compared to using each datasets individually,
    we exclude \texttt{StackExchange} and combine
    \texttt{GitHub}, \texttt{C4}, \texttt{arXiv},
    \texttt{Common Crawl}, \texttt{Books}, and
    \texttt{Wikipedia} with weights
    $0.637$, $0.099$, $0.095$, $0.076$, $0.048$, and $0.045$,
    respectively, following Fan et al.~\citep{fan2023doge}.

    \item In the case where the auxiliary distribution is close to the target distribution, 
we combine 90\% independent \texttt{StackExchange} data with 10\% of
the six-domain OOD mixture defined above. 
\end{itemize}

Before tokenization, we normalize each document using Unicode NFC,
canonicalize newline characters, and remove leading and trailing
whitespace. We then assign documents to the training, validation, and
test sets using normalized-content hashes, with a
$98\%/1\%/1\%$ split. This document-level splitting ensures that
documents with the same normalized content cannot occur in different
splits. We tokenize the documents using the GPT-2 byte-level BPE
tokenizer with a vocabulary size of $50{,}257$ and append an
end-of-text token to every document. The resulting token streams are
packed into non-overlapping blocks of $512$ tokens. For data preprocessing, we discard incomplete trailing blocks
and low-diversity blocks for which fewer than $10\%$ of the tokens are
unique. All dataset sizes reported in our experiments refer to
prediction tokens remaining after this preprocessing.

\paragraph{Model and optimization.}
All reported results use an 81.5M-parameter decoder-only transformer
with 6 layers, hidden width 768, 12 attention heads,
feed-forward width 3072, context length 512, and a 50,257-token GPT-2
vocabulary. 
We use AdamW with peak learning rate
$6\times10^{-4}$, $\beta=(0.9,0.95)$, $\epsilon=10^{-8}$, weight decay $0.1$,
gradient clipping at $1$, and a global batch of 32,768 prediction tokens.
Following \citep{hoffmann2022training}, we match the cosine learning-rate
schedule to the total training-token budget $H$ of each run:
\[
  \eta(t;H)=
  \begin{cases}
    6\times10^{-4}\,\dfrac{t}{W},
    & 0\leq t<W, \\[6pt]
    6\times10^{-5}
    + 2.7\times10^{-4}
      \left[1+\cos\!\left(\pi\dfrac{t-W}{H-W}\right)\right],
    & W\leq t\leq H,
  \end{cases}
\]
where $t$ is the number of processed prediction tokens before the update
and $W=262144$ is the fixed warmup budget.
The total budget is $H=n$, $H=n_2$, and $H=n+n_2$ for
Target-only, Aux-only, and Mixture, respectively.
Training uses PyTorch automatic mixed precision with BF16
for eligible operations. Model parameters and AdamW optimizer
states are maintained in FP32, and TF32 is enabled for the
remaining FP32 matrix-multiplication operations.

\subsection{Experiments on datasets with close distributions}
\label{apx:non-improve-close}

Figure \ref{fig:lang-scalinglaw-close} shows that when the token frequency of the two datasets is close, for $\gamma_2=1.25$, the data mixture does not improve the scaling law compared to using one of the individual datasets. Table \ref{tab:result4-near-aux-90-10-slopes} shows that this phenomenon occurs across different scales $\gamma_2.$ This qualitatively supports our theoretical finding that data mixtures improve scaling only when spectral decay rates are sufficiently different.

\begin{figure}[t]
\centering
\begin{subfigure}[t]{0.48\linewidth}
    \centering
    \includegraphics[width=\linewidth]{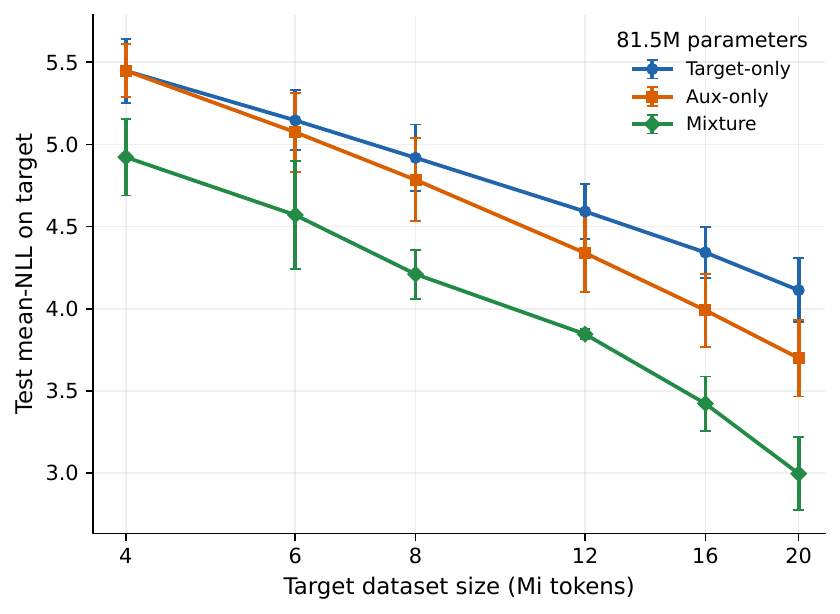}
    \caption{$\gamma_2=1.25$}
    \label{fig:close-b125}
\end{subfigure}\hfill \medskip\begin{subfigure}[t]{0.48\linewidth}
    \centering    
    \includegraphics[width=\linewidth]{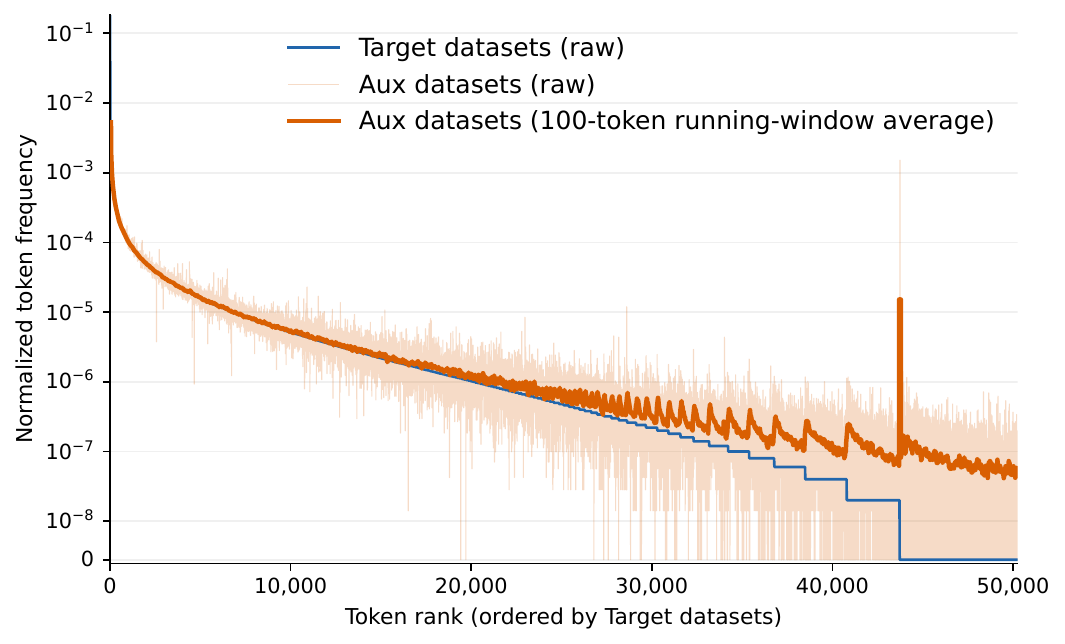}
    \caption{Token-frequency distributions}
    \label{fig:close-token-freq}
\end{subfigure}

\caption{
Target-domain test negative log-likelihood (NLL) scaling of an 81.5M-parameter GPT-2-style language model
trained from scratch on SlimPajama. Panel
(a) reports mean NLL on a fixed target-domain test set as a function
of the reference target budget $n$. We use
$n\in\{4,6,8,12,16,20\}\,\mathrm{Mi}$, except for the
$\gamma_2=1.75$ experiment that uses $n\leq16\,\mathrm{Mi}$.  Panel (b) compares the normalized
target and auxiliary token frequencies. Token types are
ordered by decreasing target-domain frequency. The light auxiliary
curve shows the raw frequencies, and the darker, thicker curve shows
their 100-token running-window average.}
\label{fig:lang-scalinglaw-close}
\end{figure}

\begin{table}[tbp]
  \centering
   \caption{Linear OLS scaling slopes on the target test set for the 81.5M-parameter model over 4--20 M target datasets tokens. More-negative values indicate a faster decrease in target NLL, i.e. better scaling law. Each slope is obtained by fitting $\text{NLL}=\alpha+s\ln(n)$ over the six Target dataset sizes $n\in\{4,6,8,12,16,20\}$ M within each seed and then averaging over three seeds. Improvement is $100\,(|s_{\mathrm{Mixture}}|-|s_X|)/|s_X|$, reported against target-only and auxiliary-only, respectively.}
  \label{tab:result4-near-aux-90-10-slopes}
  \small
  \setlength{\tabcolsep}{3pt}
  \begin{tabular}{crrrr}
    \toprule
    $\gamma_2$ & Target-only & Auxiliary-only & Mixture & \shortstack{Mixture improvement (\%)\\Target / Auxiliary} \\
    \midrule
$0.75$ & -0.8319 & -0.6085 & -0.8742 & +5.1\% / +43.7\% \\
\midrule
$1$ & -0.8319 & -0.8389 & -1.0272 & +23.5\% / +22.4\% \\
\midrule
$1.25$ & -0.8319 & -1.1155 & -1.1225 & +34.9\% / +0.6\% \\
\midrule
$1.5$ & -0.8319 & -1.4015 & -1.2513 & +50.4\% / -10.7\% \\
\midrule
$1.75$ & -0.8319 & -1.5971 & -1.3208 & +58.8\% / -17.3\% \\
\midrule
$2$ & -0.8319 & -1.7332 & -1.4361 & +72.6\% / -17.1\% \\
    \bottomrule
  \end{tabular}

\end{table}

\end{document}